\documentclass{article}
\PassOptionsToPackage{numbers, compress}{natbib}
\usepackage{subcaption}
\usepackage{natbib}
\usepackage[a4paper, total={6in, 8in}]{geometry}

\usepackage{amsmath}
\usepackage{amsfonts}
\usepackage{bm}
\usepackage{amssymb}
\usepackage{mathtools}
\usepackage{amsthm}
\usepackage{dsfont}
\usepackage{bbm} 
\usepackage{stmaryrd} 
\usepackage{enumitem}
\usepackage{todonotes}
\usepackage{overpic}

\DeclareMathOperator*{\argmin}{arg\,min}

\DeclarePairedDelimiter\abs{\lvert}{\rvert}%
\DeclarePairedDelimiter\norm{\lVert}{\rVert}%
\DeclarePairedDelimiter\opnorm{\lVert}{\rVert_{\textnormal{op}}}

\newcommand{\trace}{\textnormal{Tr}}
\newcommand{\cov}{\mathbf{\Sigma}}

\newcommand{\bigp}{\mathbf{P}}
\newcommand{\bigq}{\mathbf{Q}}

\newcommand{\wasserstein}{\mathbf{W}}

\newcommand{\kl}{\mathbf{H}}
\newcommand{\na}{\texttt{NA}}

\definecolor{rightblue}{RGB}{76,114,176} 
\definecolor{rightorange}{RGB}{221,132,82} 
\definecolor{aliceblue}{rgb}{0.94, 0.97, 1.0} 
\definecolor{darkcerulean}{rgb}{0.03, 0.27, 0.49} 
\definecolor{iris}{rgb}{0.35, 0.31, 0.81} 
\definecolor{carmine}{rgb}{0.59, 0.0, 0.09} 
\definecolor{green(munsell)}{rgb}{0.0, 0.66, 0.47} 
\definecolor{celadon}{rgb}{0.67, 0.88, 0.69} 
\definecolor{bluerow}{rgb}{0.0, 0.53, 0.74} 
\definecolor{lightorange}{RGB}{255, 219, 187} 
\definecolor{lavenderblue}{rgb}{0.8, 0.8, 1.0}
\definecolor{blue(pigment)}{rgb}{0.2, 0.2, 0.6}
\definecolor{blue-violet}{rgb}{0.54, 0.17, 0.89}
\usepackage[backref=page]{hyperref}
   \hypersetup{
     final=true,
     plainpages=false,
     pdfstartview=FitV,
     pdftoolbar=true,
     pdfmenubar=true,
     bookmarksopen=true,
     bookmarksnumbered=true,
     breaklinks=true,
     linktocpage=true,
     colorlinks=true,
     linkcolor=blue-violet, 
     urlcolor=iris, 
     citecolor=darkcerulean, 
     anchorcolor=black
   }

\usepackage{url}            
\usepackage{nicefrac}       
\usepackage{microtype}      

\usepackage{booktabs}
\usepackage{multirow}
\usepackage{graphicx}
\usepackage{caption} 
\usepackage[capitalize]{cleveref}

\usepackage{amsthm}
\theoremstyle{definition} 
\newtheorem{thm}{Theorem}[section]
\newtheorem{lem}[thm]{Lemma}
\newtheorem{prop}[thm]{Proposition}
\newtheorem{defn}[thm]{Definition}

\newtheorem{rmk}{Remark}
\newtheorem{asm}{Assumption}

\def\propcolor{lavenderblue!25}
\def\thmcolor{lightorange!45}
\usepackage{mdframed}
\newmdtheoremenv[topline=false, bottomline=false, leftline=false, rightline=false, backgroundcolor=\thmcolor,%
innertopmargin=\topskip, splittopskip=\topskip, skipbelow=\baselineskip, skipabove=\baselineskip]{boxthm}{Theorem}[section]

\newmdtheoremenv[topline=false, bottomline=false, leftline=false, rightline=false, backgroundcolor=\propcolor,%
innertopmargin=\topskip, splittopskip=\topskip, skipbelow=\baselineskip, skipabove=\baselineskip]{boxprop}[boxthm]{Proposition}

\newmdtheoremenv[topline=false, bottomline=false, leftline=false, rightline=false, backgroundcolor=\propcolor,%
innertopmargin=\topskip, splittopskip=\topskip, skipbelow=\baselineskip, skipabove=\baselineskip]{boxexample}[boxthm]{Example}
\newmdtheoremenv[topline=false, bottomline=false, leftline=false, rightline=false, backgroundcolor=\propcolor,%
innertopmargin=\topskip, splittopskip=\topskip, skipbelow=\baselineskip, skipabove=\baselineskip]{boxcor}[boxthm]{Corollary}
\newmdtheoremenv[topline=false, bottomline=false, leftline=false, rightline=false, backgroundcolor=\propcolor,%
innertopmargin=\topskip, splittopskip=\topskip, skipbelow=\baselineskip, skipabove=\baselineskip]{boxlem}[boxthm]{Lemma}
\newmdtheoremenv[topline=false, bottomline=false, leftline=false, rightline=false, backgroundcolor=\propcolor,%
innertopmargin=\topskip, splittopskip=\topskip, skipbelow=\baselineskip, skipabove=\baselineskip]{boxdef}[boxthm]{Definition}

\usepackage{tikz}
\usepackage{varwidth}

\usepackage{stmaryrd}

\usepackage{dirtytalk}
\MakeRobust{\say} 

\usepackage{makecell}

\usepackage{wrapfig}

\usepackage{enumitem}

\usepackage{nicematrix}

\usepackage[capitalize]{cleveref}

\usepackage{hyperref}

\usepackage{scalerel}
\usepackage{stackengine,wasysym}

\usepackage{algorithm}
\usepackage{algpseudocode}
 \usepackage{lineno}
 \usepackage{mathrsfs}
 \usepackage{soul}

\title{Optimal Transport with Heterogeneously Missing Data}

\author{Linus Bleistein\footnote{Correspondence to \texttt{linus.bleistein@inria.fr}.} \\ Inria \and Aurélien Bellet\\ Inria \and Julie Josse\\ Inria }

\begin{document}

\maketitle

\begin{abstract}
We consider the problem of solving the optimal transport problem between two empirical distributions with missing values. Our main assumption is that the data is missing completely at random (MCAR), but we allow for heterogeneous missingness probabilities across features and across the two distributions. As a first contribution, we show that the Wasserstein distance between empirical Gaussian distributions and linear Monge maps between arbitrary distributions can be debiased without significantly affecting the sample complexity. Secondly, we show that entropic regularized optimal transport can be estimated efficiently and consistently using iterative singular value thresholding (ISVT). We propose a validation set-free hyperparameter selection strategy for ISVT that leverages our estimator of the Bures-Wasserstein distance, which could be of independent interest in general matrix completion problems. Finally, we validate our findings on a wide range of numerical applications.
\end{abstract}

\section{Introduction}

Data missingness is one of the most common challenges faced by data science practitioners: for some individuals, certain features are not observed and are represented as $\texttt{NA}$~\citep{dempster1977maximum, bang2005doubly, zhu2022high,le2020neumiss, josse2024consistency,verchand2024high,ma2024estimation}. Missingness can be random (a sensor randomly defaults) or informative (a doctor does not order a given sample examination because it is needless for the suspected condition). To deal with missing data, two extreme strategies are often used. A first common practice is \textit{complete case analysis}, which removes individuals with missing values from the dataset. This will result in discarding increasingly important parts of the data as the dimension $d$ grows: assuming a missingness probability of $0.05$, $15\%$ of the rows are removed when $d=3$ ; for $d=100$, this percentage increases to $99.5\%$! At the other end of the spectrum is \textit{naive imputation}, where missing values are replaced by a fixed value such as the mean. Between these extremes lies a range of more sophisticated models that impute missing values by leveraging observed data \citep{dempster1977maximum,van2007multiple,van2011mice,VANBUUREN2018, little2019statistical}.

In practice, missing values are handled in different fashions depending on the final objective, be it supervised learning, parameter estimation, or matrix completion. When the final goal is a supervised learning task, a relative empirical consensus has emerged in the last few years \citep{perez2022benchmarking,morvan2024imputation}, backed by theoretical findings~\citep{le2021sa,ayme2023naive,josse2024consistency}: the specific choice of imputation matters less than the choice of a well-suited and expressive predictive model~\citep{perez2022benchmarking}. For instance, \citet{morvan2024imputation} showed in a recent study that even substantial improvements in imputation accuracy lead to only marginal gains in downstream prediction accuracy across a wide range of datasets.

When the goal is \emph{not} to predict outcomes in a supervised fashion---a setting in which model expressivity can compensate for data missingness---but to estimate statistics or to compare empirical distributions, the picture is more subtle. For instance, naive constant imputation will bias higher order moments of a distribution by lowering its variance: hence, the distributions will be distorted by the imputation method. Put otherwise, in such settings, one ideally seeks an imputation method that preserves the moments of the underlying data~\citep{van2011mice,cherief2025parametric}.

Comparing distributions requires a well-defined notion of distance. Optimal Transport (OT), first introduced by~\citet{monge1781memoire}, has become a powerful and widely used framework for this purpose. Its appealing computational properties, along with its deep connections to optimization and statistics, have led to a vast and diverse body of work, spanning causal inference~\citep{johansson2022generalization}, batch effect correction~\citep{falahkheirkhah2023domain} and many more~\citep{galichon2018optimal,peyre2019computational}. However, missing values can severely distort the moments of a distribution, thereby compromising the reliability of OT-based distributional comparisons. Figure~\ref{fig:illustration_ot_na} illustrates how missingness can affect the resulting transport matchings. This leads us to the central question we address in our work: 
\begin{center}
    \textit{how can we solve the optimal transport problem with incomplete data to mitigate the bias introduced by missingness?}
\end{center}

\begin{figure}
    \centering
    \includegraphics[width=0.7\linewidth]{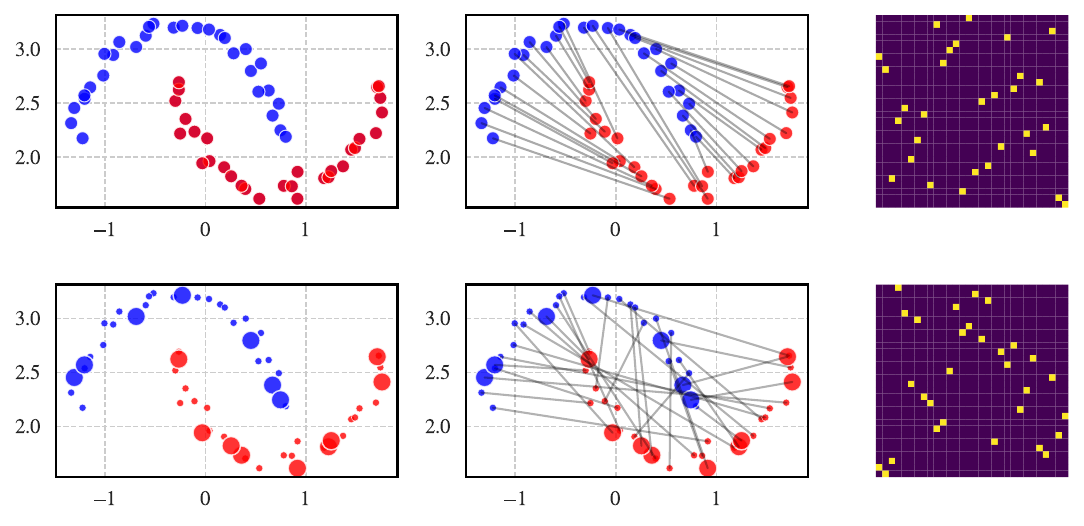}
     \caption{Illustration of the effect of missing data on the optimal transport matching in a toy example. \textbf{Left panels:} data without missing values (\textbf{top}) and with missing values (\textbf{bottom})---the downsized points have at least one missing coordinate. \textbf{Central panels}: missing data is imputed by $0$ (not shown in this figure), introducing bias in the optimal transport plan. \textbf{Right panels}: optimal coupling matrices.}
    \label{fig:illustration_ot_na}
\end{figure}

\paragraph{Contributions and Outline.} Our contribution is threefold. First, in Section \ref{section:implicit_bias}, we provide a precise analysis of the effect of naive imputation on the optimal transport problem, showing that it implicitly biases the cost function. Second, in Section \ref{section:linear_monge}, we devise a consistent, dimension-free estimator of the Bures-Wasserstein distance by leveraging recent results in high-dimensional covariance estimation, and in particular a novel non-commutative Bernstein inequality \citep{pacreau2024robust}. The modular structure of our approach further allows us to derive new generalization bounds for domain adaptation with missing values. Third, in Section \ref{section:generalized_estimator}, we move beyond the Gaussian setting to address general entropic regularized optimal transport between distributions supported on compact sets. We establish the consistency of a two-step procedure based on iterative singular value thresholding using recent cost-sensitivity bounds on entropic regularized optimal transport~\cite{keriven2022entropic}, and show that our Bures-Wasserstein estimator enables a highly effective validation set-free strategy for hyperparameter selection. Finally, Section \ref{section:applications} presents a comprehensive set of numerical experiments.

\section{Related Work}
\label{section:related_work}
\paragraph{Estimation with Missing Data.} Inference with missing data is a long standing topic in statistics~\citep{dempster1977maximum,van2007multiple, van2011mice,VANBUUREN2018, little2019statistical}. More recently, it has gained attention within the machine learning community, where research has largely focused on imputation methods aimed at improving downstream supervised learning performance. This line of work has produced a variety of algorithms that predict the value of a missing feature using all other available information~\citep{muzellec2020missing,le2020neumiss,perez2022benchmarking,morvan2024imputation}. 
Among them, \citet{muzellec2020missing} leverage optimal transport to iteratively impute missing values, aiming to preserve higher-order moments and enhance downstream supervised learning performance. In contrast, our goal is to estimate optimal transport distances under missing data.

Another line of work has analyzed the effects of two-step approaches that combine naive imputation followed by debiasing~\citep{lounici2014high,zhu2022high,josse2020debiasing,ayme2023naive,ayme2024random,pacreau2024robust}. However, little is known on imputation strategies in distribution-dependent cases such as statistical testing or domain adaptation: we take a first step in this direction by tackling the question of the estimation of the optimal transport distance.

\paragraph{Matrix Completion.} The seminar work of~\citet{candes2008exact} and~\citet{candes2010power} established a foundational link between convex optimization and the inference of missing values, and sparking a long line of research on matrix completion. Since then, numerous extensions have been developed, addressing challenges such as noisy observations, collaborative settings, and inference via stochastic first-order optimization~\citep{dhanjal2014online,klopp2015matrix,yao2017accelerated,alaya2019collective}. Hyperparameter section in these settings has received comparatively less attention. Existing approaches typically rely on leaving out entries~\citep{owen2009bi,gui2023conformalized} or use variants of the Stein estimator~\citep{josse2016adaptive}. In this work, we use ISVT~\citep{hastie2015matrix}, a popular matrix completion algorithm, to enhance optimal transport estimation in the presence of missing data. We also introduce a novel hyperparameter selection strategy based on optimal transport which does not require a validation set.

\paragraph{Statistical Optimal Transport.} The literature on modern OT and its statistical aspects is extremely vast and hard to survey. We refer the reader to~\citep{villani2009optimal,santambrogio2015optimal,peyre2019computational,bunne2023optimal,chewi2024statistical} for overviews of the topic. Our results on optimal transport under missing data are highly modular, allowing seamless integration into existing methods and pipelines. This flexibility is illustrated by our extension of the domain adaptation framework proposed by \citet{flamary2019concentration}.

\section{Setup and Mathematical Background}
\label{section:setup}
Throughout the paper, the set of positive semi-definite (resp. definite) $d\times d$ matrices is denoted by $\mathcal{S}^{+}_d$ (resp. $\mathcal{S}^{++}_d$). We write $\norm{\mathbf{M}}_\textnormal{op}$ for the operator norm of a matrix i.e. its largest eigenvalue, and by $\norm{\mathbf{M}}_\textnormal{F}:= \textnormal{Tr}^\frac{1}{2}(\mathbf{M}^\top\mathbf{M})$ its Frobenius norm.
We use $\lesssim$ to denote inequality up to a universal positive multiplicative constant. We write $a \wedge b = \min(a,b)$ and $a \vee b = \max(a,b)$. All proofs are deferred to the Appendix.

\subsection{Missing Values}

\paragraph{Data.} We consider two i.i.d. samples of random variables $\mathbf{X} = \big(\mathbf{x}_1,\dots,\mathbf{x}_{n}\big) \in \mathbb{R}^{n \times d}$ and $\mathbf{Y} = \big(\mathbf{y}_1,\dots,\mathbf{y}_{m}\big) \in \mathbb{R}^{m \times d}$ drawn from $\mu$ and $\eta$, with means $\mathbf{m}_x$ and $\mathbf{m}_y$ and covariances $\cov_x$ and $\cov_y$. We assume that we do not have access to perfect observations of $\mathbf{X}$ and $\mathbf{Y}$ but observe datasets with missing values imputed by $0$ --- we focus on $0$ imputation following the the line of work discussed above, but other constant imputation methods can be used here. We hence observe
\begin{align*}
    \mathbf{X}^{\textnormal{\texttt{NA}}} := \mathbf{x}\odot \Omega_x \quad \textnormal{and} \quad \mathbf{Y}^{\textnormal{\texttt{NA}}} :=\mathbf{y}\odot \Omega_y 
\end{align*}
where $\Omega_x \in \{0,1\}^{n\times d}$ and $\Omega_y \in \{0,1\}^{m\times d}$ are masking matrices whose $(i,j)$-entry indicates whether variable $j$ has been observed for individual $i$ i.e. $(\Omega)_{ij} =1$ or is missing i.e. $(\Omega)_{ij} =0$.

We let $\mu^{\texttt{NA}}$ and $\eta^{\texttt{NA}}$ be the distributions of $X^\textnormal{\texttt{NA}}$ and $Y^\textnormal{\texttt{NA}}$, which we formally define as a mixture between the original distribution and a set of Dirac measures at $0$. 
More precisely, this means that any sample in $\mathbf{X}^\texttt{NA}$ is distributed as the random variable $$X^\texttt{NA}:= X \odot \omega_x = \big(X^{(1)}\omega_x^{(1)},\dots, X^{(d)}\omega_x^{(d)}\big)$$, where $\omega_x = \big(\omega_x^{(1)},\dots,\omega_x^{(d)}\big) \in \{0,1\}^d$ is the random missingness, and similarly for $\mathbf{Y}^\texttt{NA}$. We denote their means by the column vectors $\mathbf{m}_x^\textnormal{\texttt{NA}}$ and $\mathbf{m}_y^\textnormal{\texttt{NA}}$, and their covariances by the matrices $\cov_x^\texttt{NA} \in \mathcal{S}^{++}_d$ and $\cov_y^\texttt{NA}\in \mathcal{S}^{++}_d$. We let 
\begin{align*}
& \mathbf{S}^{\textnormal{\texttt{NA}}}_{xn} = n^{-1}\sum_{i=1,\dots,n} \big(\mathbf{x}^{\textnormal{\texttt{NA}}}_i - \mathbf{m}^\textnormal{\texttt{NA}}_{xn}\big)\big(\mathbf{x}^{\textnormal{\texttt{NA}}}_i - \mathbf{m}^\textnormal{\texttt{NA}}_{xn}\big)^\top \\
& \mathbf{S}^{\textnormal{\texttt{NA}}}_{ym} = m^{-1}\sum_{i=1,\dots,m} \big(\mathbf{y}^{\textnormal{\texttt{NA}}}_i - \mathbf{m}^\textnormal{\texttt{NA}}_{ym}\big)\big(\mathbf{y}^{\textnormal{\texttt{NA}}}_i - \mathbf{m}^\textnormal{\texttt{NA}}_{ym}\big)^\top
\end{align*}
be the empirical covariance matrices with imputed data, where 
   $\mathbf{m}^\textnormal{\texttt{NA}}_{xn} = n^{-1}\sum_{i=1}^{n} \mathbf{x}^{\textnormal{\texttt{NA}}}_i $ and $\quad\mathbf{m}^\textnormal{\texttt{NA}}_{ym} = m^{-1} \sum_{i=1}^{m} \mathbf{y}^{\textnormal{\texttt{NA}}}_i$
are the empirical means computed again from imputed data. The empirical distributions of $\mathbf{x}^\texttt{NA}$ and $\mathbf{y}^\texttt{NA}$ are denoted by $\mu_n^\texttt{NA}$ and $\eta_m^\texttt{NA}$.
 Finally, 
let the $\psi_\alpha$-norm of a random variable be 
    $$\norm{X}_{\psi_\alpha}:= \inf\{u >0 \,:\, \mathbb{E}\exp(|X|^\alpha/u^\alpha) \leq 2\}$$
for $\alpha \geq 1$, and the $\psi_\alpha$-norm of a random vector be
    $\norm{X}_{\psi_\alpha}:= \max\limits_{u \in \mathbb{S}^{d-1}} \norm{u^\top X}_{\psi_\alpha}$. 

\begin{asm}[Sub-Gaussian data]
\label{asm:sub-Gaussian_data}
    The random variables $X,Y$ are sub-Gaussian, i.e., there exists $K_x,K_y \geq 1$ such that 
    \begin{align*}
        \norm{\big(X-\mathbf{m}_x\big)^\top u}_{\psi_2} \leq K_x\mathbb{E}\Big[\big((X-\mathbf{m}_x)^\top u\big)^2\Big] \,\, \textnormal{ and } \,\, \norm{\big(Y-\mathbf{m}_y\big)^\top v}_{\psi_2} \leq K_y\mathbb{E}\Big[\big((Y-\mathbf{m}_y)^\top v\big)^2\Big]
    \end{align*}
    for all $u,v \in \mathbb{R}^d$.
\end{asm}

\paragraph{Missing at Random.} We consider the following missingness assumption, which is commonly used to study missing values~\citep{lounici2014high,josse2020debiasing, zhu2022high, pacreau2024robust}.

\begin{asm}[Heterogenous Completely at Random Missingness (MCAR)]
\label{asm:heterogenous_mcar_and_non_zero}
    We say that the data is missing completely at random with probability vector $\mathbf{p} = (p_i)_{i=1,\dots,d} \in ]0,1]^d$ if for every column $i=1,\dots,d$ of the missingness mask $\Omega$, the entries are sampled i.i.d. from a Bernoulli distribution with parameter of success $p_i$, i.e. each entry is $1$ with probability $p_i$ and $0$ with probability $1-p_i$.
\end{asm}

In the following, we will always denote the missingness probabilities for $X$ (resp. $Y$) as $\mathbf{p}$ (resp. $\mathbf{q}$), and write $\mathbf{P} = \textnormal{diag}(\mathbf{p})$ and $\mathbf{Q} = \textnormal{diag}(\mathbf{q})$. The previous assumption ensure that these matrices are invertible, a key requirement to debiase some of the estimators we will consider. These probabilities are known, which is a common assumption in the literature~\citep{lounici2014high,klochkov2020uniform,josse2020debiasing,pacreau2024robust} but can be estimated easily.

\subsection{Optimal Transport}

\paragraph{$2$-Wasserstein Distance.} We define the 2-Wasserstein distance between two distributions $\mu$ and $\eta$ s the optimal transport cost, where the ground cost is given by the squared Mahalanobis distance parameterized by $\mathbf{M} \in \mathcal{S}_d^{+}$
\begin{align*}
    \wasserstein_\mathbf{C}(\mu,\eta) := \Big(\min\limits_{\gamma \in \Pi(\mu,\eta)} \int_{\mathbb{R}^d \times \mathbb{R}^d} \norm{x-y}^2_\mathbf{M}d\gamma(x,y)\Big)^{1/2}
\end{align*}
where the minimum is taken over couplings $\Pi(\mu,\eta)$ of $\mu$ and $\eta$, and $ \mathbf{C}(x,y) := \norm{x-y}^2_\mathbf{M} = (x-y)^\top \mathbf{M}(x-y)$. For two empirical distributions $\mu_n := \frac{1}{n}\sum_{i=1,\dots,n}  \delta_{\mathbf{x}_i}$ and $\eta_m :=\frac{1}{m}\sum_{j=1,\dots,m} \delta_{\mathbf{y}_j}$, the optimal transport problem thus writes
\begin{align*}
\tag{$\mathcal{P}$} \label{eq:vanilla_ot}
    \mathbf{W}_\mathbf{C}(\mu_n,\eta_m):= \min\limits_{\mathbf{\Pi} \in \Pi} \sum\limits_{i=1}^n \sum\limits_{j=1}^m \mathbf{\Pi}_{ij} (\mathbf{x}_i-\mathbf{y}_j)^\top \mathbf{M}(\mathbf{x}_i-\mathbf{y}_j) = \min\limits_{\mathbf{\Pi} \in \Pi} \textnormal{Tr}\big[\mathbf{\Pi}^\top \mathbf{C}\big]
\end{align*}
where $\Pi := \big\{ \mathbf{\Pi} \in \mathbb{R}_+^{n \times m}\,|\, \mathbf{\Pi}\mathbf{1}_m= \mathbf{1}_n,\, \mathbf{\Pi}^\top\mathbf{1}_n = \mathbf{1}_m \big\}$ is the set of bistochastic matrices. We let $\mathbf{\Pi}(\mathbf{C})$ be a solution to the problem \eqref{eq:vanilla_ot} with cost matrix $\mathbf{C}$.

\paragraph{Entropic Regularized Optimal Transport.} OT is often regularized to enforce structural properties on the chosen coupling, improving computational efficiency and stability. A common choice is the entropy of the coupling~\citep{peyre2019computational,genevay2019sample,feydy2019interpolating,pooladian2021entropic,keriven2022entropic} defined here as $\kl\big(\mathbf{\Pi}\big) := \sum_{ij}\mathbf{\Pi}_{ij}\log \mathbf{\Pi}_{ij} $, yielding 
\begin{align*}
     \mathbf{W}_{\mathbf{C},\varepsilon} (\mu_n,\eta_m):= \min_{\mathbf{\Pi} \in \Pi} \textnormal{Tr}\big[\mathbf{\Pi}^\top \mathbf{C}\big] + \varepsilon\kl\big(\mathbf{\Pi}\big).
\end{align*}
We denote the corresponding unique optimal coupling by $\mathbf{\Pi}_{\varepsilon} (\mathbf{C})$, where we omit the dependance on the two empirical distributions. 
\paragraph{Bures-Wasserstein Distance.} In the particular case of Gaussian distributions and Euclidean cost ($\mathbf{M} = \mathbf{I}_d$), the $2$-Wasserstein distance has an explicit formula, which gives rise to the Bures-Wasserstein distance. Assuming that $\mu = \mathcal{N}\big(\mathbf{m}_x,\cov_x\big)$ and $\eta = \mathcal{N}\big(\mathbf{m}_y,\cov_y\big)$, 
\begin{align}
\label{eq:bures_wasserstein}
    \mathbb{BW}(\mu,\eta):=\norm{\mathbf{m}_x-\mathbf{m}_y}_2^2 + \textnormal{Tr}\big(\cov_x+\cov_y\big) - 2 \textnormal{Tr}\Big[\Big(\cov_x^\frac{1}{2}\cov_y\cov_x^\frac{1}{2}\Big)^{1/2}\Big].
\end{align}
This distance is closely linked to Monge maps.
\begin{defn}
    The Monge map $\mathbf{T}^\star: \mathbb{R}^d\to \mathbb{R}^d$ is defined as a solution of 
    \begin{align}
    \label{monge_problem} \tag{$\mathcal{P}_\textnormal{Monge}$}
        \min\limits_{T:\mathbb{R}^d \to \mathbb{R}^d,\, T_\sharp \mu = \eta} \, \int \norm{x-T(x)}_2^2 d\mu(x).
    \end{align}
\end{defn}

The Monge map between Gaussians is
$
    \mathbf{T}_\ell^\star(x) := \mathbf{m}_x + \mathbf{A}^\star\big(\cov_x,\cov_y\big)\big(x-\mathbf{m}_y\big)
$
where $\mathbf{A}^\star\big(\cov_x,\cov_y\big) := \cov_x^{1/2}\big(\cov_x^{1/2}\cov_y\cov_x^{1/2}\big)^{1/2}\cov_x^{1/2}$ \citep[Remark 2.31]{peyre2019computational}. The value of \eqref{monge_problem} at optimum is precisely the Bures-Wasserstein distance \eqref{eq:bures_wasserstein}.

\paragraph{Linear Monge Maps.} Going beyond Gaussian distributions, \citet[Lemma 1]{flamary2019concentration} show that if the set of mappings is restricted to be linear, then the Monge map between distributions --- if it exists --- is identical to the Monge map between Gaussians. More precisely, if we let 
$\mathcal{T}_\textnormal{linear} =\big\{T:\mathbb{R}^d \to \mathbb{R}^d \,|\, T_\sharp \mu = \eta,\, T(x)= \mathbf{A}x + \mathbf{b}, \, \mathbf{A} \in \mathcal{S}_d^{++}\big\}$
and require $T \in \mathcal{T}_\textnormal{linear}$, then $\mathbf{T}^\star = \mathbf{T}_\ell^\star $.

\section{Bias and Implicit Regularization of Optimal Transport with Missing Data}
\label{section:implicit_bias}
As shown in Figure \ref{fig:illustration_ot_na}, handling missing values through naive imputation distorts the geometry of the optimal transport problem. In this section, we rigorously quantify this effect by building on recent results in linear regression with missing data~\citep{ayme2023naive}, showing that naive imputation introduces an \textit{implicit regularization} into the optimal transport formulation. Let
\begin{align*}
    \overline{\mathbf{\Pi}}_\mathbf{M} \in \argmin\limits_{\mathbf{\Pi} \in \Pi} \mathbb{E}\Big[\sum\limits_{i=1}^n \sum\limits_{j=1}^m \mathbf{\Pi}_{ij} \norm{\mathbf{x}_i^{\texttt{NA}}-\mathbf{y}_i^{\texttt{NA}}}^2_\mathbf{M}\,|\, \mathbf{X},\mathbf{Y}\Big].     
\end{align*}
\begin{prop}
\label{prop:implicit_reg_ot}
    Let $\overline{\mathbf{M}} := \mathbf{M} - \bigp\mathbf{M}\bigq = [(1-p_iq_j)m_{ij}]_{i,j=1}^{d}$. The transport map $ \overline{\mathbf{\Pi}}_\mathbf{M}$ solves the optimal transport problem 
    \begin{align*}
     \min\limits_{\mathbf{\Pi} \in \Pi} \textnormal{Tr}\Big[\mathbf{\Pi}^\top \big(\mathbf{C} - \mathbf{X}^\top \overline{\mathbf{M}}\mathbf{Y} \big)\Big].
    \end{align*}
\end{prop}

This proposition shows that, on average, naive imputation implicitly regularizes the cost function of the optimal transport problem by reducing the cost of matchings that are both geometrically aligned---according to the metric defined by $\mathbf{M}$---and infrequently observed, that is, when $\mathbf{x}_i^\top \overline{\mathbf{M}}\mathbf{y}_i > 0$. Taking this analysis a step further, we also derive a lower bound on the Wasserstein distance computed from incomplete data.
\begin{prop}
\label{prop:bias_ot_na}
    Assume for the sake of clarity that $\cov_x = \textnormal{diag}(\sigma^2_1,\dots,\sigma^2_d)$ with $\sigma_i >0$ for all $i=1,\dots,d$ and $\mathbf{M} = \textnormal{diag}(m_1,\dots,m_d)$.  Then if $\mathbf{P} \neq \mathbf{I}_d$
    \begin{align*}
        \wasserstein_2^2(\mu^\textnormal{\texttt{NA}},\mu) \geq \norm{(\mathbf{P}-\mathbf{I}_d)\mathbf{m}_x}_{\mathbf{M}}^2 + \mathscr{B}_\mathbf{M}^2(\mathbf{P}) > 0
    \end{align*}
    where $ \mathscr{B}_\mathbf{M}^2(\mathbf{P}) :=  \sum\limits_{i=1}^d m_{i}\Big[\sigma_i - \sqrt{p_i}\sqrt{\sigma^2_i + (1-p_i)(\mathbf{m}_x)^2_{i}}\Big]^2$.
\end{prop}
This result shows in particular that $\wasserstein_2^2(\mu^\textnormal{\texttt{NA}},\mu) > \wasserstein_2^2(\mu,\mu)=0$ if at least one variable $i$ is not perfectly observed \textit{and} $m_{ii} >0$, which highlights the importance of the interaction between $\mathbf{P}$ and $\mathbf{M}$. As expected, this lower bound collapses to $0$ when $\mathbf{P} \to \mathbf{I}_d$. Crucially, even if $\mathbf{m}_x = \mathbf{0}$ (i.e., naive imputation is unbiased), the Wasserstein distance is still bounded away from $0$ due to the difference in variance induced by constant imputation. These results motivate our next contributions, in which we develop estimators of the Wasserstein distance that are \textit{asymptotically invariant} to missing data---that is, they converge to the true value of $\wasserstein_\mathbf{M}(\mu,\eta)$ as $n,m\to \infty$.

\section{Bures-Wasserstein Distances with Missing Data}

\label{section:linear_monge}

\subsection{Estimating the Bures-Wasserstein Distance with Missing Values}

We build upon the literature on covariance matrix estimation with missing data~\citep{lounici2014high, klochkov2020uniform, pacreau2024robust} to develop a plug-in estimator of the Bures-Wasserstein distance. Building on previous work by~\citet{lounici2014high},~\citet{pacreau2024robust} propose the estimator
\begin{align*}
    \mathbf{P}^{-1}\big(\mathbf{I}_d-\mathbf{P}^{-1}\big) \textnormal{diag}\big(\mathbf{S}^{\textnormal{\texttt{NA}}}_{xn}\big) + \mathbf{P}^{-1}\mathbf{S}^{\textnormal{\texttt{NA}}}_{xn}\mathbf{P}^{-1}, \textnormal{ where }\mathbf{S}^{\textnormal{\texttt{NA}}}_{xn} = n^{-1}\sum \big(\mathbf{x}^{\textnormal{\texttt{NA}}}_i - \mathbf{m}^\textnormal{\texttt{NA}}_{xn}\big)\big(\mathbf{x}^{\textnormal{\texttt{NA}}}_i - \mathbf{m}^\textnormal{\texttt{NA}}_{xn}\big)^\top
\end{align*}
and $\mathbf{m}^\textnormal{\texttt{NA}}_{xn} := n^{-1} \sum \mathbf{x}_i^\texttt{NA}$ which is unbiased and consistent. As an intermediate step, we extend this result to a more general setting where the distribution has unknown mean. In this case, our estimator reads 
\begin{align*}
    \widehat{\cov}_x := \mathbf{P}^{-1}\big(\mathbf{I}_d-\mathbf{P}^{-1}\big) \textnormal{diag}\big(\mathbf{S}^{\textnormal{\texttt{NA}}}_{xn}\big) + \mathbf{P}^{-1}\mathbf{S}^{\textnormal{\texttt{NA}}}_{xn}\mathbf{P}^{-1}+ \mathbf{P}^{-1}\big(\mathbf{I}_d-\mathbf{P}^{-1}\big) \textnormal{diag}\big(\mathbf{m}_{xn}^{\textnormal{\texttt{NA}}} \mathbf{m}_{xn}^{\textnormal{\texttt{NA}}\top}\big).
\end{align*}

Let $\mathbf{r}(\cov):= \textnormal{Tr}\big(\mathbf{\Sigma}\big)\norm{\mathbf{\Sigma}}_\textnormal{op}^{-1}$ be the effective dimension of a matrix $\cov$.
\begin{lem}
\label{lem:extension_pacreau}
Under Assumptions \ref{asm:heterogenous_mcar_and_non_zero} and \ref{asm:sub-Gaussian_data}, for any $t>0$, and up to a term of order $\mathcal{O}(n^{-1})$, with probability at least $1-4e^{-t}$
\begin{align*}
    \Big\lVert  \widehat{\cov}_x - \cov_x \Big\rVert_\textnormal{op} \lesssim \norm{\cov_x}_\textnormal{op} \norm{\mathbf{P}}_\textnormal{op}^{-1}\Bigg(\sqrt{\frac{\mathbf{r}(\cov_x) \log \mathbf{r}(\cov_x)}{n}} \vee \sqrt{\frac{t}{n}} \vee \frac{\mathbf{r}(\cov_x)\big(t+\log \mathbf{r}(\cov_x)\big)}{\norm{\mathbf{P}}_\textnormal{op}n} \log n\Bigg).
\end{align*}
\end{lem}

From now on, we denote by $\mathbf{\Delta}_n\big(\mu^\texttt{NA},t\big)$ the r.h.s. of this inequality. The term in $\mathcal{O}(n^{-1})$ is explicitly given in the proof and only depends on the missingness probabilities and the first and second moments of the measure $\mu$ through the effective dimension of $\cov_x$. An analogous result holds for $\widehat{\cov}_y$.  This bound shows that estimating the mean and the covariance simultaneously does not significantly affect the sample complexity since the dominant term in $\mathcal{O}(n^{-1/2})$ is identical to~\citet{pacreau2024robust}.

Returning to the Bures-Wasserstein distance, we observe that it depends only on the means and covariance matrices of the distributions. This allows us to leverage the previously introduced covariance estimator to construct a consistent estimator of the Bures-Wasserstein distance.
\begin{align*}
    \widehat{\mathbb{BW}}(\mu_n^\textnormal{\texttt{NA}},\eta^\textnormal{\texttt{NA}}_m):= \norm{\bigp^{-1}\mathbf{m}^\texttt{NA}_{nx}-\bigq^{-1}\mathbf{m}^\texttt{NA}_{my}}_2^2 + \textnormal{Tr}\big(\widehat{\cov}_x+\widehat{\cov}_y\big) - 2 \textnormal{Tr}\Big[\Big(\widehat{\cov}_x^\frac{1}{2}\widehat{\cov}_y\widehat{\cov}_x^\frac{1}{2}\Big)^{1/2}\Big].
\end{align*}
For the sake of clarity, we group all terms of order lower than $n^{-1/2}$ into $\mathcal{O}\big(n^{-1},m^{-1}\big)$.

\begin{thm}
\label{thm:concentration_bw}
    Under Assumptions \ref{asm:heterogenous_mcar_and_non_zero} and \ref{asm:sub-Gaussian_data}, for any $t>0$, and up to supplementary terms in $\mathcal{O}(n^{-1})$ and $\mathcal{O}(m^{-1})$, the following holds with probability at least $1-10e^{-t}$: 
    \begin{align*}
        \Big\lvert \widehat{\mathbb{BW}}(\mu_n^\textnormal{\texttt{NA}},\eta^\textnormal{\texttt{NA}}_m)-\mathbb{BW}(\mu,\eta)\Big\rvert \lesssim &  \frac{A_1}{nm}\big(t \vee t^{\frac{1}{2}}\big)+ A_2(t) \Big(\sqrt{\frac{t}{m}} \vee \frac{t}{m}\Big)
        +  A_3(t)\Big(\sqrt{\frac{t}{n}} \vee \frac{t}{n}\Big),
    \end{align*}
    where 
     \begin{align*}
        & A_1(t):= \norm{\mathbf{P}^{-1}}_\textnormal{op}\norm{\mathbf{Q}^{-1}}_\textnormal{op}\\
        & A_2(t) := \norm{\mathbf{Q}^{-1}}_\textnormal{op}\Big(2+\norm{\cov_{x}}_\textnormal{op} +\mathbf{\Delta}_m\big(\eta^\texttt{NA},t\big) \Big)\\
        \textnormal{and} \, & A_3(t) :=  \norm{\mathbf{P}^{-1}}_\textnormal{op}\Big(2+\norm{\cov_y}_\textnormal{op}\Big[\norm{\cov_{x}}_\textnormal{op} +\mathbf{\Delta}_n\big(\mu^\texttt{NA},t\big) \Big]\lambda^{-1}_\textnormal{min}\big(\cov_x^\frac{1}{2}\big)\Big).
    \end{align*}
\end{thm}

Remarkably, our bound is \emph{entirely dimension-free}, including the terms in $\mathcal{O}(n^{-1},m^{-1})$. We obtain the same convergence rates as in the complete data setting \citep{klochkov2020uniform}. The price of missingness remains reasonably small, affecting the convergence rate of our estimator only up to constants depending on $\norm{\mathbf P^{-1}}_\textnormal{op}$ and  $\norm{\mathbf Q^{-1}}_\textnormal{op}$. However, these constants can become arbitrary large as the probability of observing a variable tends to $0$. Our bounds can be extended to entropic regularized optimal transport between Gaussians using closed form formulas available in the literature \citep{del2020statistical}. We use this extension in our experiments.

\subsection{Application to Domain Adaptation with Missing Data}
\label{sec:da}
\looseness=-1 Our Bures-Wasserstein estimator enables effective domain adaptation in the presence of missing data.
We consider a setup inspired by \citet{courty2016optimal} and \citet{flamary2019concentration}. We have access to a model $\widehat{f}_{n_\ell} : \mathbb{R}^d \to \mathcal{Y}$ pretrained by empirical risk minimization on a \emph{source} domain, which maps points $\mathbf{x} \in \mathbb{R}^d$ to predicted outcomes $y \in \mathcal{Y}$. This model is given \textit{without its training set}, a setup often referred to as hypothesis transfer \citep{kuzborskij2013stability,aghbalou2023hypothesis} which has become increasingly prevalent in the era of large models. We have access to two datasets: a curated unlabeled dataset $\mathbf{X}_t \in \mathbb{R}^{n_t \times d}$ from a \emph{target} distribution $\mathcal{P}_t$ (different from the source distribution $\mathcal{P}_s$) on which we wish to apply the model, and a small unlabeled dataset $\mathbf{X}^\texttt{NA}_s\in \mathbb{R}^{n_s \times d}$ with $n_s \ll n_t$ that is similar to the pre-training dataset from the source domain but whose data is partially missing. This may arise, for instance, when the model trainer prefers not to release their curated high-quality data, or when some individuals have partially or fully opted out of a data release agreement. In the latter case, individuals who have fully opted out are excluded from the dataset, while those who have agreed to a partial release have their undisclosed features reported as missing. As in \citet{flamary2019concentration}, we make the following assumption to connect the two distributions. 

\begin{asm}
\label{asm:mapping_shift}
The source and target distributions satisfy $\mathcal{P}_t = m_\sharp \mathcal{P}_s$
where $m(x,y) = (\mathbf{T}_\ell^\star(x),y)$.
\end{asm}

We assume similar missingness mechanisms as above, and consider an empirical risk minimization framework for a classification task. Let $\mathcal{H}_K$ be a reproducing Hilbert space associated with a symmetric kernel $K:\mathbb{R}^d \times \mathbb{R}^d \to \mathbb{R}$ such that for all $x \in \mathbb{R}^d$, $K_x(\cdot)=K(\cdot,x) \in \mathcal{H}_K$ and $f(x) = \langle f(x), K_x \rangle_{\mathcal{H}_K} $ for all $f \in \mathcal{H}_k$. Let 
\begin{align*}
    \mathscr{R}_s(f) := \mathbb{E}_{\mathcal{P}_s}\Big[\mathscr{L}\big(Y,f(X)\big)\Big] \quad \textnormal{and} \quad \mathscr{R}_t(f) := \mathbb{E}_{\mathcal{P}_t}\Big[\mathscr{L}\big(Y,f(X)\big)\Big] 
\end{align*}
be the expected risk of a predictor $f \in \mathcal{H}_K$ in the source domain and target domain. We borrow the following assumptions from \citet{flamary2019concentration}. The function $\mathscr{L}:\mathcal{Y}\times \mathcal{Y} \to \mathbb{R}_+$ is a loss function which is $M_\mathscr{L}$-Lipschitz with respect to its second variable. Assume that the Bayes predictors lie in $\mathcal{H}_K$, i.e., 
    $$f_\star^s := \argmin_{f:\mathbb{R}^d \to \mathbb{R}} \mathscr{R}_s(f) \in \mathcal{H}_{K} \quad \textnormal{and} \quad  f_\star^t := \argmin_{f:\mathbb{R}^d \to \mathbb{R}} \mathscr{R}_t(f) \in \mathcal{H}_{K}$$
and additionally satisfy $\norm{f_\star^s}_{\mathcal{H}_K} \leq 1$ and $\norm{f_\star^t}_{\mathcal{H}_K} \leq 1$. We assume that the original model has been trained on a possibly large dataset with $n_\ell$ labeled examples, that is
\begin{align*}
    \widehat{f}_{n_\ell} \in \argmin\limits_{f \in \mathcal{H}_K} \frac{1}{n_\ell}\sum\limits_{i=1}^{n_\ell} \mathscr{L}\big(f(\mathbf{x}_i),y_i\big).
\end{align*}
Finally, assume that the eigenvalues $(\zeta_k)_{k \geq 0}$ of the integral operator $T_K$ of $\mathcal{H}_K$ decay at rate $k^{-2\beta}$.

\begin{thm}
\label{thm:DA}
    Assume the Bayes predictors $f^t_\star,f^s_\star$ satisfy $\mathscr{R}_t(f^t_\star) = \mathscr{R}_s(f^s_\star)$ i.e. the Bayses risk is equal in both domains. For any $t>0$ and $n_s,n_t$ large enough, up to terms in $\mathcal{O}(n_s^{-1})$ and $\mathcal{O}(n_t^{-1})$, we have with probability at least $1-4e^{-t}$:
    \begin{align*}
         \mathscr{R}_t(\widehat{f}_{n_\ell} \circ \widehat{\mathbf{T}}) - \mathscr{R}_t(f^t_\star)\lesssim n_\ell^{-\frac{2\beta}{1+2\beta}} + \frac{t}{n_\ell}+\frac{B_1(t)}{\sqrt{n_s}} + \frac{B_2(t)}{\sqrt{n_t}} + B_3(t)\mathbf{\Delta}_{n_s}(\mu^\texttt{NA},t) + B_4(t)\mathbf{\Delta}_{n_t}(\eta,t)
    \end{align*}
    where  
        \begin{align*}
        & B_1(t) := M_f M_\mathscr{L}\norm{\mathbf{P}^{-1}}_\textnormal{op}\sqrt{K_2^x\Big(1+ 2\sqrt{K_2^x \norm{\mathbf{P}}_{\textnormal{op}}\norm{\cov_x}_\textnormal{op} t }+2t\Big)}\\
        & B_2(t) := M_f M_\mathscr{L}\norm{\cov_t}^{\frac{1}{2}}\Big(\mathbf{r}\big(\cov_t\big)+2 \sqrt{\cov_t t}+2t\Big)^{\frac{1}{2}}\\
        & B_3(t):= \frac{ M_f M_\mathscr{L}K^x_4\kappa(\cov_t) \norm{\cov_t}_\textnormal{op}}{\sqrt{\lambda_\textnormal{min}\big(\cov^{-1/2}_{n_s}\cov_{n_t}\cov_{n_s}^{-1/2}\big)}} \\
        \textnormal{and } & B_4(t) := M_f M_\mathscr{L}\frac{\kappa(\cov_s)}{\sqrt{\lambda_\textnormal{min}\big(\cov^\frac{1}{2}_{s}\cov_{t}\cov_{s}^\frac{1}{2}\big)}}. 
    \end{align*}

\end{thm}

This theorem adapts \citet[Theorem 2]{flamary2019concentration} to the presence of missing data. As for the estimation of covariance matrices, we obtain similar rates - up to missingness dependent constants - as in the complete data case.

\section{Consistent Estimation of Entropic Optimal Transport with Missing Data}
\label{section:generalized_estimator}

\subsection{Consistent Estimation Through Matrix Completion}

We now move beyond the Gaussian and linear Monge cases, and address the general entropic-regularized optimal transport problem. We propose a two-step algorithm: first impute  missing values in $\mathbf{X}^\texttt{NA}$ and $\mathbf{Y}^\texttt{NA}$, then solve optimal transport on the completed data. We consider a noisy setting:
\begin{align}
    \tag{Noisy Data}
    \label{eq:noiseless_ot}
    \mathbf{X}^\texttt{NA} = \big(\mathbf{X}+\mathbf{N}_x\big)\odot \Omega_x \textnormal{ and } \mathbf{Y}^\texttt{NA} = \big(\mathbf{Y}+\mathbf{N}_y\big)\odot \Omega_y  ,
\end{align}
in which $\mathbf{N}_x,\mathbf{N}_y$ are two noise matrices. We make the following assumption on the noise.
\begin{asm}
\label{asm:noise_asm}
    The entries of $\mathbf{N}_x = (\xi^x_{ij})_{i,j}$ and $\mathbf{N_y}=(\xi^y_{ij})_{ij}$ are i.i.d. and satisfy $\mathbb{E}(\xi^x_{ij}) = \mathbb{E}(\xi^y_{ij}) = 0$, and $\mathbb{E}\big((\xi^x_{ij})^2\big) \leq b$, $\mathbb{E}\big((\xi^y_{ij})^2\big) \leq b$ for all $i,j$.
\end{asm}
Our approach leverages iterative singular value thresholding (ISVT)~\citep{mazumder2010spectral,klopp2015matrix}, a celebrated spectral algorithm for noisy matrix completion, combined with recent cost-sensitivity bounds on entropic regularized optimal transport~\citet{keriven2022entropic}, to derive a consistent estimator of the entropic-regularized Wasserstein distance.
ISVT solves the matrix completion problem by recovering a low-rank approximation of the full matrix from partial observations with a single hyperparameter $\lambda \geq 0$---the full algorithm is given in Appendix \ref{appendix:experiments}. In our setting, we apply ISVT independently to impute the missing entries in $\mathbf{X}^\texttt{NA}$ and $\mathbf{Y}^\texttt{NA}$, and define $\widehat{\mathbf{C}}$ as the resulting cost matrix. We make two additional assumptions on the data. 
\begin{asm}
\label{asm:bounded_support}
    The measures $\mu$ and $\eta$ have bounded support, that is $\norm{X}_2 \leq \mathcal{R}_x$ and $\norm{Y}_2 \leq \mathcal{R}_y$ almost surely, and their covariance matrices have dimensions $k_\mu \leq d$ and $k_\eta \leq d$ respectively. 
\end{asm}

Under this Assumption, the cost matrix $\mathbf{C}$ has bounded entries. We let $c_\textnormal{max}$ and $c_\textnormal{min}$ be its maximal and minimal entries. 

\begin{thm}
\label{thm:main_thm_ot}    
Under Assumptions \ref{asm:heterogenous_mcar_and_non_zero} and \ref{asm:bounded_support}, we have with high probability that
\begin{align*}
 & \Big\lvert \wasserstein_{\widehat{\mathbf{C}},\varepsilon}(\mu^\texttt{NA}_n,\eta^\texttt{NA}_m)  - \wasserstein_{\mathbf{C},\varepsilon}(\mu_n,\eta_m) \Big\rvert 
    \lesssim K_\varepsilon\Bigg(\sqrt{\frac{\opnorm{\mathbf{M}}k_x\rho_\mu(n)}{\opnorm{\bigp^{-1}}n}}+\sqrt{\frac{\opnorm{\mathbf{M}}k_y \rho_\eta(m)}{\opnorm{\bigq^{-1}}m}}.\Bigg)
\end{align*}
where $\rho_n(\mu) := (\mathcal{R}_x \vee \sigma)^2 \opnorm{\bigp} + \mathcal{R}_x^2\log(n \wedge d)+b^2\log(n+d)$, $\rho_m(\eta) := (\mathcal{R}_y \vee \sigma)^2 \opnorm{\bigq} + \mathcal{R}_y^2\log(m \wedge d)+b^2\log(m+d)$
and $K_\varepsilon := \exp(\varepsilon^{-1}(2c_\textnormal{max}-c_\textnormal{min}))$. Furthermore, 
\begin{align*}
    \kl\big[\mathbf{\Pi}_\varepsilon(\mathbf{C})\,|\,\mathbf{\Pi}_\varepsilon(\widehat{\mathbf{C}}) \big] \lesssim \varepsilon^{-1}\frac{M_x}{n^{1/4}}\Big(K'_\varepsilon+K_\varepsilon\frac{M_x}{n^{1/4}}\Big) + \varepsilon^{-1}\frac{M_y}{m^{1/4}}\Big(K'_\varepsilon+K_\varepsilon\frac{M_y}{m^{1/4}}\Big)
\end{align*}
where $K_\varepsilon' := \exp\Big(\frac{1}{2 \varepsilon}(3c_\textnormal{max}-7c_\textnormal{min})\Big)$, $M_x:= \Big(\frac{\opnorm{\mathbf{M}}k_x \rho_\mu(n)}{\opnorm{\bigp^{-1}}}\Big)^{1/4}$ and $M_y := \Big(\frac{\opnorm{\mathbf{M}}k_y \rho_\eta(m)}{\opnorm{\bigq^{-1}}}\Big)^{1/4}$.
\end{thm}

The first part of our theorem shows that the consistency of the estimator $\widehat{\mathbf{C}}$ translates to the consistency of the optimal transport cost. As expected, our bound is additive in both sample sizes $n$ and $m$. Remarkably, the second part of our theorem shows that our estimator of the transport plan is also asymptotically consistent. Remark, however, that its convergence speed in $\mathcal{O}(n^{-1/4})$ is considerably slower than estimation of the Wasserstein distance---a fact already noticed by \citet{keriven2022entropic}.

\subsection{Cross-Validation Free Hyperparameter Selection Using the Bures-Wasserstein Criterion}

\begin{wrapfigure}{l}{0.5\textwidth}
  \includegraphics[width=0.5\textwidth]{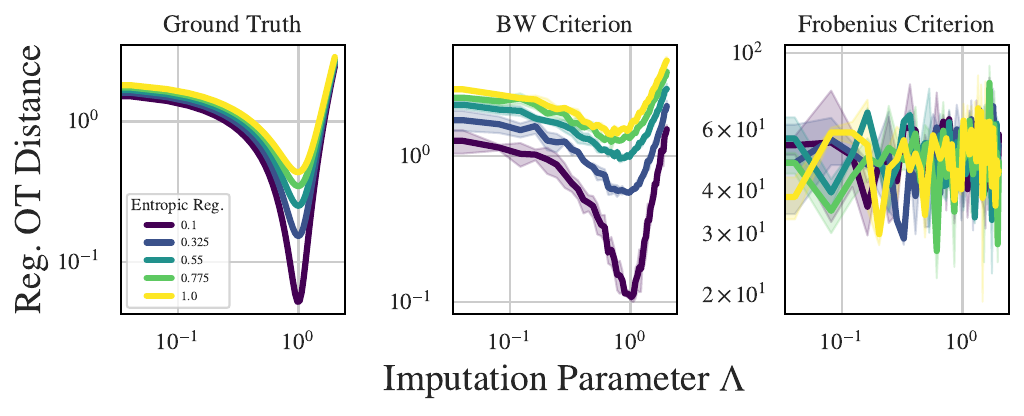}
  \caption{Value of the entropic regularized optimal transport problem as a function of $\Lambda \geq 0$ when $90\%$ of observations are missing.}
  \label{fig:illustration_bw_criterion}
  \vspace{-0.3cm}
\end{wrapfigure}

Like other matrix completion algorithms, ISVT depends on hyperparameters (specifically, $\lambda$) that require careful tuning for optimal performance.
Hyperparameter selection is generally carried out as follows: 1) a subset of the observed entries, which we call \textit{validation entries}, is selected and removed, leaving remaining \textit{training entries}; 2) the matrix completion algorithm is run on the \textit{training entries} only; 3) a validation criterion is evaluated, either on the validation entries alone or on the full matrix). Cross-validation can be applied on top of this procedure, partitioning the observed entries into multiple folds and rotating the validation set across folds. In any case, this approach introduces a sharp trade-off between the sizes of validation and training sets, since removing validation entries reduces the amount of data available for training and may distort the structure of the matrix we ultimately aim to complete. 

We propose a novel hyperparameter selection strategy for matrix completion that leverages our estimator of the Bures-Wasserstein criterion. Instead of setting aside a validation set, we compute—for each candidate hyperparameter—the estimated Bures-Wasserstein criterion between the imputed data and the original data with missing values. Missingness probabilities are estimated from the data. Since the Bures-Wasserstein score (and, by extension, our debiased estimator) provides a lower bound on the true optimal transport distance~\citep{gelbrich1990formula}, it serves as a meaningful proxy for tracking variations in the true distance. Formally speaking, for a matrix completion algorithm $\mathscr{A}_\Lambda: (\mathbb{R}\cup \{\texttt{NA}\})^{n \times d} \to \mathbb{R}^{n \times d}$ with hyperparameters $\Lambda \in \mathcal{G}_\Lambda \subset \mathbb{R}^h$, our procedure operates by selecting
\begin{align*}
    \Lambda_\star \in \argmin\limits_{\Lambda \in \mathcal{G}_\Lambda} \widehat{\mathbb{BW}}\Big[\mathscr{A}_\Lambda\Big(\mathbf{X}^\texttt{NA}\Big), \mathbf{X}^\texttt{NA}\Big].  
\end{align*}
Intuitively, this strategy favors methods which impute data in a distributionally consistent manner. Figure \ref{fig:illustration_bw_criterion} compares the true value of the entropic optimal transport problem with our Bures-Wasserstein proxy and the standard Frobenius criterion (computed on observed entries) on toy data, displaying the strong consistency of our criterion (see Appendix \ref{appendix:experiments} for experimental details).  

\section{Experiments}
\label{section:applications}

\subsection{Convergence of our estimator of the Bures-Wasserstein Distance}

\label{appendix:convergence_study}

\begin{figure}
    \centering
    \includegraphics[width=0.3\textwidth]{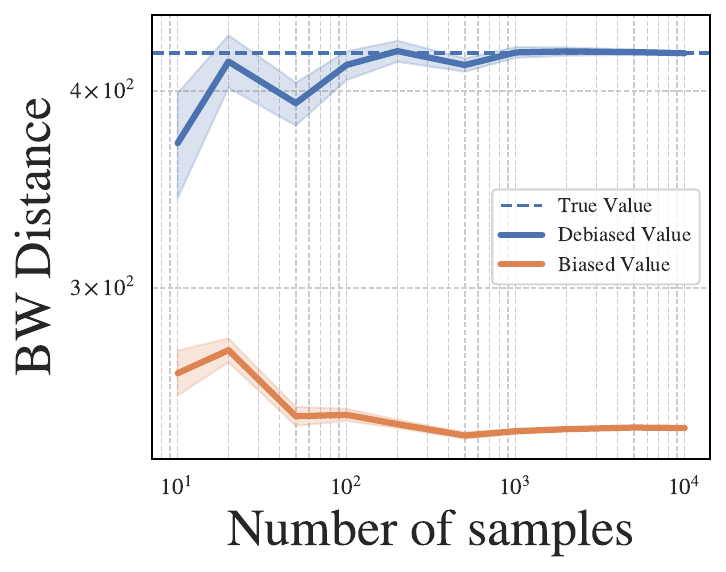}
    \includegraphics[width=0.3\textwidth]{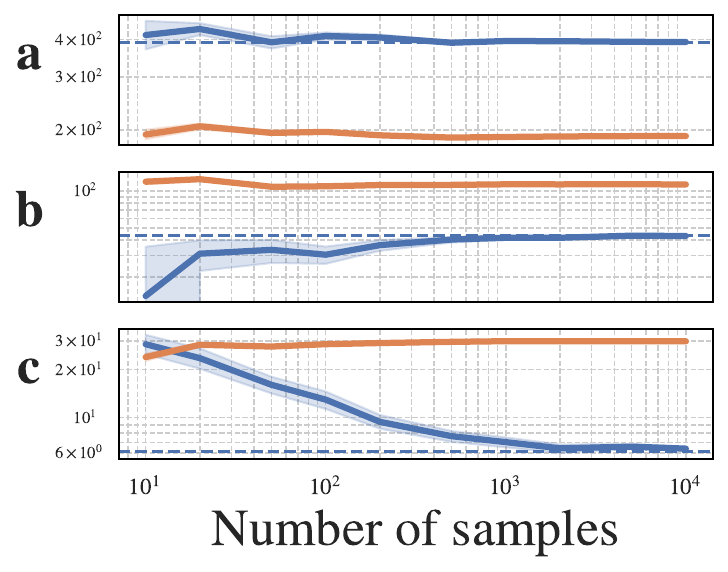}
    \includegraphics[width=0.28\textwidth]{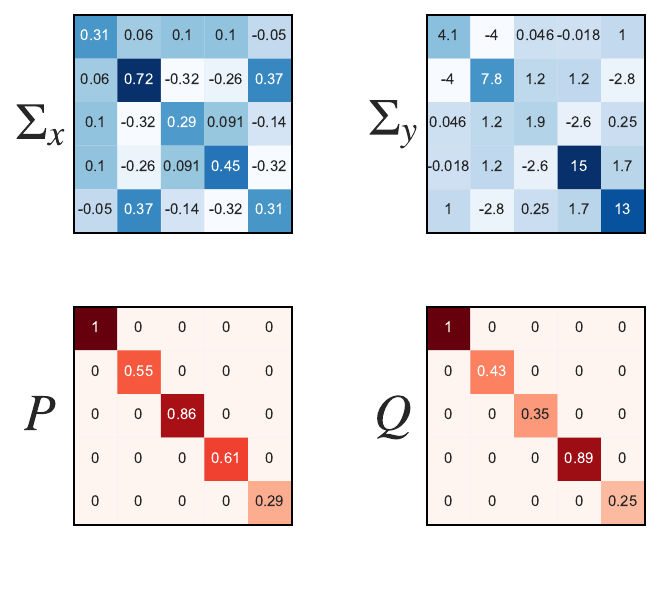}
    
    \caption{Convergence of our estimator compared to the biased estimator. We sample two fixed datasets with given covariance and mean. For every sample size, we average the estimated distance over $50$ missingness masks. \textbf{Left}: convergence of the Bures-Wasserstein distance estimator. \textbf{Center}: separate convergence for all three terms.}
    \label{fig:bw_convergence}
\end{figure}

We first analyze the empirical convergence of our Bures-Wasserstein distance estimator. We generate sample from a $5$ dimensional Gaussian, whose covariance matrices are displayed in Figure \ref{fig:bw_convergence}. The results are averaged over $50$ runs. Over all runs, the parameters of the Gaussian distributions and the missingness probabilities are kept fixed, but a new missingness mask is drawn at every run. The probabilities of missingness are drawn at random, and displayed in the same figure. 

\begin{figure}
    \centering
    \includegraphics[width=0.8\linewidth]{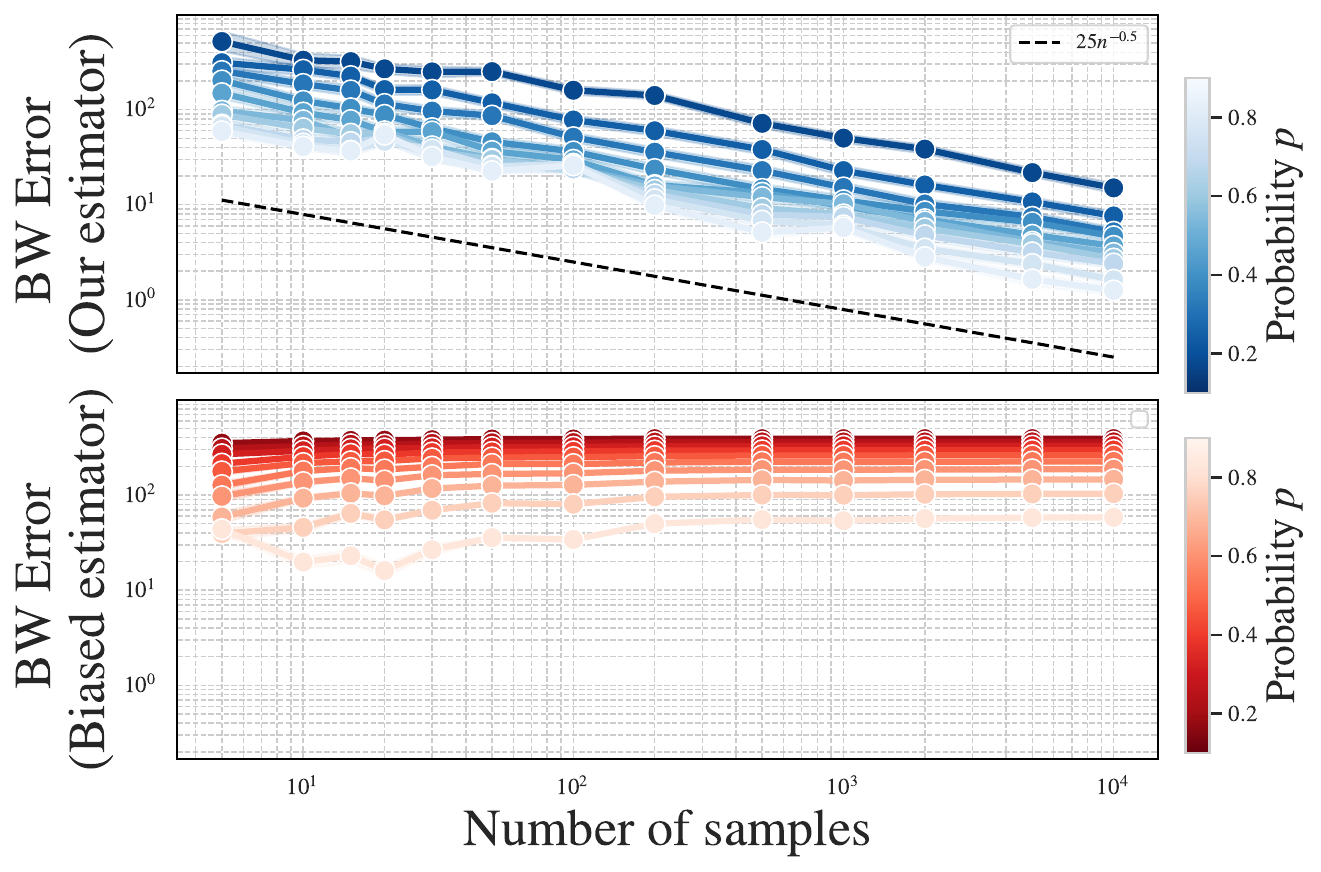}
    \caption{Convergence of our estimator over different uniform missingness probabilities. For every value of $p$, we set $\mathbf{p} = (p,\dots, p)$ and $\mathbf{q} = (p,\dots,p)$. Lighter colors correspond to higher $p$ and hence to less missing variables.}
    \label{fig:convergence_diff_proba}
\end{figure}

In Figure \ref{fig:convergence_diff_proba}, we additionally display the convergence of our estimator for different missingness probabilities. As predicted by our theoretical results, our estimator conserves the predicted convergence speed over missingness probabilities, which only affect its convergence by a constant. In this figure, the results are averaged over $100$ runs. 

\subsection{Robustness to MCAR Contamination}

\label{appendix:robustness_mnar}

\begin{figure}[h!]
    \centering
    \includegraphics[width=0.3\textwidth]{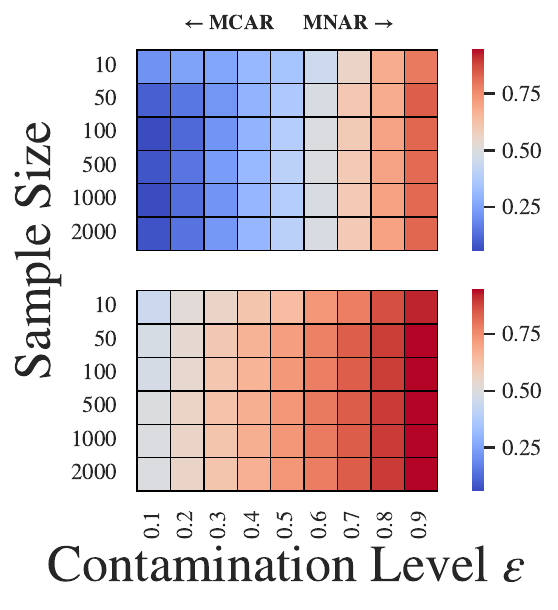}
    \caption{Robustness to MNAR data of our estimator (on \textbf{top}) vs. the biased estimator (on \textbf{bottom}). The colors of the heatmap indicate the error in Bures-Wasserstein estimation.}
    \label{fig:mnar}
\end{figure}

We also evaluate the robustness to data-dependent missingness by drawing on the recent contamination model proposed by~\citet{ma2024estimation}. More precisely, for every individual $i=1,\dots,n$ we draw the mask $\omega_x$ from the hierarchical distribution 
\begin{align*}
    \mathds{1}_{\{Z \geq \varepsilon\}}\textnormal{MCAR}(\mathbf{p}) + \mathds{1}_{\{Z \leq \varepsilon\}}\textnormal{MNAR}(\phi)
\end{align*}
were $Z \sim \mathcal{U}\big([0,1]\big)$. Concretely, this means that with probability $1-\varepsilon$, we draw the mask from a MCAR distribution with parameter $\mathbf{p} \in ]0,1[^d$, and with probability $\varepsilon$ we draw it from a data-dependent distribution, where $\phi:\mathbb{R}^d \to ]0,1]^d$ is set to be the component-wise sigmoid function. The results are displayed in Figure \ref{fig:mnar}. 

\subsection{Domain Adaptation with MNAR and MCAR Data}

\label{appendix:domain_adaptation}

To illustrate our results obtained in section \ref{sec:da}, we consider datasets randomly generated using the build-in $\texttt{make\_classification}$ module from \texttt{sklearn} in dimension $d=5$. For every of the $100$ generated dataset, we train a logistic classifier on $n_\ell = 5000$ labeled pairs. We wish to classify samples from an unlabeled dataset for which $n_t = 6000$. To align our data, we have access to $n_s = 200$ unlabeled samples. Every feature is observed with probability $0.5$. We simulate two types of shifts. We first consider a linear shift: we sample a random matrix $\mathbf{A} \in \mathbb{R}^{d \times d}$ and a random vector $\mathbf{b} \in \mathbb{R}^{d}$ and transform every target sample $\mathbf{x}_{i,t}$ through $\mathbf{A}\mathbf{x}_{i,t}+\mathbf{b}$. We then consider a non-linear transformation $\cos(\mathbf{x}_{i,t})$. The results are displayed in Figure \ref{fig:da}.

\begin{figure}
        \centering
        \includegraphics[width=0.45\textwidth]{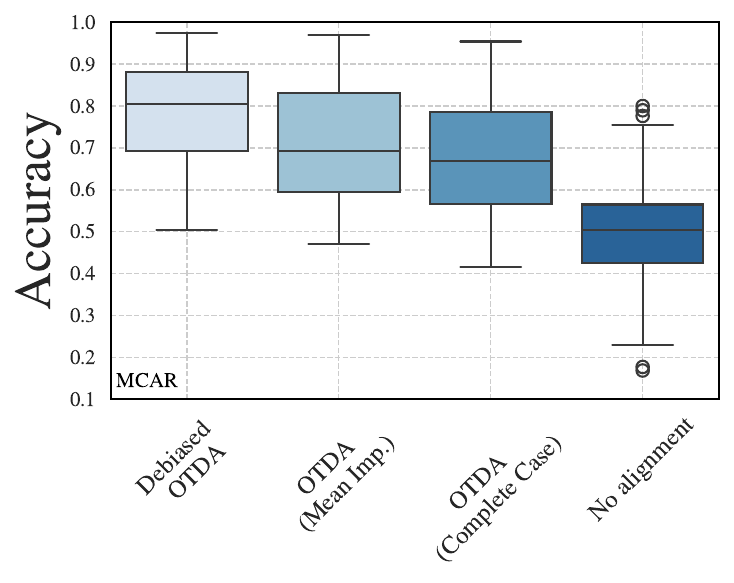}
        \includegraphics[width=0.45\textwidth]{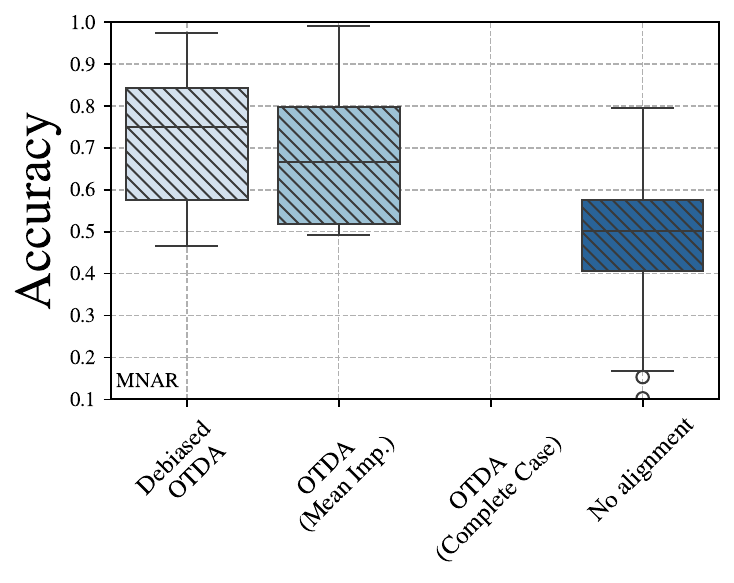}
        \includegraphics[width=0.45\textwidth]{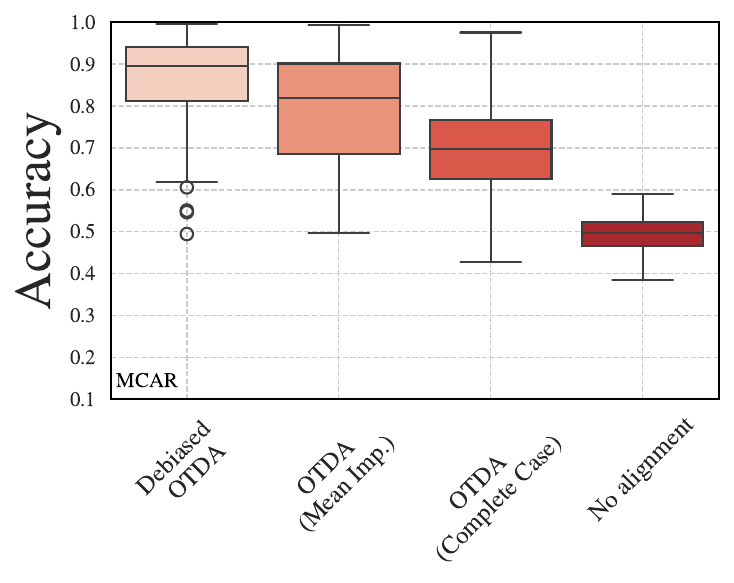}
        \includegraphics[width=0.45\textwidth]{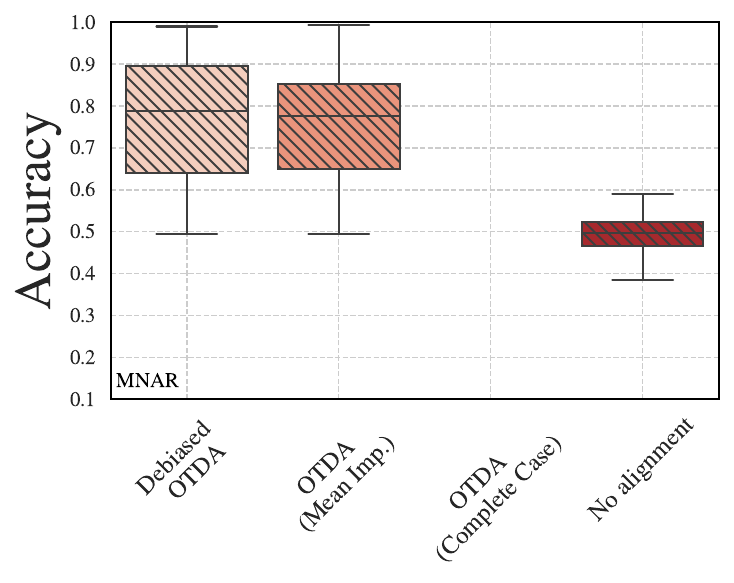}
        \caption{Domain adaptation using missing values. We compare three methods on this dataset. Our method, coined \texttt{Debiased OTDA}, aligns our data using our debiased estimator of the linear Monge map. The method \texttt{OTDA (Mean Imp.)} imputes the missing values using the mean along each dimension, before computing the linear Monge map. Finally, the method \texttt{OTDA (Complete Case)} drops all observations for which at least one feature is missing. We also report the baseline performance on the unaligned samples. \textbf{Left}: linear distribution shift. \textbf{Center Left}: linear distribution shift with MNAR data. \textbf{Center Right}: non-linear distribution shift. \textbf{Right}: non-linear distribution shift with MNAR data.}
        \label{fig:da}
\end{figure}

To asses wether our method is robust to MNAR data, we make the missingness mechanism data-dependent. For every sample in the small source dataset and every feature, the data is observed with probability $p_{ij} = \sigma(\mathbf{x}_{i,s}^{(j)})$ where $\sigma(u) = \frac{1}{1+\alpha e^{-u}}$ is the sigmoid function and we set $\alpha=2$. To debiase the linear Monge map, the missingness are learned feature-wise by taking the empirical probability of observing the data. The results are shown in Figure \ref{fig:da} in the hatched plots. We do not report the performance of the method $\texttt{OTDA (Complete Case)}$ because dropping observations makes the covariance non-invertible (which is required to compute Monge maps).

\subsection{Regularized Optimal Transport}

We now evaluate our two-step algorithm proposed in Section \ref{section:generalized_estimator}. To improve computational efficiency, we use $\texttt{softImpute}$, a slightly modified version of the ISVT algorithm~\citep{hastie2015matrix}. We generate two classes with $200$ samples each in dimension $5$. Every class has two clusters (see the documentation of the \texttt{scikit-learn} function for more details). The missingness probabilities are set uniformly to $p=0.5$ and $q=0.3$. The data is scaled using \texttt{minmax} scaling, since standardization is undesirable when comparing distributions. 

To run the experiment, we fix the generated data and randomize over missingness: at each run, a new missingness mask is drawn. For every sample size, we set the number of runs to $10$. The parameters for ISVT are chosen on a grid $\lambda_1,\lambda_0 \in$ \texttt{np.logspace(-2,2,20)}. For the Frobenius criterion, the size of the validation set is set to $20\%$, and for every run the results are averaged over $3$ random folds. To select the best hyperparameters with the regularized Bures-Wasserstein criterion, we compute our estimator using the same regularization strength as the targeted optimal transport problem. The results, displayed in Figure \ref{fig:entropic_ot}, validate the effectiveness of both our estimator and hyperparameter selection method.

\begin{figure}
    \centering
    \begin{subfigure}[b]{\textwidth}
         \centering
         \includegraphics[width=\textwidth]{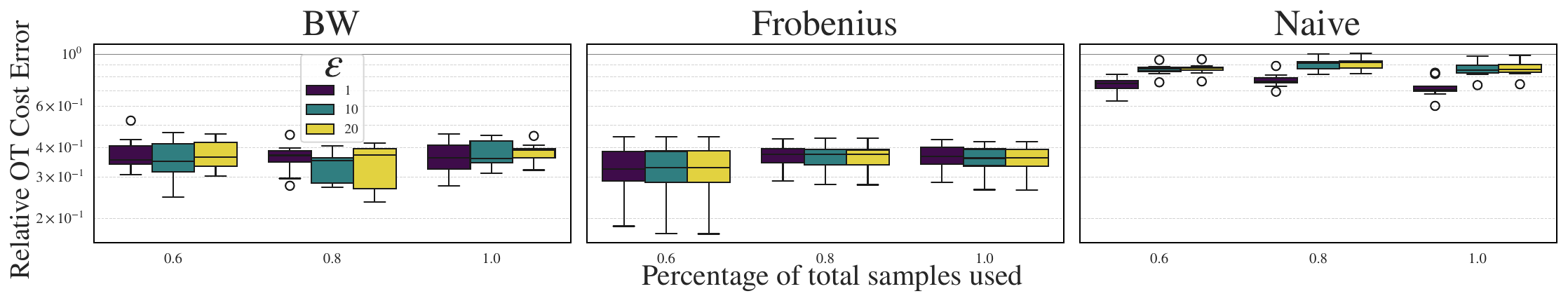}
         \label{fig:y equals x}
     \end{subfigure}
     \begin{subfigure}[b]{\textwidth}
         \centering
         \includegraphics[width=\textwidth]{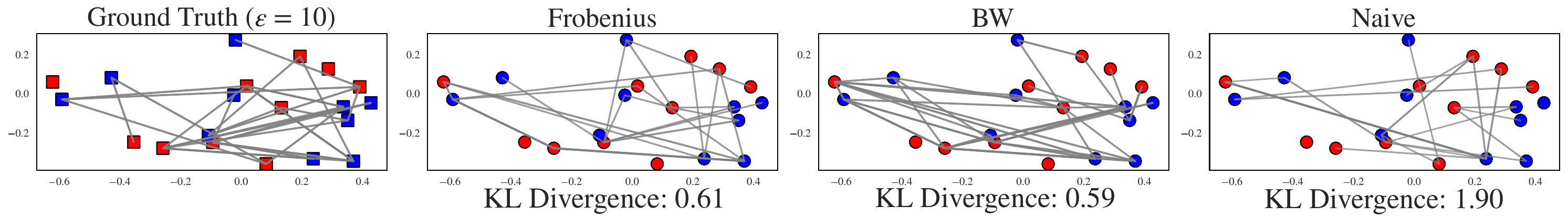}
         \label{fig:three sin x}
     \end{subfigure}
     \caption{Relative error in OT cost for different regularization strengths \textbf{(top)} and inferred OT maps for $\varepsilon=10$ \textbf{(bottom)}. We show only the top 25\% links, so some points appear unmatched.}
    \label{fig:entropic_ot}
\end{figure}

\subsection{Unregularized Optimal Transport (Random Features)}
\begin{wrapfigure}{l}{0.3\textwidth}
        \centering
        \includegraphics[width=0.42\linewidth]{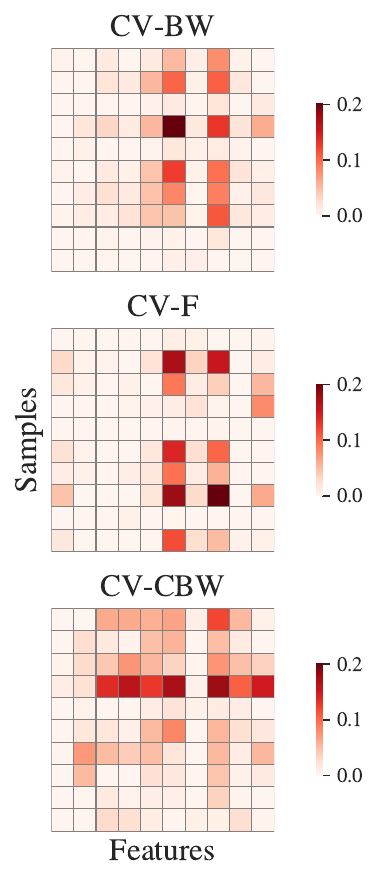}
        \caption{Entry-wise error of reconstructed matrices for the \texttt{random projection} experiment, using three different cross validation criterion.}
        \label{fig:example_random_feature}
        \vspace{-0.7cm}
\end{wrapfigure}
Moving to the unregularized problem, we generate a 2D dataset made of two separated classes using the $\texttt{make\_moons}$ function implemented in $\texttt{scikit-learn}$, before projecting it in to dimension $d=10$ via a random linear projection. The goal is then to estimate the optimal transport distance between both groups. This setting mimics the theoretical setting of matrix completion with low-rank assumption. The results are averaged over $10$ runs; for cross-validation with the Frobenius criterion, we randomly split the observed values into $20\%$ of validation entries and $80\%$ training entries.      

For cross-validation with the Bures-Wasserstein criterion, we use our criterion 
\begin{align*}
    \widehat{\mathbb{BW}}\Big[\texttt{softImpute}_\lambda\Big(\mathbf{X}^\texttt{NA}\Big), \mathbf{X}^\texttt{NA}\Big]
\end{align*}
and the so-called cross-BW criterion, defined as 
\begin{align*}
    \Big\lvert \widehat{\mathbb{BW}}\Big[\texttt{softImpute}_\lambda\big(\mathbf{X}^\texttt{NA}\big),\texttt{softImpute}_\lambda\Big(\mathbf{Y}^\texttt{NA}\Big)\Big] - \widehat{\mathbb{BW}}\Big[\mathbf{X}^\texttt{NA},\mathbf{Y}^\texttt{NA} \Big]\Big\rvert
\end{align*}

\begin{figure}
    \centering
    \includegraphics[width=0.8\linewidth]{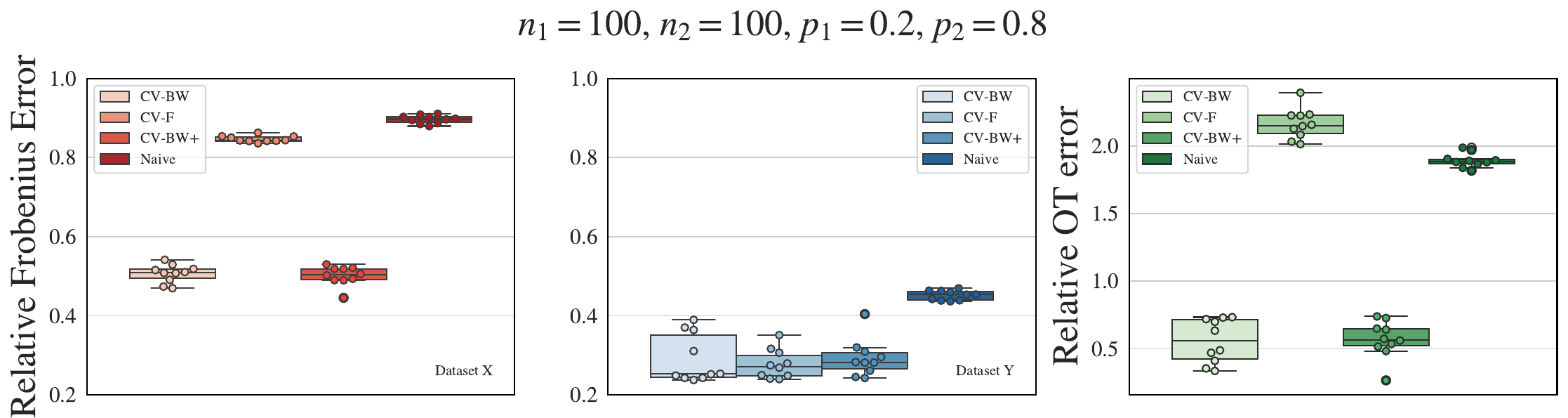}\\
    \includegraphics[width=0.8\linewidth]{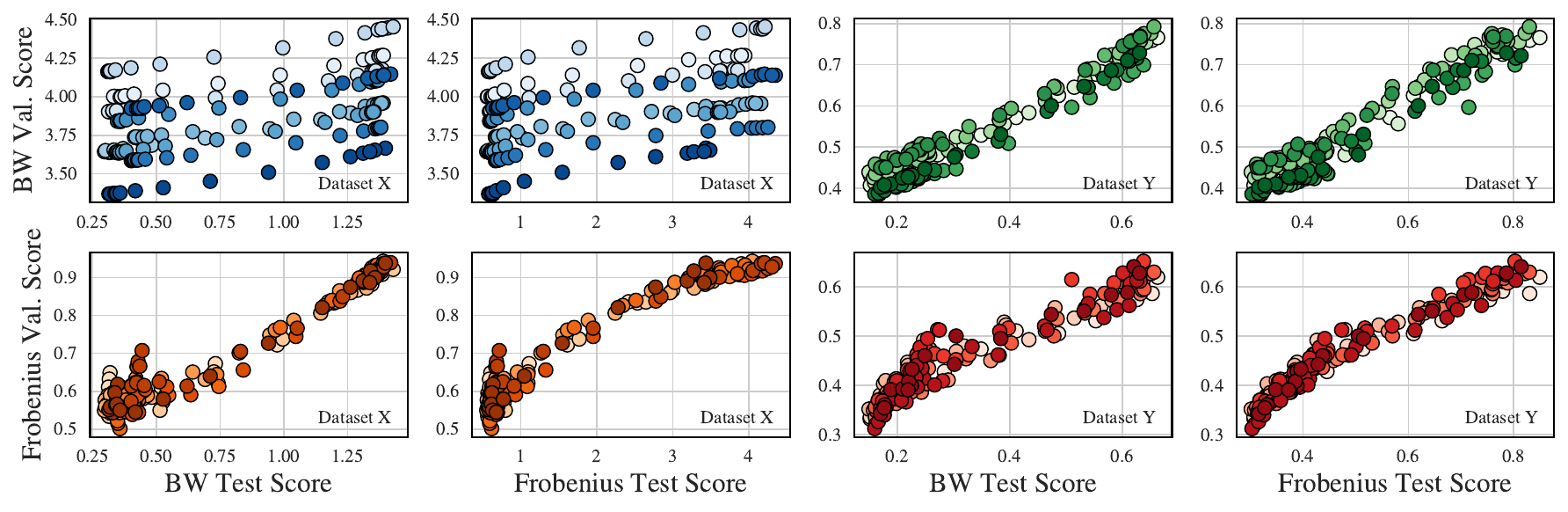}
    \caption{Results on the \texttt{random projection} experiment. The results are averaged over $10$ runs symbolized by dots on the graph. The validation set size is set to $0.2$. \textbf{Top Boxplots}: relative errors of all cross-validation methods. The data is scaled using the \texttt{MinMaxScaler} implemented in \texttt{scikit-learn}. \textbf{Bottom Boxplots:} the scatterplots on the right of the figure compare the different cross-validation scores computed and their value on the full unseen data. Each subplot compares a given cross-validation score computed on the data and compares it with its value on the the full matrix without any missing values. Each dot corresponds to on value of hyperparameters (all combinations of the grid are displayed). The plots on the left side correspond to the first dataset $\mathbf{X}$, and the ones on the right to the second dataset $\mathbf{Y}$. We average over folds and show all $10$ runs.}
    \label{fig:results_random_features}
\end{figure}

In our simulation, we set the missingness probability to be drastically higher in the first dataset. This translates directly in the performance of the considered cross-validation scores: while the Frobenius score performs well under low missingness, its performance drops dramatically under high missingness, leading ultimately to a biased estimation of the optimal transport distance. On the contrary, observe that our two cross-validation scores suffer less from data missingness on the matrix completion task, yielding eventually better estimation of the optimal transport distance. On the left of Figure \ref{fig:results_random_features}, we show the correlation between the test and validation errors for the two first criterion. Overall, both validation scores correlate with both test scores. However, the Bures-Wasserstein scores are remarkably less noisy and provide overall monotonic rankings of the different hyperparameters for a given run. Figure \ref{fig:example_random_feature} shows examples of the entrywise squared error for the three different scores: interestingly, the BW score and the Frobenius score produce comparable estimators (even though the BW score outperforms the Frobenius score), while the cross-BW score seems to produce a radically different solution, diffusing the error over multiple entries.

\subsection{Unregularized Optimal Transport on the Diabetes Dataset}

We extend our results to the \texttt{diabetes} dataset fetched from OpenML. This dataset includes $8$ biological features for $n=768$ patients, $n_+= 268$ of which have been tested positive and $n_-=500$ have been tested negative. Our goal is to compute the Wasserstein distance between both groups --- a task which might be of interest to determine wether the distributions of both groups is significantly different. We remove values at random with probability $p_+ = 0.7$ uniformly over features among positive patients, and $p_-=0.3$ among negative patients.

A first important observation is that the Bures-Wasserstein score captures similar dynamics as the Frobenius-based score. Observe however that the minima is attained earlier, suggesting the Bures-Wasserstein criterion trades precision in mean-squared error for distributional accuracy by taking the difference in covariances into account. The center-right and most-right panels show that our imputation method strongly outperforms naive imputation for optimal transport oriented matrix completion.

\begin{figure}
    \centering
    \includegraphics[width=0.8\linewidth]{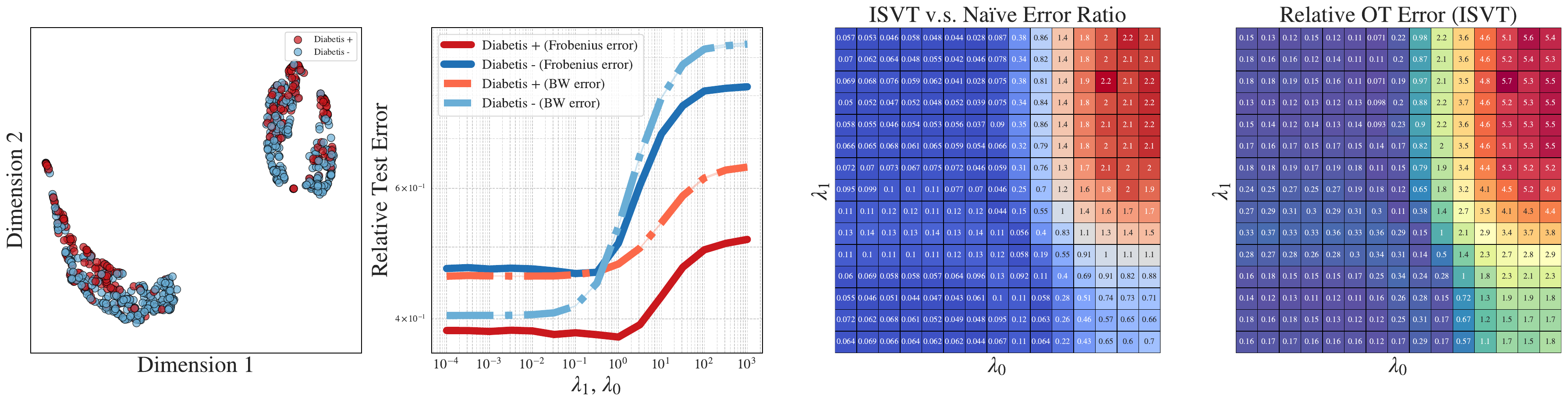}
    \caption{Results on the \texttt{diabetes} dataset. \textbf{Right:} UMAP of the data. \textbf{Center Left:} relative Frobenius error and Bures-Wasserstein score (dashed) as a function of the regularization parameters. \textbf{Center Right:} heatmap of the ratio of errors between ISVT and naive imputation (lower is better) as a function of the regularization parameters. \textbf{Right:} heatmap of relative error in OT cost estimation by ISVT (lower is better) as a function of the regularization parameters. The two last plots are obtained on the same hyperparameter grid than the center left plot. The results are averaged over $20$ runs.}
    \label{fig:diabetes_1}
\end{figure}

\section{Conclusion}
In this work, we introduced two complementary strategies for computing optimal transport distances in the presence of heterogeneous missingness: a debiasing approach for the Bures-Wasserstein distance and a matrix completion method based on ISVT for entropic-regularized optimal transport. We displayed the effectiveness of our methods on a wide set of numerical experiments. 

\paragraph{Limitations and Future Work.} We see a series of limitations and opportunities for future work.  
\begin{enumerate}
    \item Our analysis relies on the MCAR assumption, which may not always reflect real-world data, even though we have partially studied the effect of MNAR missingness. It would be interesting to design more complex estimators and to obtain theoretical guarantees in the MNAR case. 
    \item In the same spirit as recent papers by \citet{perez2022benchmarking} and \citet{morvan2024imputation}, we advocate for a large scale study of the effect of different imputation strategies on distributional tasks. Applications include generative modeling, statistical testing and quantile regression. 
    \item Beyond the entropic regularized optimal transport problem, many variants of the optimal transport problem---such as robust optimal transport, Gromov-Wasserstein and sliced optimal transport---are of high interest and could be adapted to accommodate missing values.  
\end{enumerate}

\paragraph{Acknowledgments.} We thank Clément Berenfeld, Claire Boyer, Jean-Baptiste Fermanian, Rémi Flamary, Agathe Guilloux, Yagan Hazard, Nicolas Keriven, Toru Kitagawa, Kimia Nadjahi and Jeffrey Näf for insightful discussions. LB conducted a major part of this work while visiting Inria Montpellier. 

\paragraph{Societal Impact.} This work mostly provides theoretical tools to handle missing values using low complexity, shallow algorithms. We believe that our contributions can help to reduce the use of computational resources by improving the sample complexity and precision of optimal transport estimators---compared to the complete case analysis which discards incomplete data. We purposely do not implement applications requiring extensive computational power, such as generative modeling and shape registration.  

\bibliographystyle{plainnat}
\bibliography{refs}

\appendix

\newpage
\newpage

\section{Supplementary Results and Mathematical Background}

\label{appendix:details}

\subsection{Basic Properties of Random Vectors}
\label{appendix:sub-Gaussian_rv}
\begin{defn}[sub-Gaussian Random Variable]
    A centered real-valued random variable $X$ is said to be sub-Gaussian with variance proxy $\sigma^2 $ if its moment generating function satisfies 
    \begin{align*}
        \mathbb{E}\Big[\exp(s\big(X-\mathbb{E}(X)\big))\Big] \leq \exp\Big(\frac{\sigma^2s^2}{2}\Big)
    \end{align*}
    for all $s \in \mathbb{R}$.
\end{defn}
\begin{defn}
    A random vector $X \in \mathbb{R}^d$ is said to be sub-Gaussian if for all $u \in \mathbb{S}^{d-1}$, the random variable $u^\top\big(X-\mathbb{E}(X)\big)$ is sub-Gaussian.
\end{defn}
The sub-Gaussian property is equivalent to the tail decay property 
\begin{align*}
    \mathbb{P}\Big(X-\mathbb{E}\big(X\big) > t\Big) \leq \exp\Big(\frac{-t^2}{2\sigma^2}\Big) \,\,  \textnormal{ and } \,\, \mathbb{P}\Big(X-\mathbb{E}\big(X\big) < -t\Big) \leq \exp\Big(\frac{-t^2}{2\sigma^2}\Big).
\end{align*}
We recall the following important result. 
\begin{lem}
\label{lemma:sub-Gaussian_mean}
    Let $X_1,\dots,X_n$ be independent sub-Gaussian random variables with variance proxies $\sigma_1,\dots,\sigma_n$. Then for all $a \in \mathbb{R}^n$, the random variable $\sum_{i=1}^n a_i\big(X_i-\mathbb{E}(X_i)\big)$ is sub-Gaussian with variance proxy $\sum_{i=1}^n \cov_i^2 a_i^2$.
\end{lem}
\begin{rmk}
    Note that this implies that $n^{-1}\sum_{i=1}^n \big(X_i-\mathbb{E}(X_i)\big)$ is $n^{-2}\sum_{i=1}^n \sigma_i^2$ sub-Gaussian.
\end{rmk}
A proof is given by \cite{rigollet2023high}, Corollary 1.7. We now define sub-exponential random variables. 
\begin{defn}
    A random variable $X$ is said to be sub-exponential with parameter $\lambda$ if its moment generating function satisfies 
    \begin{align*}
        \mathbb{E}\Big[\exp(s\big(X-\mathbb{E}(X)\big))\Big] \leq \exp\Big(\frac{s^2 \lambda^2}{2}\Big)
    \end{align*}
    for all $\abs{s} \leq \lambda^{-1}$.
\end{defn}
We end this section by recalling Bernstein's inequality for sub-exponential random variables. 
\begin{thm}[\citet{vershynin2010introduction}, Corollary 5.17]
    Let $X_1,\dots,X_n$ be independent sub-exponential random variables and let $K := \max_{i=1,\dots,n} \norm{X_i}_{\psi_1}.$ Then for all $t >0$, with probability at least $1-e^{-t}$
    \begin{align*}
        \Big\lvert n^{-1}\sum\limits_{i=1}^n \Big( X_i-\mathbb{E}(X_i)\Big)\Big\rvert \leq CK \Big(t^{-\frac{1}{2}}n^{-\frac{1}{2}} \vee t^{-1}n^{-1}\Big)
    \end{align*}
    for a universal constant $C >0$.
\end{thm}

\subsection{Perturbation Bound for Matrix Square Roots}

Let $A \in \mathbb{R}^{n \times n}$. The Hermitian part of $A$ is defined as 
\begin{align*}
    \textnormal{Re}\big(A\big) := \frac{1}{2}\Big(A+A^\top\Big).
\end{align*}
We recall the following lemma from \citet{schmitt1992perturbation}.

\begin{lem}
\label{lem:sqrt_matrix_bound}
    Let $A_1,A_2 \in \mathbb{R}^{n \times n}$ such that $ \textnormal{Re}\big(A_j\big) \geq \mu_j I$ with $\mu_j > 0$ for $j=1,2$. Then
    \begin{align*}
        \norm{A_1^{\frac{1}{2}}-A_2^{\frac{1}{2}}} \leq \frac{1}{\mu_1+\mu_2}\norm{A_1-A_2}.
        \end{align*}
\end{lem}

\subsection{Sensitivity bounds for the Wasserstein Distance}

We recall the two following bounds from \citep{keriven2022entropic}, which allow to bound the difference between the Wasserstein distance and the transport plan of two transport problems with different costs. These results typically require the cost function to be bounded. This assumption is satisfied if the cost function is continuous and the random variables have bounded support.

\begin{thm}[\citet{keriven2022entropic}, Lemma 1]
    We have 
    \begin{align*}
          \big\lvert \wasserstein_\mathbf{C_1}(\mu_n,\eta_m) - \wasserstein_\mathbf{C_2}(\mu_n,\eta_m)\big\rvert\leq \big\lVert \mathbf{C}_1-\mathbf{C}_2\big\rVert_\infty.
    \end{align*}
\end{thm}
\begin{thm}
\label{thm:keriven_thm_1}
    Assume that the cost functions $\mathbf{C}_1,\mathbf{C}_2:\mathbb{R}^d \times \mathbb{R}^d \to \mathbb{R}_+$ are continuous. Under Assumption \ref{asm:bounded_support}, we have 
    \begin{align*}
       \big\lvert \wasserstein_\mathbf{C_1, \varepsilon}(\mu_n,\eta_m) - \wasserstein_\mathbf{C_2,\varepsilon}(\mu_n,\eta_m)\big\rvert\leq \frac{K_\varepsilon}{\sqrt{nm}} \norm{\mathbf{C}_1-\mathbf{C}_2}_\textnormal{F}.
    \end{align*}
    where $K_\varepsilon:= \exp(\varepsilon^{-1}(2c_\textnormal{max}-c_\textnormal{min}))$ and furthermore
    \begin{align*}
        & \kl\big[\mathbf{\Pi}\big(\mathbf{C}_1\big)\,|\,\mathbf{\Pi}\big(\mathbf{C}_2\big) \big]
        \leq \frac{K_\varepsilon}{\varepsilon \sqrt{nm}}\norm{\mathbf{C}_1-\mathbf{C}_2}_\textnormal{F}  + K'_\varepsilon \Big(\frac{1}{\varepsilon^2 \sqrt{nm}} \norm{\mathbf{C}_1-\mathbf{C}_2}_\textnormal{F}\Big)^{\frac{1}{2}}
    \end{align*}
    where $K_\varepsilon' := \exp\Big(\frac{1}{2 \varepsilon}(3c_\textnormal{max}-7c_\textnormal{min})\Big)$.
\end{thm}

\subsection{Linking Covariances}

\begin{lem}
\label{lem:general_mean}
Under Assumption \ref{asm:heterogenous_mcar_and_non_zero}, we have 
    \begin{align*}
        \cov_x^\textnormal{\textnormal{\texttt{NA}}} =\mathbf{P}  \cov_x\mathbf{P} + \mathbf{P}\big(\mathbf{I}_d - \mathbf{P}\big)\textnormal{diag}\big(\cov_x\big) + \mathbf{P}\big(\mathbf{I}_d - \mathbf{P}\big)\textnormal{diag}\big(\mathbf{m}_x\big)^2
    \end{align*}
    and 
    \begin{align*}
         \cov_x = \mathbf{P}^{-1}  \cov_x^\textnormal{\texttt{NA}}\mathbf{P}^{-1} + \mathbf{P}^{-1}\big(\mathbf{I}_d - \mathbf{P}^{-1}\big)\textnormal{diag}\big(\cov_x^\textnormal{\texttt{NA}}\big) + \mathbf{P}^{-1}\big(\mathbf{I}_d - \mathbf{P}^{-1}\big)\textnormal{diag}\big(\mathbf{m}^{\textnormal{\texttt{NA}}}_x\big)^2,
    \end{align*}
    where $\textnormal{diag}\big(\mathbf{m}_x\big)$ resp. $\textnormal{diag}\big(\mathbf{m}^{\textnormal{\texttt{NA}}}_x\big)$ are $d\times d$ diagonal matrices whose entries are the values of the entries of $\mathbf{m}_x$ resp. $\mathbf{m}_x^{\textnormal{\texttt{NA}}}$.
\end{lem}

This lemma establishes a connection between the covariances of $X^\na$ and $X$, and suggests that $\cov_x$ can be estimated using an estimator of $\cov_x^\textnormal{\texttt{NA}}$ and $\mathbf{m}^{\textnormal{\texttt{NA}}}_x$ which is then debiased through multiplication by the missingness matrix $\mathbf{P}$. We hence define the estimators
\begin{align*}
    & \widehat{\cov}_x := \mathbf{P}^{-1}\big(\mathbf{I}_d-\mathbf{P}^{-1}\big) \textnormal{diag}\big(\mathbf{S}^{\textnormal{\texttt{NA}}}_{xn}\big) + \mathbf{P}^{-1}\mathbf{S}^{\textnormal{\texttt{NA}}}_{xn}\mathbf{P}^{-1}+ \mathbf{P}^{-1}\big(\mathbf{I}_d-\mathbf{P}^{-1}\big) \textnormal{diag}\big(\mathbf{m}_{xn}^{\textnormal{\texttt{NA}}} \mathbf{m}_{xn},^{\textnormal{\texttt{NA}}\top}\big) \\
    & \widehat{\cov}_y := \mathbf{Q}^{-1}\big(\mathbf{I}_d-\mathbf{Q}^{-1}\big) \textnormal{diag}\big(\mathbf{S}^{\textnormal{\texttt{NA}}}_{ym}\big) + \mathbf{Q}^{-1}\mathbf{S}^{\textnormal{\texttt{NA}}}_{ym}\mathbf{Q}^{-1}+ \mathbf{Q}^{-1}\big(\mathbf{I}_d-\mathbf{Q}^{-1}\big) \textnormal{diag}\big(\mathbf{m}_{ym}^{\textnormal{\texttt{NA}}} \mathbf{m}_{ym}^{\textnormal{\texttt{NA}}\top}\big),
\end{align*}
which extend the estimator of \citet{zhu2022high} to the heterogeneous missingness setting. 

\newpage

\section{Proofs}

\label{appendix:proofs}

\subsection{Proof of Proposition \ref{prop:implicit_reg_ot}}

\begin{proof}
    We have 
    \begin{align*}
        & \mathbb{E}\big[\norm{X^\texttt{NA}-Y^\texttt{NA}}^2_\mathbf{M} \,|\, X,Y\big]\\
        = &\mathbb{E}\big[X^{\texttt{NA} \top }\mathbf{M}X^{\texttt{NA}} + Y^{\texttt{NA} \top }\mathbf{M}Y^{\texttt{NA}} - 2 X^{\texttt{NA} \top }\mathbf{M}Y^{\texttt{NA}}\,|\, X,Y\big]\\
        = & X^\top\mathbf{P}\textnormal{Off}\big(\mathbf{M}\big)\mathbf{P}X + Y^\top\mathbf{Q}\textnormal{Off}\big(\mathbf{M}\big)\mathbf{Q}Y 
        + X^\top\mathbf{P}\textnormal{diag}\big(\mathbf{M}\big)X + X^\top\mathbf{Q}\textnormal{diag}\big(\mathbf{M}\big)X  
         - 2 X^\top \mathbf{PMQ}Y\\
         = & X^\top\mathbf{P}\mathbf{M}\mathbf{P}X + Y^\top\mathbf{Q}\mathbf{M}\mathbf{Q}Y 
         +  X^\top\mathbf{P}\textnormal{diag}\big(\mathbf{M}\big)\big(\mathbf{I}_d-\mathbf{P}\big)X +Y^\top\mathbf{Q}\textnormal{diag}\big(\mathbf{M}\big)\big(\mathbf{I}_d-\mathbf{Q}\big)Y - 2 X^\top \mathbf{PMQ}Y.
         \end{align*}
         Now, by writing $\mathbf{P} = \mathbf{P} - \mathbf{I}_d+\mathbf{I}_d$ and $\mathbf{Q} = \mathbf{Q} - \mathbf{I}_d+\mathbf{I}_d$, we get that 
        \begin{align*}
            X^\top \mathbf{PMQ}Y = X^\top \mathbf{M}Y + X^\top (\mathbf{P}-\mathbf{I}_d)\mathbf{M}Y + X^\top (\mathbf{P}-\mathbf{I}_d)\mathbf{M}(\mathbf{Q}-\mathbf{I}_d)Y + X^\top \mathbf{M}(\mathbf{Q}-\mathbf{I}_d)Y. 
        \end{align*}
         Similarly, 
        \begin{align*}
            X^\top\mathbf{P}\mathbf{M}\mathbf{P}X = & X^\top\mathbf{M}X + X^\top(\mathbf{P}-\mathbf{I}_d)\mathbf{M}(\mathbf{P}-\mathbf{I}_d)X + 2X^\top(\mathbf{P}-\mathbf{I}_d)\mathbf{M}X\\
            =&X^\top\mathbf{M}X + X^\top(\mathbf{P}-\mathbf{I}_d)\mathbf{M}(\mathbf{P}+\mathbf{I}_d)X
        \end{align*}
        and 
        \begin{align*}
            Y^\top\mathbf{Q}\mathbf{M}\mathbf{Q}Y = &Y^\top\mathbf{M}Y + Y^\top(\mathbf{Q}-\mathbf{I}_d)\mathbf{M}(\mathbf{Q}-\mathbf{I}_d)Y + 2Y^\top(\mathbf{Q}-\mathbf{I}_d)\mathbf{M}Y\\
            =&Y^\top\mathbf{M}Y + Y^\top(\mathbf{Q}-\mathbf{I}_d)\mathbf{M}(\mathbf{Q}+\mathbf{I}_d)Y.
        \end{align*}
        Combining everything, we get that 
        \begin{align*}
            & \mathbb{E}\big[\norm{X^\texttt{NA}-Y^\texttt{NA}}^2_\mathbf{M} \,|\, X,Y\big] \\
            = &\norm{X-Y}^2_\mathbf{M} + X^\top(\mathbf{P}-\mathbf{I}_d)\mathbf{M}(\mathbf{P}+\mathbf{I}_d)X + Y^\top(\mathbf{Q}-\mathbf{I}_d)\mathbf{M}(\mathbf{Q}+\mathbf{I}_d)Y\\
            + &   X^\top\mathbf{P}\textnormal{diag}\big(\mathbf{M}\big)\big(\mathbf{I}_d-\mathbf{P}\big)X +Y^\top\mathbf{Q}\textnormal{diag}\big(\mathbf{M}\big)\big(\mathbf{I}_d-\mathbf{Q}\big)Y\\
            -& X^\top (\mathbf{P}-\mathbf{I}_d)\mathbf{M}Y - X^\top (\mathbf{P}-\mathbf{I}_d)\mathbf{M}(\mathbf{Q}-\mathbf{I}_d)Y - X^\top \mathbf{M}(\mathbf{Q}-\mathbf{I}_d)Y\\
            =& \norm{X-Y}^2_\mathbf{M} +X^\top(\mathbf{P}-\mathbf{I}_d)\big(\mathbf{M}(\mathbf{P}+\mathbf{I}_d)-\mathbf{P}\textnormal{diag}\big(\mathbf{M}\big)\big)X 
            + Y^\top(\mathbf{Q}-\mathbf{I}_d)\big(\mathbf{M}(\mathbf{Q}+\mathbf{I}_d)-\mathbf{Q}\textnormal{diag}\big(\mathbf{M}\big)\big)Y\\
            - & X^\top \mathbf{P}\big(\mathbf{M} - \mathbf{P}^{-1}\mathbf{M}\mathbf{Q}^{-1}\big) \mathbf{Q}Y\\
            =& \norm{X-Y}^2_\mathbf{M} \\
            + & X^\top(\mathbf{P}-\mathbf{I}_d)\big(\mathbf{M} + \textnormal{Off}(\mathbf{M})\mathbf{P}\big)X + Y^\top(\mathbf{Q}-\mathbf{I}_d)\big(\mathbf{M} + \textnormal{Off}(\mathbf{M})\mathbf{Q}\big)Y
            -X^\top \mathbf{P}\big(\mathbf{M} - \mathbf{P}^{-1}\mathbf{M}\mathbf{Q}^{-1}\big) \mathbf{Q}Y.  
        \end{align*}
        We now combine this expression with the optimal transport problem. We get 
        \begin{align*}
            &\mathbb{E}\Big[\sum\limits_{i=1}^n \sum\limits_{j=1} \mathbf{\Pi}_{ij} \mathbf{C}^\texttt{NA}_{ij}\,|\, \mathbf{X},\mathbf{Y}\Big]\\
            =& \sum\limits_{i=1}^n \sum\limits_{j=1}^m \mathbf{\Pi}_{ij} \Big[\mathbf{C}_{ij} - \mathbf{x}_i^\top \mathbf{P}\big(\mathbf{M} - \mathbf{P}^{-1}\mathbf{M}\mathbf{Q}^{-1}\big) \mathbf{Q} \mathbf{y}_j \Big]\\
            + & \frac{1}{n} \sum\limits_{i=1}^n \mathbf{x}_i^\top(\mathbf{P}-\mathbf{I}_d)\big(\mathbf{M} + \textnormal{Off}(\mathbf{M})\mathbf{P}\big)\mathbf{x}_i 
            +  \frac{1}{m} \sum\limits_{j=1}^m \mathbf{y}_j^\top(\mathbf{Q}-\mathbf{I}_d)\big(\mathbf{M} + \textnormal{Off}(\mathbf{M})\mathbf{Q}\big)\mathbf{y}_j
        \end{align*}
        where we have used the fact that $\sum\limits_{i=1}^n \mathbf{\Pi}_{ij} = \frac{1}{m}$ and $\sum\limits_{j=1}^m \mathbf{\Pi}_{ij} = \frac{1}{n}$ for all $i,j$. Rewriting the last expression, we get that
        \begin{align}
        \label{eq:eq_7}
            \mathbb{E}\Big[\sum\limits_{i=1}^n \sum\limits_{j=1}^m \mathbf{\Pi}_{ij} \mathbf{C}^\texttt{NA}_{ij}\,|\, \mathbf{X},\mathbf{Y}\Big] = &\sum\limits_{i=1}^n \sum\limits_{j=1}^m \mathbf{\Pi}_{ij}\Big[ \mathbf{C}_{ij} + \mathbf{x}_i^\top  \mathbf{P}\big(\mathbf{M} - \mathbf{P}^{-1}\mathbf{M}\mathbf{Q}^{-1}\big)\mathbf{y}_i\Big] \\
            + &\textnormal{Tr}\Big[\mathbf{X}^\top \Omega_\mathbf{M}(\bigp)\mathbf{X} + \mathbf{Y}^\top \Omega_\mathbf{M}(\bigp)\mathbf{Y}\Big].
        \end{align}
        Letting 
        \begin{align*}
            \overline{\mathbf{\Pi}}_\mathbf{M} \in \argmin\limits_{\mathbf{\Pi} \in \Pi} \mathbb{E}\Big[\sum\limits_{i=1}^n \sum\limits_{j=1}^m \mathbf{\Pi}_{ij} \mathbf{C}^\texttt{NA}_{ij}\,|\, \mathbf{X},\mathbf{Y}\Big],
        \end{align*}
        and noticing that we can eliminate terms independent of the optimization variables $\mathbf{\Pi}$ from Equation \eqref{eq:eq_7}, we have from the former that 
        \begin{align*}
            \overline{\mathbf{\Pi}}_\mathbf{M} \in & \argmin\limits_{\mathbf{\Pi} \in \Pi}\sum\limits_{i=1}^n \sum\limits_{j=1}^m \mathbf{\Pi}_{ij} \Big[ \mathbf{C}_{ij} + \mathbf{x}_i^\top\big(\mathbf{P}\mathbf{M}\mathbf{Q} - \mathbf{M}\big)\mathbf{y}_i\Big] \\
            = & \argmin\limits_{\mathbf{\Pi} \in \Pi}\sum\limits_{i=1}^n \sum\limits_{j=1}^m \mathbf{\Pi}_{ij} \mathbf{C}_{ij} + \textnormal{Tr}\Big[\mathbf{\Pi}^\top \mathbf{X}^\top \big(\mathbf{P}\mathbf{M}\mathbf{Q} - \mathbf{M}\big)\mathbf{Y}\Big]\\
            = &  \argmin\limits_{\mathbf{\Pi} \in \Pi} \textnormal{Tr}\Big[\mathbf{\Pi}^\top \mathbf{C}\Big]  + \textnormal{Tr}\Big[\mathbf{\Pi}^\top \mathbf{X}^\top \big(\mathbf{P}\mathbf{M}\mathbf{Q} - \mathbf{M}\big)\mathbf{Y}\Big] \\
            = & \argmin\limits_{\mathbf{\Pi} \in \Pi} \textnormal{Tr}\Big[\mathbf{\Pi}^\top \mathbf{C}\Big]  - \textnormal{Tr}\Big[\mathbf{\Pi}^\top \mathbf{X}^\top \overline{\mathbf{M}}\mathbf{Y}\Big]\\
            =&\argmin\limits_{\mathbf{\Pi} \in \Pi} \textnormal{Tr}\Big[\mathbf{\Pi}^\top \big(\mathbf{C} - \mathbf{X}^\top \overline{\mathbf{M}}\mathbf{Y} \big)\Big] 
        \end{align*}
        with $\overline{\mathbf{M}} = \mathbf{M}-\mathbf{P}\mathbf{M}\mathbf{Q} = \Big[(1-p_iq_j)m_{ij}\Big]_{i,j=1}^d$. 
\end{proof}

\subsection{Proof of Proposition \ref{prop:bias_ot_na}}

\begin{proof}

We start by observing that
\begin{align*}
    \int_{\mathbb{R}^d \times \mathbb{R}^d} \norm{x-y}^2_\mathbf{M} d\gamma(x,y) = \int_{\mathbb{R}^d \times \mathbb{R}^d} \norm{\mathbf{M}^{\frac{1}{2}}x-\mathbf{M}^{\frac{1}{2}}y}^2_2 d\gamma(x,y)
\end{align*}
and that if $X$ has mean $\mathbf{m}_x$ and covariance $\cov_x$, then $\mathbf{M}^{\frac{1}{2}}X$ has mean $\mathbf{M}^{\frac{1}{2}}\mathbf{m}_x$ and covariance $\mathbf{M}^{\frac{1}{2}} \cov_x \mathbf{M}^{\frac{1}{2}}$. Recall that
\begin{align*}
    \cov_x^\textnormal{\textnormal{\texttt{NA}}} =\mathbf{P}  \cov_x\mathbf{P} + \mathbf{P}\big(\mathbf{I}_d - \mathbf{P}\big)\textnormal{diag}\big(\cov_x\big) + \mathbf{P}\big(\mathbf{I}_d - \mathbf{P}\big)\textnormal{diag}\big(\mathbf{m}_x\big)^2.
\end{align*}
We now assume that $\cov_x$ and $\mathbf{M}$ are diagonal. This implies that $\cov_x^\texttt{NA}$ is diagonal. From \citet{gelbrich1990formula},
we have 
\begin{align*}
    \wasserstein_\mathbf{M}^2(\mu,\mu^\texttt{NA}) \geq & \norm{\mathbf{M}^{\frac{1}{2}}(\mathbf{m}_\mu - \mathbf{m}_\mu^\texttt{NA})}_2^2 + \norm{ \big(\mathbf{M}^\frac{1}{2}\cov_x\mathbf{M}^\frac{1}{2}\big)^\frac{1}{2} - \big(\mathbf{M}^\frac{1}{2}\cov_x^\texttt{NA}\mathbf{M}^\frac{1}{2}\big)^\frac{1}{2}}_2^2.
\end{align*}
Since 
\begin{align*}
    \Big[\big(\mathbf{M}^\frac{1}{2}\cov_x\mathbf{M}^\frac{1}{2}\big)^\frac{1}{2} - \big(\mathbf{M}^\frac{1}{2}\cov_x^\texttt{NA}\mathbf{M}^\frac{1}{2}\big)^\frac{1}{2}\Big]_{ii} = &\sqrt{m_{ii}}\sigma_i - \sqrt{m_{ii}p_i\sigma^2_i + m_{ii}p_i(1-p_i)(\mathbf{m}_x)^2_{i}}\\
    =&\sqrt{m_{ii}}\Big(\sigma_i - \sqrt{p_i}\sqrt{\sigma^2_i + (1-p_i)(\mathbf{m}_x)^2_{i}}\Big),
\end{align*}
we obtain 
\begin{align*}
    \wasserstein_\mathbf{M}^2(\mu,\mu^\texttt{NA}) \geq & \norm{\mathbf{M}^{\frac{1}{2}}(\mathbf{m}_\mu - \mathbf{m}_\mu^\texttt{NA})}_2^2 + \sum\limits_{i=1}^d m_{ii}\Big[\sigma_i - \sqrt{p_i}\sqrt{\sigma^2_i + (1-p_i)(\mathbf{m}_x)^2_{i}}\Big]^2
\end{align*}
the proof is concluded by remarking that $\mathbf{m}_x^\texttt{NA} = \mathbf{P}\mathbf{m}_x$.
\end{proof}

\subsection{Proof of Lemma \ref{lem:extension_pacreau}}

\begin{proof}
    Recall that our estimator is 
    \begin{align*}
        \widehat{\cov}_{xn}:= \mathbf{P}^{-1}\big(\mathbf{I}_d-\mathbf{P}^{-1}\big)\textnormal{Diag}\big(\mathbf{S}_{xn}^\texttt{NA} \big) + \mathbf{P}^{-1}\cov_n^\texttt{NA}\mathbf{P}^{-1}  + \mathbf{P}^{-1}\big(\mathbf{I}_d-\mathbf{P}^{-1}\big) \textnormal{diag}\big(\mathbf{m}_{xn}^{\textnormal{\texttt{NA}}} \mathbf{m}_{xn}^{\textnormal{\texttt{NA}}\top} \big).        
    \end{align*}
    We introduce the centered observations $\mathbf{w}_i^\texttt{NA}:= \mathbf{x}_i^\texttt{NA} - \mathbf{m}_x^\texttt{NA}$. We have 
    \begin{align}
        \mathbf{S}_{xn}^\texttt{NA} = \frac{1}{n}\sum\limits_{i=1}^n \mathbf{w}_i^\texttt{NA}\mathbf{w}_i^{\texttt{NA}\top} - 2\big(\mathbf{m}_x^\texttt{NA} -\mathbf{m}_{xn}^\texttt{NA} \big)\big(\mathbf{m}_x^\texttt{NA} -\mathbf{m}_{xn}^\texttt{NA} \big)^\top.
    \end{align}
    We define the matrices
    \begin{align*}
        & \overline{\mathbf{S}}_{xn}^\texttt{NA} := \frac{1}{n}\sum\limits_{i=1}^n \mathbf{w}_i^\texttt{NA}\mathbf{w}_i^{\texttt{NA}\top}\\
        & \overline{\mathbf{M}}^\texttt{NA}_{xn} := \big(\mathbf{m}_x^\texttt{NA} -\mathbf{m}_{xn}^\texttt{NA} \big)\big(\mathbf{m}_x^\texttt{NA} -\mathbf{m}_{xn}^\texttt{NA} \big)^\top.
    \end{align*}
    Combining this decomposition with the expression of our estimator gives
    \begin{align*}
        \widehat{\cov}_{xn} = & \mathbf{P}^{-1}\big(\mathbf{I}_d-\mathbf{P}^{-1}\big)\textnormal{Diag}\big(\overline{\mathbf{S}}_{xn}^\texttt{NA} \big) + \mathbf{P}^{-1}\overline{\mathbf{S}}_{xn}^\texttt{NA}\mathbf{P}^{-1} \\
         + & \mathbf{P}^{-1}\big(\mathbf{I}_d-\mathbf{P}^{-1}\big) \textnormal{diag}\big(\mathbf{m}_{xn}^{\textnormal{\texttt{NA}}} \mathbf{m}_{xn}^{\textnormal{\texttt{NA}}\top} \big) - 2\mathbf{P}^{-1}\big(\mathbf{I}_d-\mathbf{P}^{-1}\big)\textnormal{Diag}\big(\overline{\mathbf{M}}^\texttt{NA}_n \big) - 2\mathbf{P}^{-1}\overline{\mathbf{M}}^\texttt{NA}_n\mathbf{P}^{-1}.
    \end{align*}
    Since
    \begin{align*}
        \cov_x = \mathbf{P}^{-1}  \cov_x^\textnormal{\texttt{NA}}\mathbf{P}^{-1} + \mathbf{P}^{-1}\big(\mathbf{I}_d - \mathbf{P}^{-1}\big)\textnormal{diag}\big(\cov_x^\textnormal{\texttt{NA}}\big) + \mathbf{P}^{-1}\big(\mathbf{I}_d - \mathbf{P}^{-1}\big)\textnormal{diag}\big(\mathbf{m}^{\textnormal{\texttt{NA}}}_x\big)^2,
    \end{align*}
    we get
    \begin{align*}
        \Big\lVert \widehat{\cov}_{xn} - \cov_x \Big\rVert_\textnormal{op} \leq & \Big\lVert \mathbf{P}^{-1}\big(\mathbf{I}_d-\mathbf{P}^{-1}\big)\textnormal{Diag}\big(\overline{\mathbf{S}}_{xn}^\texttt{NA} -\cov_x^\textnormal{\texttt{NA}}\big) + \mathbf{P}^{-1}\Big(\overline{\mathbf{S}}_{xn}^\texttt{NA}-\cov_x^\textnormal{\texttt{NA}}\Big)\mathbf{P}^{-1} \Big\rVert_\textnormal{op}\\
         + & \Big\lVert \mathbf{P}^{-1}\big(\mathbf{I}_d-\mathbf{P}^{-1}\big) \textnormal{diag}\big(\mathbf{m}_{xn}^{\textnormal{\texttt{NA}}} \mathbf{m}_{xn}^{\textnormal{\texttt{NA}}\top}  - \mathbf{m}^{\textnormal{\texttt{NA}}}_x\mathbf{m}^{\textnormal{\texttt{NA}}\top}_x \big) \Big\rVert \\
          + & \Big\lVert 2\mathbf{P}^{-1}\big(\mathbf{I}_d-\mathbf{P}^{-1}\big)\textnormal{Diag}\big(\overline{\mathbf{M}}^\texttt{NA}_{xn} \big) + 2\mathbf{P}^{-1}\overline{\mathbf{M}}^\texttt{NA}_{xn}\mathbf{P}^{-1}\Big\rVert .
    \end{align*}
    We define 
    \begin{align*}
        & \mathbf{A}_n := \Big\lVert \mathbf{P}^{-1}\big(\mathbf{I}_d-\mathbf{P}^{-1}\big)\textnormal{Diag}\big(\overline{\mathbf{S}}_{xn}^\texttt{NA} -\cov_x^\textnormal{\texttt{NA}}\big) + \mathbf{P}^{-1}\Big(\overline{\mathbf{S}}_{xn}^\texttt{NA}-\cov_x^\textnormal{\texttt{NA}}\Big)\mathbf{P}^{-1} \Big\rVert_\textnormal{op}\\
        & \mathbf{B}_n := \Big\lVert \mathbf{P}^{-1}\big(\mathbf{I}_d-\mathbf{P}^{-1}\big) \textnormal{diag}\big(\mathbf{m}_{xn}^{\textnormal{\texttt{NA}}} \mathbf{m}_{xn}^{\textnormal{\texttt{NA}}\top}  - \mathbf{m}^{\textnormal{\texttt{NA}}}_x\mathbf{m}^{\textnormal{\texttt{NA}}\top}_x \big) \Big\rVert\\
        & \mathbf{B}_n := \Big\lVert 2\mathbf{P}^{-1}\big(\mathbf{I}_d-\mathbf{P}^{-1}\big)\textnormal{Diag}\big(\overline{\mathbf{M}}^\texttt{NA}_{xn} \big) + 2\mathbf{P}^{-1}\overline{\mathbf{M}}^\texttt{NA}_{xn}\mathbf{P}^{-1}\Big\rVert .\\
    \end{align*}
    \paragraph{Concentration of $\mathbf{A}_n$.} The first term can be bounded using Theorem 2 in \citep{pacreau2024robust} since it corresponds to a setting with centered random variables. This yields 
    \begin{align}
     \mathbf{a}_n \lesssim \norm{\cov_x}_\textnormal{op} \norm{\mathbf{P}}_\textnormal{op}^{-1}\Bigg(\sqrt{\frac{\mathbf{r}(\cov_x) \log \mathbf{r}(\cov_x)}{n}} \vee \sqrt{\frac{t}{n}} \vee \frac{\mathbf{r}(\cov_x)\big(t+\log \mathbf{r}(\cov_x)\big)}{\norm{\mathbf{P}}_\textnormal{op}n} \log n\Bigg)
    \end{align}
    with probability at least $1-e^{-t}$.
    \paragraph{Concentration of $\mathbf{B}_n$.} We bound this term by using Bernstein's inequality for sub-exponential random variables. Let us first show that the random matrix $\textnormal{diag}\big(\mathbf{m}_{xn}^{\textnormal{\texttt{NA}}} \mathbf{m}_{xn}^{\textnormal{\texttt{NA}}\top}  - \mathbf{m}^{\textnormal{\texttt{NA}}}_x\mathbf{m}^{\textnormal{\texttt{NA}}\top}_x \big)$ is sub-exponential. We first have 
    \begin{align*}
        \mathbf{b}_n & \leq \Big\lVert \mathbf{P}^{-1}\big(\mathbf{I}_d-\mathbf{P}^{-1}\big) \Big\rVert_\textnormal{op}  \Big\lVert\textnormal{diag}\big(\mathbf{m}_{xn}^{\textnormal{\texttt{NA}}} \mathbf{m}_{xn}^{\textnormal{\texttt{NA}}\top}  - \mathbf{m}^{\textnormal{\texttt{NA}}}_x\mathbf{m}^{\textnormal{\texttt{NA}}\top}_x \big) \Big\rVert_\textnormal{op}\\ & = \Big\lVert \mathbf{P}^{-1}\big(\mathbf{I}_d-\mathbf{P}^{-1}\big) \Big\rVert_\textnormal{op} \Bigg\lVert \textnormal{diag}\Bigg[\Big(\frac{1}{n} \sum\limits_{i=1}^n \mathbf{x}^{\texttt{NA}} _i\Big)\Big(\frac{1}{n} \sum\limits_{i=1}^n \mathbf{x}^{\texttt{NA}} _i\Big)^\top- \mathbf{m}_x^{\texttt{NA}}\mathbf{m}_x^{\texttt{NA}\top }\Bigg]\Bigg\rVert_\textnormal{op}.
        \end{align*}
        Since the second term is a diagonal matrix, we have 
        \begin{align*}
         & \Bigg\lVert \textnormal{diag}\Bigg[\Big(\frac{1}{n} \sum\limits_{i=1}^n \mathbf{x}^{\texttt{NA}} _i\Big)\Big(\frac{1}{n} \sum\limits_{i=1}^n \mathbf{x}^{\texttt{NA}} _i\Big)^\top- \mathbf{m}_x^{\texttt{NA}}\mathbf{m}_x^{\texttt{NA}\top }\Bigg]\Bigg\rVert_\textnormal{op}\\
        = & \max\limits_{u=1,\dots,d} \Big\lvert \Big(\frac{1}{n} \sum\limits_{i=1}^n \mathbf{x}^{(u)\texttt{NA}} _i\Big)^2 - \big(\mathbf{m}_x^{(u)\texttt{NA}}\big)^2\Big\rvert\\
        = & \max\limits_{u=1,\dots,d} \Big\lvert \Big(\frac{1}{n} \sum\limits_{i=1}^n \mathbf{x}^{(u)\texttt{NA}}_i\Big)^2 - \mathbb{E}\Big[\Big(\frac{1}{n} \sum\limits_{i=1}^n \mathbf{x}^{(u)\texttt{NA}} _i\Big)^2\Big]-\frac{1}{n(n-1)} \mathbb{E}\Big(\sum\limits_{i=1}^n \big(\mathbf{x}_i^{(u)\texttt{NA}}\big)^2\Big)\Big\rvert\\
        \leq  &\max\limits_{u=1,\dots,d} \Big\lvert \Big(\frac{1}{n} \sum\limits_{i=1}^n \mathbf{x}^{(u)\texttt{NA}}_i\Big)^2 - \mathbb{E}\Big[\Big(\frac{1}{n} \sum\limits_{i=1}^n \mathbf{x}^{(u)\texttt{NA}} _i\Big)^2\Big]\Big\rvert + \frac{1}{n-1} \max\limits_{u=1,\dots,d} \Big(\cov_x^{\texttt{NA}}\Big)_{uu} + (\mathbf{m}_x^{(u)\texttt{NA}})^2 \label{eq:1}\\
        \leq & \max\limits_{u=1,\dots,d} \Big\lvert \Big(\frac{1}{n} \sum\limits_{i=1}^n \mathbf{x}^{(u)\texttt{NA}}_i\Big)^2 - \mathbb{E}\Big[\Big(\frac{1}{n} \sum\limits_{i=1}^n \mathbf{x}^{(u)\texttt{NA}} _i\Big)^2\Big]\Big\rvert + \frac{1}{n-1}\Big(\norm{\cov^\texttt{NA}_x}_\textnormal{op} + \norm{\mathbf{m}_x^\texttt{NA}}^2_\infty\Big).
    \end{align*}
    Since the random variable $n^{-1} \sum_{i=1}^n \mathbf{x}^{(u)\texttt{NA}}_i$ is sub-Gaussian, its squared value is sub-exponential. We have by taking the $\psi_1$-norm that
    \begin{align}
        & \Big\lVert \max\limits_{u=1,\dots,d} \Big\lvert \Big(\frac{1}{n} \sum\limits_{i=1}^n \mathbf{x}^{(u)\texttt{NA}}_i\Big)^2 - \mathbb{E}\Big[\Big(\frac{1}{n} \sum\limits_{i=1}^n \mathbf{x}^{(u)\texttt{NA}} _i\Big)^2\Big]\Big\rvert \Big\rVert_{\psi_1} \\
        \leq & \max\limits_{u=1,\dots,d}\Big\lVert \Big(\frac{1}{n} \sum\limits_{i=1}^n \mathbf{x}^{(u)\texttt{NA}}_i\Big)^2 - \mathbb{E}\Big[\Big(\frac{1}{n} \sum\limits_{i=1}^n \mathbf{x}^{(u)\texttt{NA}} _i\Big)^2\Big] \Big\rVert_{\psi_1}\\
        \lesssim & \max\limits_{u=1,\dots,d}\Big\lVert \Big(\frac{1}{n} \sum\limits_{i=1}^n \mathbf{x}^{(u)\texttt{NA}}_i\Big)^2 \Big\rVert_{\psi_1}
    \end{align}
    where the last inequality is obtained by centering \citep[Exercice 2.7.10]{vershynin2018high}. Moreover
    \begin{align}
        \max\limits_{u=1,\dots,d}\Big\lVert \Big(\frac{1}{n} \sum\limits_{i=1}^n \mathbf{x}^{(u)\texttt{NA}}_i\Big)^2 \Big\rVert_{\psi_1} =  \max\limits_{u=1,\dots,d}\Big\lVert \frac{1}{n} \sum\limits_{i=1}^n \mathbf{x}^{(u)\texttt{NA}}_i\Big\rVert_{\psi_2}^2
    \end{align}
    using \citet[Lemma 2.7.6]{vershynin2018high}. We finally have
    \begin{align}
        \Big\lVert \max\limits_{u=1,\dots,d} \Big\lvert \Big(\frac{1}{n} \sum\limits_{i=1}^n \mathbf{x}^{(u)\texttt{NA}}_i\Big)^2 - \mathbb{E}\Big[\Big(\frac{1}{n} \sum\limits_{i=1}^n \mathbf{x}^{(u)\texttt{NA}} _i\Big)^2\Big]\Big\rvert \Big\rVert_{\psi_1} \lesssim \max\limits_{u=1,\dots,d}\Big\lVert \frac{1}{n} \sum\limits_{i=1}^n \mathbf{x}^{(u)\texttt{NA}}_i\Big\rVert_{\psi_2}^2 \leq n^{-1}K_x \textnormal{Tr}(\cov_x^\texttt{NA}).
    \end{align}
    A final application of Bernstein's inequality gives 
    \begin{align}
        \mathbf{b}_n & \lesssim n^{-1}K_x \textnormal{Tr}(\cov_x^\texttt{NA}) \Big(t^{-\frac{1}{2}}\vee t\Big)+ \frac{1}{n-1}\Big(\norm{\cov^\texttt{NA}_x}_\textnormal{op} + \norm{\mathbf{m}_x^\texttt{NA}}^2_\infty\Big)\\
        & \lesssim \frac{1}{n} \textnormal{Tr}(\cov_x^\texttt{NA}\big(t^{\frac{1}{2}}\vee t\big)+\norm{\cov^\texttt{NA}_x}_\textnormal{op} + \norm{\mathbf{m}_x^\texttt{NA}}^2_\infty\Big)
    \end{align}
    with probability at least $1-e^{-t}$. 
    \paragraph{Concentration of $\mathbf{C}_n$.} We have 
    \begin{align}
        \mathbf{c}_n \leq 2\norm{\mathbf{P}^{-1}\big(\mathbf{I}_d-\mathbf{P}^{-1}\big)}_\textnormal{op} \Big\lVert \textnormal{Diag}\big(\overline{\mathbf{M}}^\texttt{NA}_{xn} \big)\Big\rVert_\textnormal{op}  + 2\norm{\mathbf{P}^{-2}}_\textnormal{op}\Big\lVert \overline{\mathbf{M}}^\texttt{NA}_{xn} \Big\rVert_\textnormal{op}.  
    \end{align}
    Since 
    \begin{align}
        & \Big\lVert \textnormal{Diag}\big(\overline{\mathbf{M}}^\texttt{NA}_{xn} \big)\Big\rVert_\textnormal{op} \leq \Big\lVert \overline{\mathbf{M}}^\texttt{NA}_n \Big\rVert_\textnormal{op} = \norm{\mathbf{m}_x^\texttt{NA} -\mathbf{m}_{xn}^\texttt{NA}}^2_2,
    \end{align}
    we get from \citet{flamary2019concentration}, Equation $(18)$ that 
    \begin{align}
        \mathbf{c}_n \leq\frac{2}{n}\norm{\cov_x^\texttt{NA}}_\textnormal{op}\Big(\mathbf{r}(\cov_x^\texttt{NA})+2\big(\mathbf{r}(\cov_x^\texttt{NA})t\big)^\frac{1}{2}+2t\Big)\Big(\norm{\mathbf{P}^{-1}\big(\mathbf{I}_d-\mathbf{P}^{-1}\big)}_\textnormal{op}+\norm{\mathbf{P}^{-2}}_\textnormal{op} \Big)
    \end{align}
    with probability at least $1-e^{-t}$.

    \paragraph{Putting Everything Together.} Using Lemma \ref{lem:tech_lem_1}, we have 
    \begin{align*}
        \mathbf{b}_n \lesssim \frac{1}{n}\Big( \norm{\mathbf{P}}_\textnormal{op}\textnormal{Tr}(\cov_x) \big(t^{\frac{1}{2}}\vee t\big)+\norm{\mathbf{P}}^2_\textnormal{op}\norm{\cov_x}_\textnormal{op} + \norm{\mathbf{P}}^2_\textnormal{op}\norm{\mathbf{m}_x}^2_\infty\Big)
    \end{align*}
    and
    \begin{align*}
        \mathbf{c}_n 
        \lesssim & \Big(\frac{2}{n}\norm{\mathbf{P}}_\textnormal{op}\textnormal{Tr}\big(\cov_x\big) + \frac{4}{n}\norm{\cov_x}_\textnormal{op}\norm{\mathbf{P}}^{\frac{3}{2}}_\textnormal{op}\big(\mathbf{r}(\cov_x)t\big)^\frac{1}{2}+\frac{4t}{n}\norm{\mathbf{P}}^2_\textnormal{op}\norm{\cov_x}_\textnormal{op} \Big)\Big(\norm{\mathbf{P}^{-1}\big(\mathbf{I}_d-\mathbf{P}^{-1}\big)}_\textnormal{op}+\norm{\mathbf{P}^{-2}}_\textnormal{op} \Big)\\
        \lesssim & \frac{2}{n} \norm{\cov_x}_\textnormal{op}\norm{\mathbf{P}}_\textnormal{op}\Big(\mathbf{r}(\cov_x) + 2 \sqrt{\mathbf{r}(\cov_x)\norm{\mathbf{P}}_\textnormal{op} t} + 2t\Big)\Big(\norm{\mathbf{P}^{-1}\big(\mathbf{I}_d-\mathbf{P}^{-1}\big)}_\textnormal{op}+\norm{\mathbf{P}^{-2}}_\textnormal{op} \Big)\\
        \lesssim & \frac{2}{n} \norm{\cov_x}_\textnormal{op}\Big(\mathbf{r}(\cov_x) + 2 \sqrt{\mathbf{r}(\cov_x)\norm{\mathbf{P}}_\textnormal{op} t} + 2t\Big)\Big(\norm{\big(\mathbf{I}_d-\mathbf{P}^{-1}\big)}_\textnormal{op}+\norm{\mathbf{P}^{-1}}_\textnormal{op} \Big).
    \end{align*}
    Hence with probability at least $1-4e^{-t}$, we have
    \begin{align*}
         \Big\lVert \widehat{\cov}_{xn} - \cov_x \Big\rVert_\textnormal{op}
        \lesssim & \norm{\cov_x}_\textnormal{op} \norm{\mathbf{P}}_\textnormal{op}^{-1}\Bigg(\sqrt{\frac{\mathbf{r}(\cov_x) \log \mathbf{r}(\cov_x)}{n}} \vee \sqrt{\frac{t}{n}} \vee \frac{\mathbf{r}(\cov_x)\big(t+\log \mathbf{r}(\cov_x)\big)}{\norm{\mathbf{P}}_\textnormal{op}n} \log n\Bigg) \\
         + & \frac{1}{n}\Big( \norm{\mathbf{P}}_\textnormal{op}\textnormal{Tr}(\cov_x) \big(t^{-\frac{1}{2}}\vee t\big)+\norm{\mathbf{P}}^2_\textnormal{op}\norm{\cov_x}_\textnormal{op} + \norm{\mathbf{P}}^2_\textnormal{op}\norm{\mathbf{m}_x}^2_\infty\Big)\\
         + & \frac{2}{n} \norm{\cov_x}_\textnormal{op}\Big(\mathbf{r}(\cov_x) + 2 \sqrt{\mathbf{r}(\cov_x)\norm{\mathbf{P}}_\textnormal{op} t} + 2t\Big)\Big(\norm{\big(\mathbf{I}_d-\mathbf{P}^{-1}\big)}_\textnormal{op}+\norm{\mathbf{P}^{-1}}_\textnormal{op} \Big)\\
         =&  \norm{\cov_x}_\textnormal{op} \norm{\mathbf{P}}_\textnormal{op}^{-1}\Bigg(\sqrt{\frac{\mathbf{r}(\cov_x) \log \mathbf{r}(\cov_x)}{n}} \vee \sqrt{\frac{t}{n}} \vee \frac{\mathbf{r}(\cov_x)\big(t+\log \mathbf{r}(\cov_x)\big)}{\norm{\mathbf{P}}_\textnormal{op}n} \log n\Bigg)\\
         +& \frac{1}{n}\norm{\cov_x}_\textnormal{op}\Bigg(\mathbf{r}(\cov_x)\Big(2\norm{\big(\mathbf{I}_d-\mathbf{P}^{-1}\big)}_\textnormal{op}+2\norm{\mathbf{P}^{-1}}_\textnormal{op} + \norm{\mathbf{P}}_\textnormal{op}\big(t^{-\frac{1}{2}}\vee t\big)\Big) + \norm{\mathbf{P}}_\textnormal{op}^2\big(1+\norm{\mathbf{m}_x}_\infty^2\big)\\
         + & 4 \sqrt{\mathbf{r}(\cov_x)\norm{\mathbf{P}}_\textnormal{op} t}\Big(\norm{\big(\mathbf{I}_d-\mathbf{P}^{-1}\big)}_\textnormal{op}+\norm{\mathbf{P}^{-1}}_\textnormal{op} \Big) + 4t\Big(\norm{\big(\mathbf{I}_d-\mathbf{P}^{-1}\big)}_\textnormal{op}+\norm{\mathbf{P}^{-1}}_\textnormal{op} \Big) \Bigg).
    \end{align*}
    Setting
    \begin{align}
    \label{eq:def_r_mu}
        &K^x_1(t):= \norm{\cov_x}_\textnormal{op}\Bigg(\mathbf{r}(\cov_x)\Big(2\norm{\big(\mathbf{I}_d-\mathbf{P}^{-1}\big)}_\textnormal{op}+2\norm{\mathbf{P}^{-1}}_\textnormal{op} + \norm{\mathbf{P}}_\textnormal{op}\big(t^{-\frac{1}{2}}\vee t\big)\Big) + \norm{\mathbf{P}}_\textnormal{op}^2\big(1+\norm{\mathbf{m}_x}_\infty^2\big)\\
         + & 4 \sqrt{\mathbf{r}(\cov_x)\norm{\mathbf{P}}_\textnormal{op} t}\Big(\norm{\big(\mathbf{I}_d-\mathbf{P}^{-1}\big)}_\textnormal{op}+\norm{\mathbf{P}^{-1}}_\textnormal{op} \Big) + 4t\Big(\norm{\big(\mathbf{I}_d-\mathbf{P}^{-1}\big)}_\textnormal{op}+\norm{\mathbf{P}^{-1}}_\textnormal{op} \Big) \Bigg)
    \end{align}
    we finally have the bound
    \begin{align*}
        \Big\lVert  \widehat{\cov}_x - \cov_x \Big\rVert_\textnormal{op} \lesssim & \norm{\cov_x}_\textnormal{op} \norm{\mathbf{P}}_\textnormal{op}^{-1}\Bigg(\sqrt{\frac{\mathbf{r}(\cov_x) \log \mathbf{r}(\cov_x)}{n}} \vee \sqrt{\frac{t}{n}} \vee \frac{\mathbf{r}(\cov_x)\big(t+\log \mathbf{r}(\cov_x)\big)}{\norm{\mathbf{P}}_\textnormal{op}n} \log n\Bigg)
        +  \frac{K^x(t)}{n}.
    \end{align*}
\end{proof}
For $\eta$, we obtain the similar bound
 \begin{align*}
        \Big\lVert  \widehat{\cov}_y - \cov_y \Big\rVert_\textnormal{op} \lesssim & \norm{\cov_y}_\textnormal{op} \norm{\mathbf{Q}}_\textnormal{op}^{-1}\Bigg(\sqrt{\frac{\mathbf{r}(\cov_y) \log \mathbf{r}(\cov_y)}{m}} \vee \sqrt{\frac{t}{m}} \vee \frac{\mathbf{r}(\cov_y)\big(t+\log \mathbf{r}(\cov_m)\big)}{\norm{\mathbf{Q}}_\textnormal{op}m} \log m\Bigg)
        + \frac{K^y_1(t)}{m}
    \end{align*}
with probability $1-4e^{-t}$, where 
\begin{align*}
    &K^y_1(t):= \norm{\cov_x}_\textnormal{op}\Bigg(\mathbf{r}(\cov_x)\Big(2\norm{\big(\mathbf{I}_d-\mathbf{P}^{-1}\big)}_\textnormal{op}+2\norm{\mathbf{P}^{-1}}_\textnormal{op} + \norm{\mathbf{P}}_\textnormal{op}\big(t^{-\frac{1}{2}}\vee t\big)\Big) + \norm{\mathbf{P}}_\textnormal{op}^2\big(1+\norm{\mathbf{m}_x}_\infty^2\big)\\
         + & 4 \sqrt{\mathbf{r}(\cov_x)\norm{\mathbf{P}}_\textnormal{op} t}\Big(\norm{\big(\mathbf{I}_d-\mathbf{P}^{-1}\big)}_\textnormal{op}+\norm{\mathbf{P}^{-1}}_\textnormal{op} \Big) + 4t\Big(\norm{\big(\mathbf{I}_d-\mathbf{P}^{-1}\big)}_\textnormal{op}+\norm{\mathbf{P}^{-1}}_\textnormal{op} \Big) \Bigg)
\end{align*}

\subsection{Proof of Theorem \ref{thm:concentration_bw}}

\begin{proof}
    By the triangle inequality
    \begin{align*}
        \Big\lvert \widehat{\mathbb{BW}}(\mu_n^\textnormal{\texttt{NA}},\eta^\textnormal{\texttt{NA}}_m)-\mathbb{BW}(\mu,\eta)\Big\rvert \leq \abs{\widehat{\mathbf{a}}(\mu_n^\texttt{NA},\eta^\texttt{NA}_m) - \mathbf{a}(\mu,\eta)}  
        +  \abs{\widehat{\mathbf{b}}(\mu_n^\texttt{NA},\eta^\texttt{NA}_m)-\mathbf{b}(\mu,\eta)}
        + \abs{\widehat{\mathbf{c}}(\mu_n^\texttt{NA},\eta^\texttt{NA}_m) - \mathbf{c}(\mu,\eta)}.
    \end{align*}
    We concentrate all three terms separately. 

    \paragraph{Concentration of $\abs{\widehat{\mathbf{a}}(\mu_n^\texttt{NA},\eta^\texttt{NA}_m) - \mathbf{a}(\mu,\eta)}$.}

    From the triangle inequality
    \begin{align}
    \label{eq:eq_2}
        \abs{\widehat{\mathbf{a}}(\mu_n^\texttt{NA},\eta^\texttt{NA}_m) - \mathbf{a}(\mu,\eta)} \leq &  \Big\lvert \frac{1}{n}\norm{\sum\limits_{i=1}^{n} \mathbf{P}^{-\frac{1}{2}}\mathbf{x}_i^{\textnormal{\texttt{NA}}} }_2^2-\frac{1}{n}\sum\limits_{i\neq j}^n \mathbf{x}_i^{\texttt{NA}\top }\mathbf{P}^{-1}\mathbf{x}_j^{\texttt{NA}} - \textnormal{Tr}\big(\widehat{\cov}_x\big) - \mathbf{m}_x^\top\mathbf{m}_x \Big\rvert\\
        \label{eq:eq_3}
         + & \Big\lvert \frac{1}{m}\norm{\sum\limits_{i=1}^{m} \mathbf{Q}^{-\frac{1}{2}}\mathbf{y}_i^{\textnormal{\texttt{NA}}} }_2^2-\frac{1}{m}\sum\limits_{i\neq j}^n \mathbf{y}_i^{\texttt{NA}\top }\mathbf{Q}^{-1}\mathbf{y}_j^{\texttt{NA}} - \textnormal{Tr}\big(\widehat{\cov}_y\big) - \mathbf{m}_y^\top\mathbf{m}_y \Big\rvert\\
         \label{eq:eq_4}
         + & 2\Big\lvert \mathbf{m}_x^\top \mathbf{m}_y - \mathbf{m}^{\textnormal{\texttt{NA}}\top}_{xn}\big(\mathbf{P}\mathbf{Q}\big)^{-1}\mathbf{m}^{\textnormal{\texttt{NA}}}_{ym} \Big\rvert.  
    \end{align}

    Starting with Equation \eqref{eq:eq_4}, we have 
    \begin{align}
        \Big\lvert \mathbf{m}_x^\top \mathbf{m}_y - \mathbf{m}^{\textnormal{\texttt{NA}}\top}_{xn}\big(\mathbf{P}\mathbf{Q}\big)^{-1}\mathbf{m}^{\textnormal{\texttt{NA}}}_{ym} \Big\rvert = \Big\lvert \frac{1}{nm}\sum\limits_{i,j} \Big[\mathbf{x}^\texttt{NA}_i \big(\mathbf{PQ}\big)^{-1}\mathbf{y}^\texttt{NA}_j - \mathbf{m}_x^\top\mathbf{m}_y \Big]\Big\rvert.  
    \end{align}
    This random variable is sub-exponential. Indeed, 
    \begin{align}
        & \Big\lVert \frac{1}{nm}\sum\limits_{i,j} \Big[\mathbf{x}^\texttt{NA}_i \big(\mathbf{PQ}\big)^{-1}\mathbf{y}^\texttt{NA}_j - \mathbf{m}_x^\top\mathbf{m}_y \Big]\Big\rVert_{\psi_1}\\
        =& \Big\lVert \frac{1}{nm}\sum\limits_{i,j} \Big[\mathbf{x}^\texttt{NA}_i \big(\mathbf{PQ}\big)^{-1}\mathbf{y}^\texttt{NA}_j\Big]\Big\rVert_{\psi_1}
    \end{align}
    by centering. Using \citet[Lemma 2.7.7]{vershynin2018high}, we obtain
    \begin{align}
        \Big\lVert \frac{1}{nm}\sum\limits_{i,j} \Big[\mathbf{x}^\texttt{NA}_i \big(\mathbf{PQ}\big)^{-1}\mathbf{y}^\texttt{NA}_j\Big]\Big\rVert_{\psi_1} \leq &\Big\lVert \frac{1}{n}\sum\limits_{i=1}^n  \mathbf{P}^{-1} \mathbf{x}_i^\texttt{NA} \Big\rVert_{\psi_2}\Big\lVert \frac{1}{m}\sum\limits_{j=1}^m  \mathbf{Q}^{-1} \mathbf{y}_j^\texttt{NA} \Big\rVert_{\psi_2} \\
        \leq & \frac{1}{n} \norm{\mathbf{P}^{-1}}_\textnormal{op} \norm{\mathbf{x}_i^\texttt{NA}}_{\psi_2}\frac{1}{m} \norm{\mathbf{Q}^{-1}}_\textnormal{op} \norm{\mathbf{y}_j^\texttt{NA}}_{\psi_2}.
    \end{align}
    Applying Bernstein's inequality gives 
    \begin{align*}
         \Big\lvert \mathbf{m}_x^\top \mathbf{m}_y - \mathbf{m}^{\textnormal{\texttt{NA}}\top}_{xn}\big(\mathbf{P}\mathbf{Q}\big)^{-1}\mathbf{m}^{\textnormal{\texttt{NA}}}_{ym} \Big\rvert \lesssim \norm{\mathbf{P}^{-1}}_\textnormal{op}\norm{\mathbf{Q}^{-1}}_\textnormal{op} \frac{1}{nm}\big(t \vee t^{\frac{1}{2}}\big)
    \end{align*}
    with probability at least $1-e^{-t}$. We now turn to Equation \eqref{eq:eq_2} --- Equation \eqref{eq:eq_3} can be handled in a similar fashion. We have that
    \begin{align}
        & \Big\lvert \frac{1}{n}\sum\limits_{i=1}^{n}\norm{ \mathbf{P}^{-\frac{1}{2}}\mathbf{x}_i^{\textnormal{\texttt{NA}}} }_2^2 - \textnormal{Tr}\big(\widehat{\cov}_x\big) - \mathbf{m}_x^\top\mathbf{m}_x \Big\rvert \\
        = &\Bigg\lvert \frac{1}{n}\sum\limits_{i=1}^{n}\norm{ \mathbf{P}^{-\frac{1}{2}}\mathbf{x}_i^{\textnormal{\texttt{NA}}} }_2^2 - \textnormal{Tr}\big(\widehat{\cov}_x\big) - \mathbb{E}\Big[\frac{1}{n}\sum\limits_{i=1}^{n}\norm{ \mathbf{P}^{-\frac{1}{2}}\mathbf{x}_i^{\textnormal{\texttt{NA}}} }_2^2 - \textnormal{Tr}\big(\widehat{\cov}_x\big) \Big]\Bigg\rvert\\
        \leq &  \frac{1}{n}\Bigg\lvert\sum\limits_{i=1}^{n}\norm{ \mathbf{P}^{-\frac{1}{2}}\mathbf{x}_i^{\textnormal{\texttt{NA}}} }_2^2 -\mathbb{E}\Big[\norm{ \mathbf{P}^{-\frac{1}{2}}\mathbf{x}_i^{\textnormal{\texttt{NA}}}}_2^2\Big]\Bigg\rvert + \Bigg\lvert \textnormal{Tr}\big(\widehat{\cov}_{xn}\big) - \textnormal{Tr}\big(\cov_x\big)\Bigg\rvert. 
    \end{align}
    Starting with the first term of this sum, we have by centering
    \begin{align*}
       \Bigg\lVert \frac{1}{n}\Bigg\lvert\sum\limits_{i=1}^{n}\norm{ \mathbf{P}^{-\frac{1}{2}}\mathbf{x}_i^{\textnormal{\texttt{NA}}} }_2^2 -\mathbb{E}\Big[\norm{ \mathbf{P}^{-\frac{1}{2}}\mathbf{x}_i^{\textnormal{\texttt{NA}}}}_2^2\Big]\Bigg\rvert \Bigg\rVert_{\psi_1} \lesssim \Bigg\lVert \frac{1}{n}\sum\limits_{i=1}^{n}\norm{ \mathbf{P}^{-\frac{1}{2}}\mathbf{x}_i^{\textnormal{\texttt{NA}}} }_2^2 \Bigg\rVert_{\psi_1} \leq \norm{\mathbf{P}^{-1}}_\textnormal{op}\frac{1}{n} \sum\limits_{i=1}^n\big\lVert \mathbf{x}_i^{\texttt{NA}}\big\rVert_{\psi_2}^2.
    \end{align*}
    Applying Bernstein's inequality gives with probability at least $1-e^{-t}$ 
    \begin{align*}
        \frac{1}{n}\Bigg\lvert\sum\limits_{i=1}^{n}\norm{ \mathbf{P}^{-\frac{1}{2}}\mathbf{x}_i^{\textnormal{\texttt{NA}}} }_2^2 -\mathbb{E}\Big[\norm{ \mathbf{P}^{-\frac{1}{2}}\mathbf{x}_i^{\textnormal{\texttt{NA}}}}_2^2\Big]\Bigg\rvert \lesssim \norm{\mathbf{P}^{-1}}_\textnormal{op} \Bigg(\sqrt{\frac{t}{n}} \vee \frac{t}{n}\Bigg).
    \end{align*}
    We now turn to the second term. From Lemma \ref{lem:tech_lem_3}, we have 
    \begin{align}
        \textnormal{Tr}\Big(\widehat{\cov}_{xn}\Big) = \textnormal{Tr}\Big(\mathbf{P}^{-1}\mathbf{W}_{xn}^\textnormal{\texttt{NA}}\Big) - 4\big\lVert\mathbf{P}^{-1}\varepsilon_{xn}^\textnormal{\texttt{NA}}\big\rVert_\textnormal{2}^2 + 2\big\lVert\mathbf{P}^{-\frac{1}{2}}\varepsilon_{xn}^\textnormal{\texttt{NA}} \big\rVert_2^2 - \Big\lVert \mathbf{P}^{-\frac{1}{2}}\big(\mathbf{I}_d-\mathbf{P}^{-1}\big)^{\frac{1}{2}}\mathbf{m}^\textnormal{\texttt{NA}}_{xn}\Big\rVert_2^2
    \end{align}
where 
\begin{align*}
     & \mathbf{W}_{xn}^\texttt{NA}:= \frac{1}{n}\sum\limits_{i=1}^n \big(\mathbf{x}_i^{\texttt{NA}} - \mathbf{m}_x^\texttt{NA}\big)\big(\mathbf{x}_i^{\texttt{NA}} - \mathbf{m}_x^\texttt{NA}\big)^\top\\
     & \varepsilon_{xn}^\texttt{NA} := \mathbf{m}^\texttt{NA}_{xn}-\mathbf{m}^\texttt{NA}_{x}.
 \end{align*}
    We also have  
    \begin{align*}
        \textnormal{Tr}\big(\cov_x\big) = \textnormal{Tr}\big(\mathbf{P}^{-1}\cov_x^\texttt{NA}\big) - \mathbf{m}_{x}^{\texttt{NA}\top}\mathbf{P}^{-1}\big(\mathbf{I}_d-\mathbf{P}^{-1}\big)\mathbf{m}_{x}^\texttt{NA}.
    \end{align*}
    Putting everything together, we obtain 
    \begin{align*}
         & \textnormal{Tr}\big(\widehat{\cov}_{xn}\big) - \textnormal{Tr}\big(\cov_x\big) \\
         = &  \textnormal{Tr}\Bigg( \mathbf{P}^{-1}\mathbf{W}_{xn}^\texttt{NA} -\mathbf{P}^{-1} \cov_x^\texttt{NA}\Bigg) \\
         - & 4\big\lVert\mathbf{P}^{-1}\varepsilon_{xn}^\textnormal{\texttt{NA}}\big\rVert_\textnormal{2}^2 + 2\big\lVert\mathbf{P}^{-\frac{1}{2}}\varepsilon_{xn}^\textnormal{\texttt{NA}} \big\rVert_2^2 - \Big\lVert \mathbf{P}^{-\frac{1}{2}}\big(\mathbf{I}_d-\mathbf{P}^{-1}\big)^{\frac{1}{2}}\mathbf{m}^\textnormal{\texttt{NA}}_{xn}\Big\rVert_2^2\\
         +& \Big\lVert \mathbf{P}^{-\frac{1}{2}}\big(\mathbf{I}_d-\mathbf{P}^{-1}\big)^{\frac{1}{2}}\mathbf{m}^\textnormal{\texttt{NA}}_{x}\Big\rVert_2^2\\
         =& \textnormal{Tr}\Bigg( \mathbf{P}^{-1}\mathbf{W}_{xn}^\texttt{NA} -\mathbf{P}^{-1} \cov_x^\texttt{NA}\Bigg) + \Big\lVert \mathbf{P}^{-\frac{1}{2}}\big(\mathbf{I}_d-\mathbf{P}^{-1}\big)^{\frac{1}{2}}\big(\mathbf{m}^\textnormal{\texttt{NA}}_{x} - \mathbf{m}^\textnormal{\texttt{NA}}_{xn}\big)\Big\rVert_2^2\\
        + & 2 \mathbf{m}^{\textnormal{\texttt{NA}}\top}_{x}\mathbf{P}^{-1}\big(\mathbf{I}_d-\mathbf{P}^{-1}\big)\mathbf{m}^\textnormal{\texttt{NA}}_{xn} - 2\norm{\mathbf{P}^{-\frac{1}{2}}\big(\mathbf{I}_d-\mathbf{P}^{-1}\big)^{\frac{1}{2}} \mathbf{m}_{xn}^\texttt{NA}} - 4\big\lVert\mathbf{P}^{-1}\varepsilon_{xn}^\textnormal{\texttt{NA}}\big\rVert_\textnormal{2}^2 + 2\big\lVert\mathbf{P}^{-\frac{1}{2}}\varepsilon_{xn}^\textnormal{\texttt{NA}} \big\rVert_2^2\\
        = & \textnormal{Tr}\Bigg( \mathbf{P}^{-1}\mathbf{W}_{xn}^\texttt{NA} -\mathbf{P}^{-1} \cov_x^\texttt{NA}\Bigg) + \Big\lVert \mathbf{P}^{-\frac{1}{2}}\big(\mathbf{I}_d-\mathbf{P}^{-1}\big)^{\frac{1}{2}} \varepsilon_{xn}^\textnormal{\texttt{NA}}\Big\rVert_2^2  \\
        + & 2 \mathbf{m}^{\textnormal{\texttt{NA}}\top}_{x}\mathbf{P}\big(\mathbf{I}_d-\mathbf{P}^{-1}\big)\varepsilon_{xn}^\textnormal{\texttt{NA}} - 4\big\lVert\mathbf{P}^{-1}\varepsilon_{xn}^\textnormal{\texttt{NA}}\big\rVert_\textnormal{2}^2 + 2\big\lVert\mathbf{P}^{-\frac{1}{2}}\varepsilon_{xn}^\textnormal{\texttt{NA}} \big\rVert_2^2.
    \end{align*}
    First, we have 
    \begin{align*}
        \Big\lvert \textnormal{Tr}\Bigg( \mathbf{P}^{-1}\mathbf{W}_{xn}^\texttt{NA} -\mathbf{P}^{-1} \cov_x^\texttt{NA}\Bigg)\big\rvert \leq & \norm{\mathbf{P}^{-1}}_\textnormal{op}\Big\lvert \textnormal{Tr}\Big(\mathbf{W}_{xn}^\texttt{NA} -\cov_x^\texttt{NA}\Big)\big\rvert\\
        =&  \norm{\mathbf{P}^{-1}}_\textnormal{op}\Big\lvert \frac{1}{n}\sum\limits_{i=1}^n \norm{\mathbf{x}_i^\texttt{NA} - \mathbf{m}_x^\texttt{NA}}_2^2 - \mathbb{E}\big(\norm{X^\texttt{NA} - \mathbf{m}_x^\texttt{NA}}_2^2\big)\Big\rvert.
    \end{align*}
   By centering, we have 
   \begin{align*}
       \Big\lVert\frac{1}{n}\sum\limits_{i=1}^n \norm{\mathbf{x}_i^\texttt{NA} - \mathbf{m}_x^\texttt{NA}}_2^2 - \mathbb{E}\big(\norm{X^\texttt{NA} - \mathbf{m}_x^\texttt{NA}}_2^2\big) \Big\rVert_{\psi_1} \lesssim  \Big\lVert\frac{1}{n}\sum\limits_{i=1}^n \norm{\mathbf{x}_i^\texttt{NA} - \mathbf{m}_x^\texttt{NA}}_2^2 \Big\rVert_{\psi_1}.
   \end{align*}
   We also observe that 
   \begin{align*}
       \Big\lVert \norm{\mathbf{x}_i^\texttt{NA} - \mathbf{m}_x^\texttt{NA}}_2^2 \Big\rVert_{\psi_1} \lesssim \Big\lVert \norm{\mathbf{x}_i^\texttt{NA}}_2 \Big\rVert_{\psi_2}^2.
   \end{align*}
   Applying Bernstein's inequality leaves us with
   \begin{align*}
       \Big\lvert \textnormal{Tr}\Bigg( \mathbf{P}^{-1}\mathbf{W}_{xn}^\texttt{NA} -\mathbf{P}^{-1} \cov_x^\texttt{NA}\Bigg)\big\rvert \lesssim  \norm{\mathbf{P}^{-1}}_\textnormal{op} \Bigg(\sqrt{\frac{t}{n}} \vee \frac{t}{n}\Bigg)
   \end{align*}
   with probability at least $1-e^{-t}$. Turning to the remaining terms, we have 
   \begin{align*}
      & \Bigg\lvert  \Big\lVert \mathbf{P}^{-\frac{1}{2}}\big(\mathbf{I}_d-\mathbf{P}^{-1}\big)^{\frac{1}{2}} \varepsilon_{xn}^\textnormal{\texttt{NA}}\Big\rVert_2^2 
        + 2 \mathbf{m}^{\textnormal{\texttt{NA}}\top}_{x}\mathbf{P}^{-1}\big(\mathbf{I}_d-\mathbf{P}^{-1}\big)\varepsilon_{xn}^\textnormal{\texttt{NA}} - 4\big\lVert\mathbf{P}^{-1}\varepsilon_{xn}^\textnormal{\texttt{NA}}\big\rVert_\textnormal{2}^2 + 2\big\lVert\mathbf{P}^{-\frac{1}{2}}\varepsilon_{xn}^\textnormal{\texttt{NA}} \big\rVert_2^2 \Bigg\rvert \\
        \leq & \Big\lVert \mathbf{P}^{-1}\big(\mathbf{I}_d-\mathbf{P}^{-1}\big) \Big\rVert_\textnormal{op} \Big\lVert \varepsilon_{xn}^\textnormal{\texttt{NA}}\Big\rVert_2^2 + 2\Big\lVert\mathbf{P}^{-1}\big(\mathbf{I}_d-\mathbf{P}^{-1}\big) \Big\rVert_\textnormal{op} \norm{\mathbf{P}\mathbf{m}_x}_2^2 \Big\lVert \varepsilon_{xn}^\textnormal{\texttt{NA}}\Big\rVert_2^2 - 4 \Big\lVert\mathbf{P}^{-2} \Big\rVert_\textnormal{op}\Big\lVert \varepsilon_{xn}^\textnormal{\texttt{NA}}\Big\rVert_2^2 + 2 \Big\lVert\mathbf{P}^{-1} \Big\rVert_\textnormal{op}\Big\lVert \varepsilon_{xn}^\textnormal{\texttt{NA}}\Big\rVert_2^2\\
        \leq & \Bigg[\Big\lVert \mathbf{P}^{-1}\big(\mathbf{I}_d-\mathbf{P}^{-1}\big) \Big\rVert_\textnormal{op} \big(1 + 2\norm{\mathbf{P}\mathbf{m}_x}_2^2\big) +2 \big(\Big\lVert\mathbf{P}^{-1}\Big\rVert_\textnormal{op} - 2\Big\lVert\mathbf{P}^{-2}\Big\rVert_\textnormal{op}\big) \Bigg]\Big\lVert \varepsilon_{xn}^\textnormal{\texttt{NA}}\Big\rVert_2^2.
   \end{align*}
   To conclude the proof, we use \citet[Equation 18]{flamary2019concentration} to get 
   \begin{align*}
       \Big\lVert \varepsilon_{xn}^\textnormal{\texttt{NA}}\Big\rVert_2^2 \lesssim \frac{1}{n} \norm{\cov_x^\texttt{NA}}_\textnormal{op}\Big[ \mathbf{r}\big(\cov_x^\texttt{NA}\big) + 2\sqrt{ \mathbf{r}\big(\cov_x^\texttt{NA}\big) t} + 2t\Big]
   \end{align*}
   with probability at least $1-e^{-t}$. Now, since 
   \begin{align*}
       \mathbf{r}\big(\cov_x^\texttt{NA}\big) = \frac{\trace\big(\cov_x^\texttt{NA}\big)}{\norm{\big(\cov_x^\texttt{NA}\big)}_\textnormal{op}},
   \end{align*}
   we have 
   \begin{align*}
        \mathbf{r}\big(\cov_x^\texttt{NA}\big)\norm{\cov_x^\texttt{NA}}_\textnormal{op} = \trace\big(\cov_x^\texttt{NA}\big) \leq & \norm{\mathbf{P}}_{\textnormal{op}} \big(\textnormal{Tr}\big(\cov_x\big) + \norm{\mathbf{m}_x}_2^2\big)\\
        = & \norm{\mathbf{P}}_{\textnormal{op}} \big(\norm{\cov_x}_\textnormal{op} \mathbf{r}\big(\cov_x\big) + \norm{\mathbf{m}_x}_2^2\big)
   \end{align*}
   and 
   \begin{align*}
        \mathbf{r}\big(\cov_x^\texttt{NA}\big)\norm{\cov_x^\texttt{NA}}^2_\textnormal{op} \leq \norm{\mathbf{P}}_{\textnormal{op}}^2 \norm{\cov_x}_\textnormal{op} \big(\norm{\cov_x}_\textnormal{op} \mathbf{r}\big(\cov_x\big) + \norm{\mathbf{m}_x}_2^2\big)
   \end{align*}
   from Lemma \ref{lem:tech_lem_1}, and finally obtain 
    \begin{align*}
        \Big\lVert \varepsilon_{xn}^\textnormal{\texttt{NA}}\Big\rVert_2^2 \lesssim & \frac{K_2^x\Big(1+ 2\sqrt{K_2^x \norm{\mathbf{P}}_{\textnormal{op}}\norm{\cov_x}_\textnormal{op} t}\Big)}{n} + \frac{2t}{n}
    \end{align*}
    where 
   \begin{align}
       K_2^x := \norm{\mathbf{P}}_{\textnormal{op}} \big(\norm{\cov_x}_\textnormal{op} \mathbf{r}\big(\cov_x\big) + \norm{\mathbf{m}_x}_2^2\big).
   \end{align}
   With a similar definition of 
   \begin{align}
        K_2^y := \norm{\mathbf{Q}}_{\textnormal{op}} \big(\norm{\cov_y}_\textnormal{op} \mathbf{r}\big(\cov_y\big) + \norm{\mathbf{m}_y}_2^2\big),
   \end{align}
   we have in the end that 
   \begin{align*}
         \abs{\widehat{\mathbf{a}}(\mu_n^\texttt{NA},\eta^\texttt{NA}_m) - \mathbf{a}(\mu,\eta)} 
         \lesssim & \norm{\mathbf{P}^{-1}}_\textnormal{op}\norm{\mathbf{Q}^{-1}}_\textnormal{op} \frac{1}{nm}\big(t \vee t^{\frac{1}{2}}\big) +  \norm{\mathbf{P}^{-1}}_\textnormal{op} \Bigg(\sqrt{\frac{t}{n}} \vee \frac{t}{n}\Bigg) + \norm{\mathbf{Q}^{-1}}_\textnormal{op}\Bigg(\sqrt{\frac{t}{m}} \vee \frac{t}{m}\Bigg)\\
         + & \frac{K_2^xK_3^x \Big(1+ 2\sqrt{K_2^x \norm{\mathbf{P}}_{\textnormal{op}}\norm{\cov_x}_\textnormal{op} t}\Big)+2K_3^xt}{n}\\
         + & \frac{K_2^yK_3^y \Big(1+ 2\sqrt{K_2^y \norm{\mathbf{Q}}_{\textnormal{op}}\norm{\cov_y}_\textnormal{op} t}\Big)+2K_3^yt}{m} 
   \end{align*}
   with probability at least $1-5e^{-t}$, where 
   \begin{align*}
       & K^x_3 := \Big\lVert \mathbf{P}^{-1}\big(\mathbf{I}_d-\mathbf{P}^{-1}\big) \Big\rVert_\textnormal{op} \big(1 + 2\norm{\mathbf{P}\mathbf{m}_x}_2^2\big) +2 \big(\Big\lVert\mathbf{P}^{-1}\Big\rVert_\textnormal{op} - 2\Big\lVert\mathbf{P}^{-2}\Big\rVert_\textnormal{op}\big) \\
       & K^y_3 := \Big\lVert \mathbf{Q}^{-1}\big(\mathbf{I}_d-\mathbf{Q}^{-1}\big) \Big\rVert_\textnormal{op} \big(1 + 2\norm{\mathbf{Q}\mathbf{m}_y}_2^2\big) +2 \big(\Big\lVert\mathbf{Q}^{-1}\Big\rVert_\textnormal{op} - 2\Big\lVert\mathbf{Q}^{-2}\Big\rVert_\textnormal{op}\big) . 
   \end{align*}
   
    \paragraph{Concentration of $\abs{\widehat{\mathbf{b}}(\mu_n^\texttt{NA},\eta^\texttt{NA}_m)-\mathbf{b}(\mu,\eta)}$. } We have
    \begin{align*}
        \abs{\widehat{\mathbf{b}}(\mu_n^\texttt{NA},\eta^\texttt{NA}_m)-\mathbf{b}(\mu,\eta)} \leq  \abs{\textnormal{Tr}\big(\widehat{\cov}_{xn}\big) - \textnormal{Tr}\big(\cov_x\big)} + \abs{\textnormal{Tr}\big(\widehat{\cov}_{ym}\big) - \textnormal{Tr}\big(\cov_y\big)}. 
    \end{align*}
    The same argument as above gives 
    \begin{align*}
        \abs{\widehat{\mathbf{b}}(\mu_n^\texttt{NA},\eta^\texttt{NA}_m)-\mathbf{b}(\mu,\eta)} \lesssim  & \norm{\mathbf{P}^{-1}}_\textnormal{op} \Bigg(\sqrt{\frac{t}{n}} \vee \frac{t}{n}\Bigg) +  \norm{\mathbf{Q}^{-1}}_\textnormal{op}\Bigg(\sqrt{\frac{t}{m}} \vee \frac{t}{m}\Bigg) \\
        + &  \frac{K_2^xK_3^x \Big(1+ 2\sqrt{K_2^x \norm{\mathbf{P}}_{\textnormal{op}}\norm{\cov_x}_\textnormal{op} t}\Big)+2K_3^xt}{n}\\
         + & \frac{K_2^yK_3^y \Big(1+ 2\sqrt{K_2^y \norm{\mathbf{Q}}_{\textnormal{op}}\norm{\cov_y}_\textnormal{op} t}\Big)+2K_3^yt}{m}
    \end{align*}
    with probability at least $1-2e^{-t}$.

    \paragraph{Concentration of $\abs{\widehat{\mathbf{c}}(\mu_n^\texttt{NA},\eta^\texttt{NA}_m)-\mathbf{c}(\mu,\eta)}$.} We have 
    \begin{align*}
       & \big\lvert \textnormal{Tr} \Big(\widehat{\cov}_{xn}^{\frac{1}{2}}\widehat{\cov}_{ym}\widehat{\cov}_{xn}^{\frac{1}{2}}\Big)^\frac{1}{2} - \textnormal{Tr}\Big(\cov^\frac{1}{2}_x\cov_y\cov_x^\frac{1}{2}\Big)^{\frac{1}{2}} \big\rvert \\
        \leq & \frac{\norm{\widehat{\cov}_{xn}}_\textnormal{op}}{\lambda_\textnormal{min}\Big(\big(\cov^\frac{1}{2}_x\cov_y\cov_x^\frac{1}{2}\big)^\frac{1}{2}\Big)} \Bigg[\abs{\textnormal{Tr}\Big[\big(\widehat{\cov}_{ym}-\cov_y\big)\Big]}+\frac{\norm{\cov_y}_\textnormal{op}}{\lambda_\textnormal{min}\big(\cov_x^\frac{1}{2}\big)}\abs{\textnormal{Tr}\big[\big(\widehat{\cov}_{xn}-\cov_x\big)\big]}\Bigg] \label{eq:eq_5}
    \end{align*}
    from Lemma \ref{lem:tech_lem_5}. Using the same arguments as above, we first obtain
    \begin{align*}
        & \abs{\textnormal{Tr}\Big[\big(\widehat{\cov}_{ym}-\cov_y\big)\Big]}+\frac{\norm{\cov_y}_\textnormal{op}}{\lambda_\textnormal{min}\big(\cov_x^\frac{1}{2}\big)}\abs{\textnormal{Tr}\big[\big(\widehat{\cov}_{xn}-\cov_x\big)\big]} \\
        \lesssim  &  \norm{\mathbf{Q}^{-1}}_\textnormal{op}\Bigg(\sqrt{\frac{t}{m}} \vee \frac{t}{m}\Bigg) +\frac{K_2^yK_3^y \Big(1+ 2\sqrt{K_2^y \norm{\mathbf{Q}}_{\textnormal{op}}\norm{\cov_y}_\textnormal{op} t}\Big)+2K_3^yt}{m}  \\
        + & \frac{\norm{\cov_y}_\textnormal{op}}{\lambda_\textnormal{min}\big(\cov_x^\frac{1}{2}\Big)}\Bigg(\norm{\mathbf{P}^{-1}}_\textnormal{op}\Bigg(\sqrt{\frac{t}{n}} \vee \frac{t}{n}\Bigg) +  \frac{K_2^xK_3^x \Big(1+ 2\sqrt{K_2^x \norm{\mathbf{P}}_{\textnormal{op}}\norm{\cov_x}_\textnormal{op} t}\Big)+2K_3^xt}{n}\Bigg)\\
    \end{align*}
    with probability at least $1-2e^{-t}$.
    The constant in front of Equation \eqref{eq:eq_5} remains to be bounded. We have
    \begin{align*}
        \norm{\widehat{\cov}_{xn}}_\textnormal{op} \leq & \norm{\widehat{\cov}_{xn}-\cov_{x}}_\textnormal{op} + \norm{\cov_{x}}_\textnormal{op}\\
        \lesssim & \norm{\cov_{x}}_\textnormal{op} + \norm{\cov_x}_\textnormal{op} \norm{\mathbf{P}}_\textnormal{op}^{-1}\Bigg(\sqrt{\frac{\mathbf{r}(\cov_x) \log \mathbf{r}(\cov_x)}{n}} \vee \sqrt{\frac{t}{n}} \vee \frac{\mathbf{r}(\cov_x)\big(t+\log \mathbf{r}(\cov_x)\big)}{\norm{\mathbf{P}}_\textnormal{op}} \log n\Bigg) + \frac{1}{n}K_1^x(t)\\
        =& \norm{\cov_{x}}_\textnormal{op} + \mathbf{\Delta}_n\big(\mu^\texttt{NA},t\big) + \frac{1}{n}K_1^x(t)
    \end{align*}
    with probability at least $1-e^{-t}$.
    Combining the above gives the upper bound
    \begin{align*}
        & \Bigg[\norm{\cov_{x}}_\textnormal{op} +\mathbf{\Delta}_n\big(\mu^\texttt{NA},t\big)+   \frac{1}{n}K_1^x(t)\Bigg] \\
        \times & \Bigg\{\norm{\mathbf{Q}^{-1}}_\textnormal{op}\Bigg(\sqrt{\frac{t}{m}} \vee \frac{t}{m}\Bigg) +\frac{K_2^yK_3^y \Big(1+ 2\sqrt{K_2^y \norm{\mathbf{Q}}_{\textnormal{op}}\norm{\cov_y}_\textnormal{op} t}\Big)+2K_3^yt}{m}  \\
         + & \frac{\norm{\cov_y}_\textnormal{op}}{\lambda_\textnormal{min}\big(\cov_x^\frac{1}{2}\Big)}\Bigg(\norm{\mathbf{P}^{-1}}_\textnormal{op}\Bigg(\sqrt{\frac{t}{n}} \vee \frac{t}{n}\Bigg) + \frac{K_2^xK_3^x \Big(1+ 2\sqrt{K_2^x \norm{\mathbf{P}}_{\textnormal{op}}\norm{\cov_x}_\textnormal{op} t}\Big)+2K_3^xt}{n}\Bigg)\Bigg\}
    \end{align*}
    on $\abs{\widehat{\mathbf{c}}(\mu_n^\texttt{NA},\eta^\texttt{NA}_m)-\mathbf{c}(\mu,\eta)}$ with probability at least $1-3e^{-t}$.

    \paragraph{Final bound.} Combining all the above gives us that 
    \begin{align*}
        \Big\lvert \widehat{\mathbb{BW}}(\mu_n^\textnormal{\texttt{NA}},\eta^\textnormal{\texttt{NA}}_m)-\mathbb{BW}(\mu,\eta)\Big\rvert
    \end{align*}
    is upper bounded, up to a multiplicative constant, by
    \begin{align*}
         & \norm{\mathbf{P}^{-1}}_\textnormal{op}\norm{\mathbf{Q}^{-1}}_\textnormal{op} \frac{1}{nm}\big(t \vee t^{\frac{1}{2}}\big) + \norm{\mathbf{P}^{-1}}_\textnormal{op} \Bigg(\sqrt{\frac{t}{n}} \vee \frac{t}{n}\Bigg) +  \norm{\mathbf{Q}^{-1}}_\textnormal{op}\Bigg(\sqrt{\frac{t}{m}} \vee \frac{t}{m}\Bigg) \\
        + & \frac{K_2^xK_3^x \Big(1+ 2\sqrt{K_2^x \norm{\mathbf{P}}_{\textnormal{op}}\norm{\cov_x}_\textnormal{op} t}\Big)+2K_3^xt}{n}
         + \frac{K_2^yK_3^y \Big(1+ 2\sqrt{K_2^y \norm{\mathbf{Q}}_{\textnormal{op}}\norm{\cov_y}_\textnormal{op} t}\Big)+2K_3^yt}{m}. \\
         + &\Bigg[\norm{\cov_{x}}_\textnormal{op} +\mathbf{\Delta}_n\big(\mu^\texttt{NA},t\big)+   \frac{1}{n}K_1^x(t)\Bigg] \\
        \times & \Bigg\{\norm{\mathbf{Q}^{-1}}_\textnormal{op}\Bigg(\sqrt{\frac{t}{m}} \vee \frac{t}{m}\Bigg) +\frac{K_2^yK_3^y \Big(1+ 2\sqrt{K_2^y \norm{\mathbf{Q}}_{\textnormal{op}}\norm{\cov_y}_\textnormal{op} t}\Big)+2K_3^yt}{m}  \\
         + & \frac{\norm{\cov_y}_\textnormal{op}}{\lambda_\textnormal{min}\big(\cov_x^\frac{1}{2}\Big)}\Bigg(\norm{\mathbf{P}^{-1}}_\textnormal{op}\Bigg(\sqrt{\frac{t}{n}} \vee \frac{t}{n}\Bigg) + \frac{K_2^xK_3^x \Big(1+ 2\sqrt{K_2^x \norm{\mathbf{P}}_{\textnormal{op}}\norm{\cov_x}_\textnormal{op} t}\Big)+2K_3^xt}{n}\Bigg)\Bigg\}
    \end{align*}
    with probability at least $1-10e^{-t}$, which we may rewrite as 
    \begin{align*}
        \Big\lvert \widehat{\mathbb{BW}}(\mu_n^\textnormal{\texttt{NA}},\eta^\textnormal{\texttt{NA}}_m)-\mathbb{BW}(\mu,\eta)\Big\rvert \lesssim & \norm{\mathbf{P}^{-1}}_\textnormal{op}\norm{\mathbf{Q}^{-1}}_\textnormal{op} \frac{1}{nm}\big(t \vee t^{\frac{1}{2}}\big) \\
        + & \norm{\mathbf{P}^{-1}}_\textnormal{op} \Bigg(\sqrt{\frac{t}{n}} \vee \frac{t}{n}\Bigg)\Bigg(2+\frac{\norm{\cov_y}_\textnormal{op}\Big(\norm{\cov_{x}}_\textnormal{op} +\mathbf{\Delta}_n\big(\mu^\texttt{NA},t\big) \Big)}{\lambda_\textnormal{min}\big(\cov_x^\frac{1}{2}\Big)}\Bigg)\\
        +&\norm{\mathbf{Q}^{-1}}_\textnormal{op}\Bigg(\sqrt{\frac{t}{m}} \vee \frac{t}{m}\Bigg)\Bigg(2+\norm{\cov_{x}}_\textnormal{op} +\mathbf{\Delta}_n\big(\mu^\texttt{NA},t\big) \Bigg) + \mathcal{O}(n^{-1}) + \mathcal{O}(m^{-1}).
    \end{align*}
\end{proof}
\subsection{Proof of Theorem \ref{thm:DA}}
\label{app:flamary}
\begin{proof}
    Our proof closely follows \citet{flamary2019concentration}. From their Proposition 1, we have for all $f \in \mathcal{H}_K$ that
    \begin{align*}
        \mathbb{E}_{\mathcal{P}_t}\Big[\mathcal{L}\big(Y,f \circ \widehat{\mathbf{T}}^{-1}(X)\big)\Big] \leq \mathscr{R}_s(f) + M_fM_\mathscr{L} \norm{\widehat{\mathbf{A}}^{-1}}_\textnormal{op} \, \mathbb{E}_{X \sim \mu^\texttt{NA}}\Big[\norm{\mathbf{T}_\star(X) - \widehat{\mathbf{T}}(X)}\Big]
    \end{align*}
    Combining this with \citet[Theorem 3]{flamary2019concentration} and Lemma \ref{lem:tech_lem_6} yields
    \begin{align*}
       \mathscr{E}_t\Big(\widehat{f}_{n_\ell} \circ \widehat{\mathbf{T}}\Big) \lesssim & n_\ell^{-\frac{2\beta}{1+2\beta}} + \frac{t}{n_\ell} + M_f M_\mathscr{L}\frac{\norm{\cov_t}^{\frac{1}{2}}}{\sqrt{n_t}}\Big(\mathbf{r}\big(\cov_t\big)+2 \sqrt{\cov_t t}+2t\Big)^{\frac{1}{2}} \\
       + & M_f M_\mathscr{L}\frac{\norm{\mathbf{P}^{-1}}_\textnormal{op}}{\sqrt{n_s}} \sqrt{K_2^x\Big(1+ 2\sqrt{K_2^x \norm{\mathbf{P}}_{\textnormal{op}}\norm{\cov_x}_\textnormal{op} t }+2t\Big)}\\
        + & M_f M_\mathscr{L}\Bigg[\frac{\kappa(\cov_s)}{\sqrt{\lambda_\textnormal{min}\big(\cov^\frac{1}{2}_{s}\cov_{t}\cov_{s}^\frac{1}{2}\big)}} \mathbf{\Delta}_{n_t}(\eta,t)+\frac{\kappa(\cov_t) \norm{\cov_t}_\textnormal{op}}{\sqrt{\lambda_\textnormal{min}\big(\cov^{-1/2}_{n_s}\cov_{n_t}\cov_{n_s}^{-1/2}\big)}}\mathbf{\Delta}_{n_s}(\mu^\texttt{NA},t)\Bigg]K^x_4\\
         + & \mathcal{O}(n_s^{-1}) + \mathcal{O}(n_t^{-1}).
    \end{align*}
    
\end{proof}

\subsection{Proof of Theorem \ref{thm:main_thm_ot}}

\begin{proof}
    Let $\widehat{\mathbf{C}}:= \big(\mathbf{c}(\widehat{\mathbf{x}}^\texttt{NA}_i, \widehat{\mathbf{y}}^\texttt{NA}_j)\big) \in \mathbb{R}^{n \times m} $ be the cost matrix obtained from the data matrices $\widehat{\mathbf{x}}^\texttt{NA},\widehat{\mathbf{y}}^\texttt{NA}$ completed by singular value thresholding. We have
    \begin{align*}
        \Big\lvert \wasserstein_{\widehat{\mathbf{C}}}(\mu^\texttt{NA}_n,\eta^\texttt{NA}_m)  - \wasserstein_{\mathbf{C}}(\mu_n,\eta_m) \Big\rvert \leq \frac{K_\varepsilon}{\sqrt{nm}} \norm{\mathbf{C}-\widehat{\mathbf{C}}}_\textnormal{F}
    \end{align*}
    from Theorem \ref{thm:keriven_thm_1}.
    Let us focus on the term $$\norm{\mathbf{C}-\widehat{\mathbf{C}}}_\textnormal{F} = \sqrt{\sum_{ij} \Big(\mathbf{c}_{ij} - \widehat{\mathbf{c}}_{ij}\big)^2} = \sqrt{\sum_{ij} \big(\norm{\mathbf{M}^{\frac{1}{2}}(\mathbf{x}_i-\mathbf{y}_j)}_2 - \norm{\mathbf{M}^{\frac{1}{2}}(\widehat{\mathbf{x}}^\texttt{NA}_i-\widehat{\mathbf{y}}^\texttt{NA}_j)}_2\Big)^2}.$$ Using Lipschitz continuity of the norm, we immediately have 
    \begin{align*}
        & \sqrt{\sum_{ij} \Big(\norm{\mathbf{M}^{\frac{1}{2}}(\mathbf{x}_i-\mathbf{y}_j)}_2 - \norm{\mathbf{M}^{\frac{1}{2}}(\widehat{\mathbf{x}}^\texttt{NA}_i-\widehat{\mathbf{y}}^\texttt{NA}_j)}_2\Big)^2} \\
        \leq & \sqrt{\sum_{ij} \Big(\norm{\mathbf{M}^{\frac{1}{2}}(\widehat{\mathbf{x}}^\texttt{NA}_i-\mathbf{x}_i)}_2 + \norm{\mathbf{M}^{\frac{1}{2}}(\widehat{\mathbf{y}}^\texttt{NA}_j-\mathbf{y}_j)}_2\Big)^2}\\
         \leq & \sqrt{\sum_{ij} 2\Big(\norm{\mathbf{M}^{\frac{1}{2}}(\widehat{\mathbf{x}}^\texttt{NA}_i-\mathbf{x}_i)}_2 \Big)^2+ 2\Big(\norm{\mathbf{M}^{\frac{1}{2}}(\widehat{\mathbf{y}}^\texttt{NA}_j-\mathbf{y}_j)}_2\Big)^2}\\
         \leq & \sqrt{2 \sum_{ij} \Big(\norm{\mathbf{M}^{\frac{1}{2}}(\widehat{\mathbf{x}}^\texttt{NA}_i-\mathbf{x}_i)}_2 \Big)^2} + \sqrt{2 \sum_{ij} \Big(\norm{\mathbf{M}^{\frac{1}{2}}(\widehat{\mathbf{y}}^\texttt{NA}_j-\mathbf{y}_j)}_2 \Big)^2}
    \end{align*}
    using $\sqrt{a+b} \leq \sqrt{a}+\sqrt{b}$ and $(a+b)^2 \leq 2a^2 + 2b^2$, and consequently 
    \begin{align*}
        \frac{1}{\sqrt{nm}} \norm{\mathbf{C}-\widehat{\mathbf{C}}}_\textnormal{F} \leq \sqrt{\frac{2}{n} \sum_{i} \Big(\norm{\mathbf{M}^{\frac{1}{2}}(\widehat{\mathbf{x}}^\texttt{NA}_i-\mathbf{x}_i)}_2 \Big)^2} +  \sqrt{\frac{2}{m} \sum_{j} \Big(\norm{\mathbf{M}^{\frac{1}{2}}(\widehat{\mathbf{y}}^\texttt{NA}_j-\mathbf{y}_j)}_2 \Big)^2}. 
    \end{align*}
    Now, we recognize
    \begin{align*}
        \sqrt{\frac{2}{n} \sum_{i} \norm{\mathbf{M}^{\frac{1}{2}}(\widehat{\mathbf{x}}^\texttt{NA}_i-\mathbf{x}_i)}_2^2} = \sqrt{\frac{2}{n}}\norm{\mathbf{M}^{\frac{1}{2}}\big(\mathbf{x}-\widehat{\mathbf{x}}^\texttt{NA}\big)}_\textnormal{F}
    \end{align*}
    and we have 
    \begin{align*}
        \sqrt{\frac{2}{n}}\norm{\mathbf{M}^{\frac{1}{2}}\big(\mathbf{x}-\widehat{\mathbf{x}}^\texttt{NA}\big)}_\textnormal{F} \leq  \sqrt{\frac{2}{n}}\opnorm{\mathbf{M}^{\frac{1}{2}}}\norm{\mathbf{x}-\widehat{\mathbf{x}}^\texttt{NA}}_\textnormal{F}
    \end{align*}
    We are now ready to use the theoretical guarantees of \citet{klopp2015matrix}. Under Assumption \ref{asm:bounded_support}, $$\norm{\mathbf{x}}_\infty \leq \mathcal{R}_x.$$
    
    From \citet[Theorem 5]{klopp2015matrix} it now follows that under Assumption \ref{asm:noise_asm}
    \begin{align*}
        \sqrt{\frac{2}{n}}\norm{\mathbf{M}^{\frac{1}{2}}\big(\mathbf{x}-\widehat{\mathbf{x}}^\texttt{NA}\big)}_\textnormal{F} \lesssim & \sqrt{\frac{\opnorm{\mathbf{M}}\textnormal{rank}(\mathbf{X})\Big[\mathcal{R}_x \vee \sigma)^2 \opnorm{\bigp} + \mathcal{R}_x^2\log(n \wedge d)+b^2\log(n+d)\Big]}{\opnorm{\bigp^{-1}}n}}
    \end{align*}
    and similarly 
    \begin{align*}
        \sqrt{\frac{2}{m}}\norm{\mathbf{M}^{\frac{1}{2}}\big(\mathbf{y}-\widehat{\mathbf{y}}^\texttt{NA}\big)}_\textnormal{F} \lesssim & \sqrt{\frac{\opnorm{\mathbf{M}}\textnormal{rank}(\mathbf{Y})\Big[\mathcal{R}_y \vee \sigma)^2 \opnorm{\bigq} + \mathcal{R}_y^2\log(m \wedge d)+b^2\log(m+d)\Big]}{\opnorm{\bigq^{-1}}m}}
    \end{align*}
    Combining everything, we obtain 
    \begin{align*}
         & \Big\lvert \wasserstein_{\widehat{\mathbf{C}}}(\mu^\texttt{NA}_n,\eta^\texttt{NA}_m)  - \wasserstein_{\mathbf{C}}(\mu_n,\eta_m) \Big\rvert \\
         \lesssim  & K_\varepsilon \sqrt{\frac{\opnorm{\mathbf{M}}\textnormal{rank}(\mathbf{X})\Big[(M^\frac{1}{2}\mathcal{R}_x \vee \sigma)^2 \opnorm{\bigp} + M\mathcal{R}_x^2\log(n \wedge d)+b^2\log(n+d)\Big]}{\opnorm{\bigp^{-1}}n}} \\
         + & K_\varepsilon\sqrt{\frac{\opnorm{\mathbf{M}}\textnormal{rank}(\mathbf{Y})\Big[(M^\frac{1}{2}\mathcal{R}_y \vee \sigma)^2 \opnorm{\bigq} + M\mathcal{R}_y^2\log(m \wedge d)+b^2\log(m+d)\Big]}{\opnorm{\bigq^{-1}}m}}\\
         =: & K_\varepsilon\sqrt{\frac{\opnorm{\mathbf{M}}k_x\rho_\mu(n)}{\opnorm{\bigp^{-1}}n}}+K_\varepsilon\sqrt{\frac{\opnorm{\mathbf{M}}k_y \rho_\eta(m)}{\opnorm{\bigq^{-1}}m}}.
    \end{align*}
    Concerning the divergence between the transport plans, we similarly have 
    \begin{align*}
        \kl\big[\mathbf{\Pi}\big(\mathbf{C}\big)\,|\,\mathbf{\Pi}\big(\widehat{\mathbf{C}}\big) \big] 
        \lesssim & K_\varepsilon\sqrt{\frac{\mathscr{H}_\mathbf{M}(\mu) \rho_\mu(n)}{\opnorm{\bigp^{-1}}n}}+K_\varepsilon\sqrt{\frac{\mathscr{H}_\mathbf{M}(\eta) \rho_\eta(m)}{\opnorm{\bigq^{-1}}m}}\\
        + & K'_\varepsilon \Bigg(\frac{\mathscr{H}_\mathbf{M}(\mu) \rho_\mu(n)}{\opnorm{\bigp^{-1}}n}\Bigg)^{\frac{1}{4}} + K'_\varepsilon \Big( \Bigg(\frac{\mathscr{H}_\mathbf{M}(\eta) \rho_\eta(m)}{\opnorm{\bigq^{-1}}m}\Bigg)^\frac{1}{4}.
    \end{align*}
    
\end{proof}

\section{Additional Results}

\label{appendix:additional_results}

\subsection{Proof of Lemma \ref{lem:general_mean}}

\begin{proof}
    We start from the decomposition
    \begin{align*}
        \cov_x^\textnormal{\texttt{NA}} =& \begin{bmatrix}
            \mathbb{E}\Big((X^{(1)\textnormal{\texttt{NA}}}-\mathbf{m}_x^{(1) \na}\big)\big((X^{(1)\textnormal{\texttt{NA}}}-\mathbf{m}_x^{(1)\na}\big)\Big) & \cdots & \mathbb{E}\Big((X^{(1)\textnormal{\texttt{NA}}}-\mathbf{m}_x^{(1)\na}\big)\big((X^{(d)\textnormal{\texttt{NA}}}-\mathbf{m}_x^{(d)\na}\big)\Big)\\
            \vdots & &\vdots \\
           \mathbb{E}\Big((X^{(1)\textnormal{\texttt{NA}}}-\mathbf{m}_x^{(1)\na}\big)\big((X^{(d)\textnormal{\texttt{NA}}}-\mathbf{m}_x^{(d)\na}\big)\Big) & \cdots & \mathbb{E}\Big((X^{(d)\textnormal{\texttt{NA}}}-\mathbf{m}_x^{(d)\na}\big)\big((X^{(d)\textnormal{\texttt{NA}}}-\mathbf{m}_x^{(d)\na}\big)\Big)
        \end{bmatrix} \\
        =  & \mathbf{R}_x^\textnormal{\texttt{NA}} - \mathbf{m}_x^\textnormal{\texttt{NA}}\mathbf{m}_x^{\textnormal{\texttt{NA}}\top}
    \end{align*}
    where 
    \begin{align*}
        \mathbf{R}_x^\textnormal{\texttt{NA}} := \Big(\mathbb{E}(X^{(i)\textnormal{\texttt{NA}}}X^{(j)\textnormal{\texttt{NA}}})\Big)_{i,j=1,\dots, d} 
    \end{align*}
    is the autocorrelation matrix.
    From the decomposition first proposed by \citet{lounici2014high}, we have
    \begin{align*}
        \mathbf{R}_x^\textnormal{\texttt{NA}} = \mathbf{P}  \mathbf{R}_x\mathbf{P} + \mathbf{P}\big(\mathbf{I}_d - \mathbf{P}\big)\textnormal{diag}\big(  \mathbf{R}_x\big)
    \end{align*}
    which gets us 
    \begin{align*}
        \cov_x^\textnormal{\texttt{NA}} = &\mathbf{P}  \cov_x\mathbf{P} + \mathbf{P}\big(\mathbf{I}_d - \mathbf{P}\big)\textnormal{diag}\big(  \mathbf{R}_x\big)\\
         =& \mathbf{P}  \cov_x\mathbf{P} + \mathbf{P}\big(\mathbf{I}_d - \mathbf{P}\big)\textnormal{diag}\big(\cov_x + \mathbf{m}_x\mathbf{m}_x^\top \big)\\
         =& \mathbf{P}  \cov_x\mathbf{P} + \mathbf{P}\big(\mathbf{I}_d - \mathbf{P}\big)\textnormal{diag}\big(\cov_x\big) + \mathbf{P}\big(\mathbf{I}_d - \mathbf{P}\big)\textnormal{diag}\big(\mathbf{m}_x^2\big).
    \end{align*}
    The last claim follows similarly from a direct computation and from
    \begin{align*}
        \mathbf{m}_x^\textnormal{\texttt{NA}}\mathbf{m}_x^{\textnormal{\texttt{NA}}\top} = \mathbf{P}\mathbf{m}_x\mathbf{m}_x^\top \mathbf{P}.
    \end{align*}
\end{proof}

\subsection{Additional Lemma 1}

\begin{lem}
\label{lem:tech_lem_1}
    We have 
    \begin{align*}
          &\textnormal{Tr}\big(\cov_x^\texttt{NA}\big) \leq \norm{\mathbf{P}}_{\textnormal{op}} \big(\textnormal{Tr}\big(\cov_x\big) + \norm{\mathbf{m}_x}_2^2\big), \\
          &\norm{\cov_x^\texttt{NA}}_\textnormal{op}  \leq  \norm{\mathbf{P}}_\textnormal{op}^2 \norm{\cov_x}_\textnormal{op}\\ \textnormal{and} \,\, &  \mathbf{r}\big(\cov_x^\texttt{NA}\big) \leq \norm{\mathbf{P}}^3_{\textnormal{op}}c^{-1}_\sigma\big(\mathbf{r}\big(\cov_x\big) + \norm{\mathbf{m}_x}_2^2\big).
    \end{align*}
\end{lem}

\begin{proof}
    First, we have 
    \begin{align*}
        \textnormal{Tr}\big(\cov_x^\texttt{NA}\big) = &\textnormal{Tr}\big(\mathbf{P}\cov_x\big) + \textnormal{Tr}\Big(\mathbf{P}\big(\mathbf{I}_d-\mathbf{P}\big)\textnormal{diag}\big(\mathbf{m}_x\big)^2\Big)\\
        \leq & \norm{\mathbf{P}}_{\textnormal{op}} \big(\textnormal{Tr}\big(\cov_x\big) + \norm{\mathbf{m}_x}_2^2\big)
    \end{align*}
    since $\textnormal{Tr}\Big(\mathbf{P}\big(\mathbf{I}_d-\mathbf{P}\big)\textnormal{diag}\big(\mathbf{m}_x\big)^2\Big)  \leq \textnormal{Tr}\Big(\mathbf{P}\textnormal{diag}\big(\mathbf{m}_x\big)^2\Big)$. We now turn to the operator norm of $\cov_x^\texttt{NA}$. We have 
    \begin{align*}
        \norm{\cov_x^\texttt{NA}}_\textnormal{op} = &\max\limits_{\norm{u}_2=1} u^\top \cov_x^\texttt{NA} u\\
        = &\max\limits_{\norm{u}_2=1} u^\top \Big(\cov_x^\texttt{NA} - \textnormal{diag}\big(\cov_x^\texttt{NA}\big)+ \textnormal{diag}\big(\cov_x^\texttt{NA}\big)\Big) u\\
        \leq & \max\limits_{\norm{u}_2=1} u^\top \Big(\cov_x^\texttt{NA} - \textnormal{diag}\big(\cov_x^\texttt{NA}\big)\big)\Big) u +  \underbrace{\max\limits_{\norm{u}_2=1} u^\top \Big(-\textnormal{diag}\big(\cov_x^\texttt{NA}\big)\Big) u}_{ \leq 0}\\
        \leq & \max\limits_{\norm{u}_2=1} u^\top \Big(\cov_x^\texttt{NA} - \textnormal{diag}\big(\cov_x^\texttt{NA}\big)\big)\Big) u\\
        = & \max\limits_{\norm{u}_2=1} \sum\limits_{\substack{i,j=1 \\ i \neq j}}^d u_iu_j \big(\cov_x^\texttt{NA}\big)_{ij}.
    \end{align*}
    Since for off-diagonal elements of $\cov_x^\texttt{NA}$, we have $\big(\cov_x^\texttt{NA}\big)_{ij} = p_ip_j \big(\cov_x\big)_{ij}$, we obtain 
    \begin{align}
         \norm{\cov_x^\texttt{NA}}_\textnormal{op} \leq \norm{\mathbf{P}}_\textnormal{op}^2 \norm{\cov_x}_\textnormal{op}. 
    \end{align}
    Lastly, 
    \begin{align*}
        \mathbf{r}\big(\cov_x^\texttt{NA}\big) \leq \frac{\norm{\mathbf{P}}_{\textnormal{op}} \textnormal{Tr}\big(\cov_x\big) }{\norm{\cov_x^\texttt{NA}}_\textnormal{op}} +\frac{\norm{\mathbf{P}}_{\textnormal{op}}\norm{\mathbf{m}_x}_2^2}{\norm{\cov_x^\texttt{NA}}_\textnormal{op}}  \leq &  \frac{\norm{\mathbf{P}}_{\textnormal{op}} \norm{\cov_x}_\textnormal{op} }{\norm{\cov_x^\texttt{NA}}_\textnormal{op}}\mathbf{r}\big(\cov_x\big) +\frac{\norm{\mathbf{P}}_{\textnormal{op}}\norm{\mathbf{m}_x}_2^2}{\norm{\cov_x^\texttt{NA}}_\textnormal{op}} \\
        \leq & \frac{\norm{\mathbf{P}}^3_{\textnormal{op}} \norm{\cov_x}_\textnormal{op} }{c_\sigma}\mathbf{r}\big(\cov_x\big) + \norm{\mathbf{P}}_\textnormal{op}^3 c_\sigma^{-1}\norm{\mathbf{m}_x}_2^2.
    \end{align*}
\end{proof}

 \subsection{Additional Lemma 3}

Define 
\begin{align*}
     & \mathbf{W}_{xn}^\texttt{NA}:= \frac{1}{n}\sum\limits_{i=1}^n \big(\mathbf{x}_i^{\texttt{NA}} - \mathbf{m}_x^\texttt{NA}\big)\big(\mathbf{x}_i^{\texttt{NA}} - \mathbf{m}_x^\texttt{NA}\big)^\top\\
     & \varepsilon_{xn}^\texttt{NA} := \mathbf{m}^\texttt{NA}_{xn}-\mathbf{m}^\texttt{NA}_{x}.
 \end{align*}
 
 \begin{lem}
 \label{lem:tech_lem_3}
     We have 
    \begin{align*}
        \textnormal{Tr}\Big(\widehat{\cov}_{xn}\Big) = \textnormal{Tr}\Big(\mathbf{P}^{-1}\mathbf{W}_{xn}^\textnormal{\texttt{NA}}\Big) - 4\big\lVert\mathbf{P}^{-1}\varepsilon_{xn}^\textnormal{\texttt{NA}}\big\rVert_\textnormal{2}^2 + 2\big\lVert\mathbf{P}^{-\frac{1}{2}}\varepsilon_{xn}^\textnormal{\texttt{NA}} \big\rVert_2^2 - \Big\lVert \mathbf{P}^{-\frac{1}{2}}\big(\mathbf{I}_d-\mathbf{P}^{-1}\big)^{\frac{1}{2}}\mathbf{m}^\textnormal{\texttt{NA}}_{xn}\Big\rVert_2^2.
    \end{align*}
 \end{lem}
 \begin{proof}
 First, we have 
     \begin{align*}
         \mathbf{S}_{xn}^\texttt{NA} =    \mathbf{W}_{xn}^\texttt{NA} - 2\varepsilon_{xn}^\texttt{NA}\varepsilon_{xn}^{\texttt{NA}\top}
     \end{align*}
     Hence
     \begin{align*}
         \widehat{\cov}_{xn} = & \mathbf{P}^{-1}\mathbf{W}_{xn}^\texttt{NA}\mathbf{P}^{-1} + \mathbf{P}^{-1}\big(\mathbf{I}_d-\mathbf{P}^{-1}\big)\textnormal{diag}\big(\mathbf{W}_{xn}^\texttt{NA}\big) \\
        - & 2\mathbf{P}^{-1} \varepsilon_{xn}^\texttt{NA}\varepsilon_{xn}^{\texttt{NA}\top}\mathbf{P}^{-1} \\
         + & 2 \mathbf{P}^{-1}\big(\mathbf{I}_d-\mathbf{P}^{-1}\big)\textnormal{diag}\big(\varepsilon_{xn}^\texttt{NA}\varepsilon_{xn}^{\texttt{NA}\top}\big)\\
         + & \mathbf{P}^{-1}\big(\mathbf{I}_d-\mathbf{P}^{-1}\big)\textnormal{diag}\big(\mathbf{m}^\texttt{NA}_{xn}\mathbf{m}^{\texttt{NA}\top}_{xn}\big)\\
         = & \mathbf{P}^{-1}\mathbf{W}_{xn}^\texttt{NA}\mathbf{P}^{-1} + \mathbf{P}^{-1}\big(\mathbf{I}_d-\mathbf{P}^{-1}\big)\textnormal{diag}\big(\mathbf{W}_{xn}^\texttt{NA}\big)-2\mathbf{P}^{-2}\Big( \varepsilon_{xn}^\texttt{NA}\varepsilon_{xn}^{\texttt{NA}\top} + \textnormal{diag}\big(\varepsilon_{xn}^\texttt{NA}\varepsilon_{xn}^{\texttt{NA}\top}\big)\Big)\\
         + & 2 \mathbf{P}^{-1}\textnormal{diag}\big(\varepsilon_{xn}^\texttt{NA}\varepsilon_{xn}^{\texttt{NA}\top}\big) + \mathbf{P}^{-1}\big(\mathbf{I}_d-\mathbf{P}^{-1}\big)\textnormal{diag}\big(\mathbf{m}^\texttt{NA}_{xn}\mathbf{m}^{\texttt{NA}\top}_{xn}\big).
     \end{align*}
     Since $\textnormal{Tr}\big(xx^\top\big) = \norm{x}_2^2$ for all $x\in \mathbb{R}^d$, taking the trace of this expression gives 
     \begin{align*}
         \textnormal{Tr}\Big(\widehat{\cov}_{xn}\Big) = \textnormal{Tr}\Big(\mathbf{P}^{-1}\mathbf{W}_{xn}^\texttt{NA}\Big) - 4\big\lVert\mathbf{P}^{-1}\varepsilon_{xn}^\texttt{NA}\big\rVert_\textnormal{2}^2 + 2\big\lVert\mathbf{P}^{-\frac{1}{2}}\varepsilon_{xn}^\texttt{NA} \big\rVert_2^2 + \Big\lVert \mathbf{P}^{-\frac{1}{2}}\big(\mathbf{I}_d-\mathbf{P}^{-1}\big)^{\frac{1}{2}}\mathbf{m}^\texttt{NA}_{xn}\Big\rVert_2^2. 
     \end{align*}
 \end{proof}

\subsection{Additional Lemma 4}

\begin{lem}
\label{lem:tech_lem_4}
    We have 
    \begin{align*}
        \Big\lvert \textnormal{Tr}\Big(\mathbf{A}^{\frac{1}{2}} - \mathbf{B}^{\frac{1}{2}}\Big) \Big\rvert \leq & \frac{1}{\lambda_\textnormal{min}\Big(\mathbf{A}^\frac{1}{2}\Big) + \lambda_\textnormal{min}\Big(\mathbf{B}^\frac{1}{2}\Big)}\Big\lvert \textnormal{Tr}\Big(\mathbf{A}- \mathbf{B}\Big) \Big\rvert \\
        \leq & \frac{1}{\min\Big\{\lambda_\textnormal{min}\Big(\mathbf{A}^\frac{1}{2}\Big), \lambda_\textnormal{min}\Big(\mathbf{B}^\frac{1}{2}\Big)\Big\}}\Big\lvert \textnormal{Tr}\Big(\mathbf{A}- \mathbf{B}\Big) \Big\rvert.
    \end{align*}
\end{lem}
\begin{proof}
    The difference $ \mathbf{A}^{\frac{1}{2}} - \mathbf{B}^{\frac{1}{2}}$ solves the matrix equation 
    \begin{align*}
        \Big(\mathbf{A}^{\frac{1}{2}} - \mathbf{B}^{\frac{1}{2}}\Big)\mathbf{A}^{\frac{1}{2}} + \Big(\mathbf{A}^{\frac{1}{2}} - \mathbf{B}^{\frac{1}{2}}\Big)\mathbf{B}^{\frac{1}{2}} = \mathbf{A}- \mathbf{B}.
    \end{align*}
    From this, we have 
    \begin{align*}
        \textnormal{Tr}\Big(\big(\mathbf{A}^{\frac{1}{2}} - \mathbf{B}^{\frac{1}{2}}\big)\mathbf{A}^{\frac{1}{2}}\Big) + \textnormal{Tr}\Big(\big(\mathbf{A}^{\frac{1}{2}} - \mathbf{B}^{\frac{1}{2}}\big)\mathbf{B}^{\frac{1}{2}}\Big) = \textnormal{Tr}\big(\mathbf{A}- \mathbf{B}\big)
    \end{align*}
    which implies
    \begin{align*}
         \textnormal{Tr}\Big(\big(\mathbf{A}^{\frac{1}{2}} - \mathbf{B}^{\frac{1}{2}}\big)\big(\mathbf{A}^{\frac{1}{2}} + \mathbf{B}^{\frac{1}{2}}\big)\Big) = \textnormal{Tr}\big(\mathbf{A}- \mathbf{B}\big).
    \end{align*}
    Now, since $\textnormal{Tr}\big(\mathbf{A}\mathbf{B}\big) \geq \textnormal{Tr}(\mathbf{B})\lambda_{\textnormal{min}}\big(\mathbf{A}\big)$, and $\lambda_\textnormal{min}\big(\mathbf{A}^{\frac{1}{2}} + \mathbf{B}^{\frac{1}{2}}\big) > 0$ we have 
    \begin{align*}
     \abs{\textnormal{Tr}\big(\mathbf{A}^{\frac{1}{2}} - \mathbf{B}^{\frac{1}{2}}\big)} \lambda_\textnormal{min}\big(\mathbf{A}^{\frac{1}{2}} + \mathbf{B}^{\frac{1}{2}}\big) \leq \abs{\textnormal{Tr}\big(\mathbf{A}- \mathbf{B}\big)}.
    \end{align*}
    From Weyl's inequality, we obtain $\lambda_\textnormal{min}\big(\mathbf{A}^{\frac{1}{2}} + \mathbf{B}^{\frac{1}{2}}\big) \geq \lambda_\textnormal{min}\big(\mathbf{A}^{\frac{1}{2}}\big) + \lambda_\textnormal{min}\big(\mathbf{B}^{\frac{1}{2}}\big)$. From this the result follows.
\end{proof}

\subsection{Additional Lemma 5}

\begin{lem}
\label{lem:tech_lem_5}
    We have 
    \begin{align*}
       & \big\lvert \textnormal{Tr} \Big(\widehat{\cov}_{xn}^{\frac{1}{2}}\widehat{\cov}_{ym}\widehat{\cov}_{xn}^{\frac{1}{2}}\Big)^\frac{1}{2} - \textnormal{Tr}\Big(\cov^\frac{1}{2}_x\cov_y\cov_x^\frac{1}{2}\Big)^{\frac{1}{2}} \big\rvert \\
        \leq & \frac{\norm{\widehat{\cov}_{xn}}_\textnormal{op}}{\lambda_\textnormal{min}\Big(\big(\cov^\frac{1}{2}_x\cov_y\cov_x^\frac{1}{2}\big)^\frac{1}{2}\Big)} \Bigg[\abs{\textnormal{Tr}\Big[\big(\widehat{\cov}_{ym}-\cov_y\big)\Big]}+\frac{\norm{\cov_y}_\textnormal{op}}{\lambda_\textnormal{min}\big(\cov_x^\frac{1}{2}\big)}\abs{\textnormal{Tr}\big[\big(\widehat{\cov}_{xn}-\cov_x\big)\big]}\Bigg]
    \end{align*}
\end{lem}
\begin{proof}
    Lemma \ref{lem:tech_lem_4} gives 
        \begin{align*}
           & \big\lvert \textnormal{Tr} \Big(\widehat{\cov}_{xn}^{\frac{1}{2}}\widehat{\cov}_{ym}\widehat{\cov}_{xn}^{\frac{1}{2}}\Big)^\frac{1}{2} - \textnormal{Tr}\Big(\cov^\frac{1}{2}_x\cov_y\cov_x^\frac{1}{2}\Big)^{\frac{1}{2}} \big\rvert \\
           \leq & \frac{1}{\lambda_\textnormal{min}\Big(\big(\cov^\frac{1}{2}_x\cov_y\cov_x^\frac{1}{2}\big)^\frac{1}{2}\Big)}\big\lvert \textnormal{Tr}\Big(\widehat{\cov}_{xn}^{\frac{1}{2}}\widehat{\cov}_{ym}\widehat{\cov}_{xn}^{\frac{1}{2}} - \cov^\frac{1}{2}_x\cov_y\cov_x^\frac{1}{2}\Big) \big\rvert.
        \end{align*}
    Now, 
    \begin{align*}
            & \widehat{\cov}_{xn}^\frac{1}{2}\widehat{\cov}_{ym}\widehat{\cov}_{xn}^\frac{1}{2} - \cov^\frac{1}{2}_x\cov_y\cov_x^\frac{1}{2} \\
             = & \widehat{\cov}_{xn}^\frac{1}{2}\widehat{\cov}_{ym}\widehat{\cov}_{xn}^\frac{1}{2} -  \widehat{\cov}_{xn}^\frac{1}{2}\cov_y\widehat{\cov}_{xn}^\frac{1}{2} + \widehat{\cov}_{xn}^\frac{1}{2}\cov\widehat{\cov}_{xn}^\frac{1}{2} - \cov^\frac{1}{2}_x\cov_y\cov_x^\frac{1}{2}\\
            = &  \widehat{\cov}_{xn}^\frac{1}{2}\widehat{\cov}_{ym}\widehat{\cov}_{xn}^\frac{1}{2} -  \widehat{\cov}_{xn}^\frac{1}{2}\cov_y\widehat{\cov}_{xn}^\frac{1}{2} + \widehat{\cov}_{xn}^\frac{1}{2}\cov_y\widehat{\cov}_{xn}^\frac{1}{2}-\widehat{\cov}_{xn}^\frac{1}{2}\cov_y\cov_x^\frac{1}{2} + \widehat{\cov}_{xn}^\frac{1}{2}\cov_y\cov_x^\frac{1}{2} - \cov^\frac{1}{2}_x\cov_y\cov_x^\frac{1}{2}.
        \end{align*}
        We bound every term in this decomposition separately. Let us first define 
        \begin{align*}
            & \mathbf{A}_{nm} := \widehat{\cov}_{xn}^\frac{1}{2}\widehat{\cov}_{ym}\widehat{\cov}_{xn}^\frac{1}{2} -  \widehat{\cov}_{xn}^\frac{1}{2}\cov_y\widehat{\cov}_{xn}^\frac{1}{2}\\
            & \mathbf{B}_{nm} := \widehat{\cov}_{xn}^\frac{1}{2}\cov_y\widehat{\cov}_{xn}^\frac{1}{2}-\widehat{\cov}_{xn}^\frac{1}{2}\cov_y\cov_x^\frac{1}{2}\\
            & \mathbf{C}_{nm}:= \widehat{\cov}_{xn}^\frac{1}{2}\cov_y\cov_x^\frac{1}{2} - \cov^\frac{1}{2}_x\cov_y\cov_x^\frac{1}{2}.
        \end{align*}
        \paragraph{Bounding $\mathbf{A}_{nm}$.} We have 
        \begin{align*}
            \widehat{\cov}_{xn}^\frac{1}{2}\widehat{\cov}_{ym}\widehat{\cov}_{xn}^\frac{1}{2} -  \widehat{\cov}_{xn}^\frac{1}{2}\cov_y\widehat{\cov}_{xn}^\frac{1}{2} = \widehat{\cov}_{xn}^\frac{1}{2}\big(\widehat{\cov}_{ym}-\cov_y\big)\widehat{\cov}_{xn}^\frac{1}{2}.
        \end{align*}
        Taking the trace gives 
        \begin{align*}
            \textnormal{Tr}\big(\mathbf{A}_{nm}\big) = \textnormal{Tr}\Big[\widehat{\cov}_{xn}\big(\widehat{\cov}_{ym}-\cov_y\big)\Big]
        \end{align*}
        and therefor
        \begin{align*}
            \abs{\textnormal{Tr}\big(\mathbf{A}_{nm}\big)} \leq \norm{\widehat{\cov}_{xn}}_\textnormal{op}\abs{\textnormal{Tr}\Big[\big(\widehat{\cov}_{ym}-\cov_y\big)\Big]}
        \end{align*}
        \paragraph{Bounding $\mathbf{B}_{nm}$.} We have  
        \begin{align*}
             \textnormal{Tr}\big(\widehat{\cov}_{xn}^\frac{1}{2}\cov_y\widehat{\cov}_{xn}^\frac{1}{2}-\widehat{\cov}_{xn}^\frac{1}{2}\cov_y\cov_x^\frac{1}{2}\big) =  \textnormal{Tr}\Big[\widehat{\cov}_{xn}^\frac{1}{2}\cov_y\big(\widehat{\cov}_{xn}^\frac{1}{2}-\cov_x^\frac{1}{2}\big)\Big].
        \end{align*}
        From Lemma \ref{lem:tech_lem_4}, we get
        \begin{align*}
            \abs{\textnormal{Tr}\big(\mathbf{B}_{nm}\big)} \leq \frac{\norm{\widehat{\cov}_{xn}}_\textnormal{op}\norm{\cov_y}_\textnormal{op}}{\lambda_\textnormal{min}\big(\cov_x^\frac{1}{2}\big)} \abs{\textnormal{Tr}\big[\big(\widehat{\cov}_{xn}-\cov_x\big)\big]}. 
        \end{align*}
        \paragraph{Bounding $\mathbf{C}_{nm}$.} We similarly have 
        \begin{align*}
           \abs{\textnormal{Tr}\big(\mathbf{C}_{nm}\big)}  \leq \frac{\norm{\widehat{\cov}_{xn}}_\textnormal{op}\norm{\cov_y}_\textnormal{op}}{\lambda_\textnormal{min}\big(\cov_x^\frac{1}{2}\big)} \abs{\textnormal{Tr}\big[\big(\widehat{\cov}_{xn}-\cov_x\big)\big]}.
        \end{align*}
        In the end, we obtain
        \begin{align*}
            & \big\lvert \textnormal{Tr} \Big(\widehat{\cov}_{xn}^{\frac{1}{2}}\widehat{\cov}_{ym}\widehat{\cov}_{xn}^{\frac{1}{2}}\Big)^\frac{1}{2} - \textnormal{Tr}\Big(\cov^\frac{1}{2}_x\cov_y\cov_x^\frac{1}{2}\Big)^{\frac{1}{2}} \big\rvert \\
           \leq & \frac{1}{\lambda_\textnormal{min}\Big(\big(\cov^\frac{1}{2}_x\cov_y\cov_x^\frac{1}{2}\big)^\frac{1}{2}\Big)} \norm{\widehat{\cov}_{xn}}_\textnormal{op}\abs{\textnormal{Tr}\Big[\big(\widehat{\cov}_{ym}-\cov_y\big)\Big]} \\
           + & \frac{2}{\lambda_\textnormal{min}\Big(\big(\cov^\frac{1}{2}_x\cov_y\cov_x^\frac{1}{2}\big)^\frac{1}{2}\Big)} \frac{\norm{\widehat{\cov}_{xn}}_\textnormal{op}\norm{\cov_y}_\textnormal{op}}{\lambda_\textnormal{min}\big(\cov_x^\frac{1}{2}\big)} \abs{\textnormal{Tr}\big[\big(\widehat{\cov}_{xn}-\cov_x\big)\big]}\\
           = & \frac{\norm{\widehat{\cov}_{xn}}_\textnormal{op}}{\lambda_\textnormal{min}\Big(\big(\cov^\frac{1}{2}_x\cov_y\cov_x^\frac{1}{2}\big)^\frac{1}{2}\Big)} \Bigg[\abs{\textnormal{Tr}\Big[\big(\widehat{\cov}_{ym}-\cov_y\big)\Big]}+\frac{\norm{\cov_y}_\textnormal{op}}{\lambda_\textnormal{min}\big(\cov_x^\frac{1}{2}\big)}\abs{\textnormal{Tr}\big[\big(\widehat{\cov}_{xn}-\cov_x\big)\big]}\Bigg].
        \end{align*}
\end{proof}

 \subsection{Additional Lemma 6}

Let
\begin{align*}
    \widehat{\mathbf{T}}_\ell(x):= \mathbf{m}_{n_t} + \widehat{\mathbf{A}}_{n_sn_t}\Big(\mathbf{P}^{-1}x-\mathbf{P}^{-1}\mathbf{m}^\texttt{NA}_{n_s}\Big)
\end{align*}
where we define the symmetric matrix
\begin{align*}
    \widehat{\mathbf{A}}_{n_sn_t} := \widehat{\cov}_{n_s}^{-\frac{1}{2}}\Big(\widehat{\cov}^\frac{1}{2}_{n_s}\widehat{\cov}_{n_t}\widehat{\cov}_{n_s}^\frac{1}{2}\Big)^{\frac{1}{2}}\widehat{\cov}_{n_s}^{-\frac{1}{2}}.
\end{align*}
We conversely let 
\begin{align*}
     \mathbf{T}_\ell^\star(x):= \mathbf{m}_{t} + \mathbf{A}_\star\Big(x-\mathbf{m}_{s}\Big)
\end{align*}
with 
\begin{align*}
    \mathbf{A}_\star := \cov_{s}^{-\frac{1}{2}}\Big(\cov^\frac{1}{2}_s\cov_t\cov_s^\frac{1}{2}\Big)^{\frac{1}{2}}\cov_{s}^{-\frac{1}{2}}
\end{align*}

\begin{lem}
\label{lem:tech_lem_6}
    For $n_s$ and $n_t$ large enough, we have 
        \begin{align*}
        \mathbb{E}_{X \sim \mu^\texttt{NA}}\Big[\norm{\mathbf{T}_\ell^\star(X) - \widehat{\mathbf{T}}_\ell(X)}\Big] 
        \lesssim &\frac{\norm{\cov_t}^{\frac{1}{2}}}{\sqrt{n_t}}\Big(\mathbf{r}\big(\cov_t\big)+2 \sqrt{\cov_t t}+2t\Big)^{\frac{1}{2}} + \frac{\norm{\mathbf{P}^{-1}}_\textnormal{op}}{\sqrt{n_s}} \sqrt{K_2^x\Big(1+ 2\sqrt{K_2^x \norm{\mathbf{P}}_{\textnormal{op}}\norm{\cov_x}_\textnormal{op} t }+2t\Big)}\\
        + & \Bigg[\frac{\kappa(\cov_s)}{\sqrt{\lambda_\textnormal{min}\big(\cov^\frac{1}{2}_{s}\cov_{t}\cov_{s}^\frac{1}{2}\big)}} \mathbf{\Delta}_{n_t}(\eta,t)+\frac{\kappa(\cov_t) \norm{\cov_t}_\textnormal{op}}{\sqrt{\lambda_\textnormal{min}\big(\cov^{-1/2}_{n_s}\cov_{n_t}\cov_{n_s}^{-1/2}\big)}}\mathbf{\Delta}_{n_s}(\mu^\texttt{NA},t)\Bigg]K^x_4\\
         + & \mathcal{O}(n_s^{-1}) + \mathcal{O}(n_t^{-1})
    \end{align*}
    with probability at least $1-4e^{-t}$, where 
    \begin{align*}
        K_4^x := \sqrt{\norm{\mathbf{P}}_{\textnormal{op}} \big(\norm{\cov_x}_\textnormal{op} \mathbf{r}\big(\cov_x\big) + \norm{\mathbf{m}_x}_2^2\big)}.
    \end{align*}
\end{lem}

\begin{proof}
    Our proof follows the proof of \citet[Theorem 1]{flamary2019concentration}. First observe that for all $x \in \mathbb{R}^d$,
    \begin{align*}
        \norm{\mathbf{T}_\ell^\star(x) - \widehat{\mathbf{T}}(x)} \leq \norm{\mathbf{m}_t - \mathbf{m}_{n_t}}_2 + \norm{\mathbf{A}_\star - \widehat{\mathbf{A}}_{n_sn_t}}_\textnormal{op} \norm{\mathbf{P}^{-1}x-\mathbf{m}_s}_2 + \norm{\widehat{\mathbf{A}}_{n_sn_t}}\norm{\mathbf{P}^{-1}\mathbf{m}^\texttt{NA}_{n_s}-\mathbf{m}_s}_2.
    \end{align*}
    We define 
    \begin{align*}
        & \mathbf{a}_{n_t}:= \norm{\mathbf{m}_t - \mathbf{m}_{n_t}}_2\\
        & \mathbf{b}_{n_s n_t}:=\norm{\mathbf{A}_\star - \widehat{\mathbf{A}}_{n_sn_t}}_\textnormal{op}\\
        &  \mathbf{c}_{n_s n_t}:=\norm{\widehat{\mathbf{A}}_{n_sn_t}}\norm{\mathbf{P}^{-1}\mathbf{m}^\texttt{NA}_{n_s}-\mathbf{m}_s}
    \end{align*}
    \paragraph{Concentration of $\mathbf{a}_{n_t}$.} 
    From \citet{hsu2012tail}, we obtain with probability at least $1-e^{-t}$
    \begin{align*}
       \norm{\mathbf{m}_t - \mathbf{m}_{n_t}}_2 \lesssim \frac{\norm{\cov_t}^{\frac{1}{2}}}{\sqrt{n_t}}\Big(\mathbf{r}\big(\cov_t\big)+2 \sqrt{\cov_t t}+2t\Big)^{\frac{1}{2}}.
    \end{align*}
    \paragraph{Concentration of $\mathbf{b}_{n_s n_t}$.} Following the proof of \citet[Theorem 3]{flamary2019concentration}, we have 
    \begin{align*}
        \norm{\mathbf{A}_\star - \widehat{\mathbf{A}}_{n_sn_t}}_\textnormal{op} \lesssim \frac{\kappa(\widehat{\cov}_s)}{\sqrt{\lambda_\textnormal{min}\big(\widehat{\cov}^\frac{1}{2}_{n_s}\cov_{n_t}\widehat{\cov}_{n_s}^\frac{1}{2}\big)}} \norm{\cov_{t}-\widehat{\cov}_{n_t}}_\textnormal{op} + \frac{\kappa(\cov_t) \norm{\cov_t \sharp \widehat{\cov}^{-1}_{n_s}}\norm{\cov_t \sharp \cov^{-1}}}{\sqrt{\lambda_\textnormal{min}\big(\widehat{\cov}^{-1/2}_{n_s}\cov_{n_t}\widehat{\cov}_{n_s}^{-1/2}\big)}} \norm{\cov_{s}-\widehat{\cov}_{n_s}}_\textnormal{op}
    \end{align*}
    where $\kappa(\mathbf{A}):= \norm{\mathbf{A}^{-1}}_\textnormal{op}\norm{\mathbf{A}}_\textnormal{op}$ is the conditioning number of a matrix $\mathbf{A}$ and $$\mathbf{A}\sharp \mathbf{B}:=\mathbf{A}^\frac{1}{2}\Big(\mathbf{A}^{-\frac{1}{2}} \mathbf{B}\mathbf{A}^{-\frac{1}{2}}\Big) \mathbf{A}^\frac{1}{2} = \mathbf{A}\Big(\mathbf{A}^{-1}\mathbf{B}\Big)^{\frac{1}{2}}$$ is the geometric mean between two matrices. We now define the event $\mathcal{E}_1$ as 
    \begin{align*}
        \Big\{\norm{ \cov^{-1}_s\big(\cov_{s}-\widehat{\cov}_{n_s}\big)}_\textnormal{op} \leq \frac{1}{2}\Big\} \cap \Big\{\norm{\cov_{s}-\widehat{\cov}_{n_s}}_\textnormal{op} \leq \frac{ \norm{\cov_s}_\textnormal{op}}{2}\Big\} \cap \Bigg\{ \norm{\cov_{s}-\widehat{\cov}_{n_s}}_\textnormal{op} \leq \frac{\lambda_\textnormal{min}\Big(\cov_s^\frac{1}{2} \cov_t\cov_s^\frac{1}{2}\Big)}{2 \sqrt{6}\norm{\cov_t}_\textnormal{op}\sqrt{\kappa(\cov_s)}}\Bigg\}.
    \end{align*}
    On this event, \citet{flamary2019concentration} showed that
    \begin{align*}
        \mathbf{b}_{n_s n_t} \lesssim \frac{\kappa(\cov_s)}{\sqrt{\lambda_\textnormal{min}\big(\cov^\frac{1}{2}_{s}\cov_{t}\cov_{s}^\frac{1}{2}\big)}} \norm{\cov_{t}-\widehat{\cov}_{n_t}}_\textnormal{op} + \frac{\kappa(\cov_t) \norm{\cov_t}_\textnormal{op}\norm{\cov^{-1}}_\textnormal{op}}{\sqrt{\lambda_\textnormal{min}\big(\cov^{-1/2}_{s}\cov_{t}\cov_{s}^{-1/2}\big)}} \norm{\cov_{s}-\widehat{\cov}_{n_s}}_\textnormal{op}.
    \end{align*}
    We conclude this concentration bound using Lemma \ref{lem:extension_pacreau}, hence obtaining
    \begin{align*}
          \mathbf{b}_{n_s n_t} \lesssim & \frac{\kappa(\cov_s)}{\sqrt{\lambda_\textnormal{min}\big(\cov^\frac{1}{2}_{s}\cov_{t}\cov_{s}^\frac{1}{2}\big)}} \mathbf{\Delta}_{n_t}(\eta,t)+\frac{\kappa(\cov_t) \norm{\cov_t}_\textnormal{op}}{\sqrt{\lambda_\textnormal{min}\big(\cov^{-1/2}_{n_s}\cov_{n_t}\cov_{n_s}^{-1/2}\big)}}\mathbf{\Delta}_{n_s}(\mu^\texttt{NA},t) + \frac{1}{n_s}K_1^x(t)+ \frac{1}{n_t}K_1^y(t).
    \end{align*}
    with probability at least $1-2e^{-t}$.

    \paragraph{Concentration of $\mathbf{c}_{n_sn_t}$.} First, since $\widehat{\mathbf{A}}_{n_sn_t}$ concentrates around its mean, we have for large enough $n_s$ and $n_t$, and for a large enough multiplicative constant $C>0$, that
    \begin{align*}
        \norm{\widehat{\mathbf{A}}_{n_sn_t}}_\textnormal{op}\norm{\mathbf{P}^{-1}\widehat{\mathbf{m}}^\texttt{NA}_s-\mathbf{m}_s}_2 \leq C \norm{\mathbf{P}^{-1}\widehat{\mathbf{m}}^\texttt{NA}_s-\mathbf{m}_s}_2.
    \end{align*}
    The remaining term can be concentrated similarly to $\mathbf{a}_{n_t}$ and we get with probability at least $1-e^{-t}$ that
    \begin{align*}
        \mathbf{c}_{n_s} \lesssim & \norm{\mathbf{P}^{-1}}_\textnormal{op}\frac{\norm{\cov^\texttt{NA}_s}^{\frac{1}{2}}}{\sqrt{n_s}}\Big(\mathbf{r}\big(\cov^\texttt{NA}_s\big)+2 \sqrt{\mathbf{r}\big(\cov^\texttt{NA}_s\big) t}+2t\Big)^{\frac{1}{2}}\\
        = &  \frac{\norm{\mathbf{P}^{-1}}_\textnormal{op}}{\sqrt{n_s}} \sqrt{K_2^x\Big(1+ 2\sqrt{K_2^x \norm{\mathbf{P}}_{\textnormal{op}}\norm{\cov_x}_\textnormal{op} t }+2t\Big)}
    \end{align*}

    \paragraph{Final bound.} Combining everything gives 
    \begin{align*}
        \norm{\mathbf{T}^\star_\ell(x) - \widehat{\mathbf{T}}(x)} \lesssim & \frac{\norm{\cov_t}^{\frac{1}{2}}}{\sqrt{n_t}}\Big(\mathbf{r}\big(\cov_t\big)+2 \sqrt{\cov_t t}+2t\Big)^{\frac{1}{2}} +  \norm{\mathbf{P}^{-1}}_\textnormal{op} \Bigg[\frac{K_2^x\Big(1+ 2\sqrt{K_2^x \norm{\mathbf{P}}_{\textnormal{op}}\norm{\cov_x}_\textnormal{op} t }+2t\Big)}{n}\Bigg]^{\frac{1}{2}}\\
        + & \Bigg[\frac{\kappa(\cov_s)}{\sqrt{\lambda_\textnormal{min}\big(\cov^\frac{1}{2}_{s}\cov_{t}\cov_{s}^\frac{1}{2}\big)}} \mathbf{\Delta}_{n_t}(\eta,t)+\frac{\kappa(\cov_t) \norm{\cov_t}_\textnormal{op}}{\sqrt{\lambda_\textnormal{min}\big(\cov^{-1/2}_{n_s}\cov_{n_t}\cov_{n_s}^{-1/2}\big)}}\mathbf{\Delta}_{n_s}(\mu^\texttt{NA},t)\Bigg]\norm{\mathbf{P}^{-1}x-\mathbf{m}_s}\\
         + & \mathcal{O}(n_s^{-1}) + \mathcal{O}(n_t^{-1})
    \end{align*}
    with probability at least $1-4e^{-t}$.
    It remains to bound $\mathbb{E}_{X\sim \mu^\texttt{NA}}\norm{\mathbf{P}^{-1}X-\mathbf{m}_s}_2$. We have 
    \begin{align*}
        \mathbb{E}_{X\sim \mu^\texttt{NA}}\norm{\mathbf{P}^{-1}X-\mathbf{m}_s}_2 \leq \sqrt{\mathbb{E}_{X\sim \mu^\texttt{NA}}\norm{\mathbf{P}^{-1}X-\mathbf{m}_s}^2_2}
    \end{align*}
    by Cauchy-Schwarz, and now 
    \begin{align*}
        \sqrt{\mathbb{E}_{X\sim \mu^\texttt{NA}}\norm{\mathbf{P}^{-1}X-\mathbf{m}_s}^2_2} \lesssim &  \norm{\cov_x^\texttt{NA}}^\frac{1}{2}_\textnormal{op}\sqrt{\mathbf{r}\big(\cov_x^\texttt{NA}\big)}\\
        =&\sqrt{\norm{\mathbf{P}}_{\textnormal{op}} \big(\norm{\cov_x}_\textnormal{op} \mathbf{r}\big(\cov_x\big) + \norm{\mathbf{m}_x}_2^2\big)}.
    \end{align*}
    Combining everything, we obtain on $\mathcal{E}_1$ 
    \begin{align*}
        \mathbb{E}_{X \sim \mu^\texttt{NA}}\Big[\norm{\mathbf{T}_\star(X) - \widehat{\mathbf{T}}(X)}\Big] 
        \lesssim &\frac{\norm{\cov_t}^{\frac{1}{2}}}{\sqrt{n_t}}\Big(\mathbf{r}\big(\cov_t\big)+2 \sqrt{\cov_t t}+2t\Big)^{\frac{1}{2}} + \frac{\norm{\mathbf{P}^{-1}}_\textnormal{op}}{\sqrt{n_s}} \sqrt{K_2^x\Big(1+ 2\sqrt{K_2^x \norm{\mathbf{P}}_{\textnormal{op}}\norm{\cov_x}_\textnormal{op} t }+2t\Big)}\\
        + & \Bigg[\frac{\kappa(\cov_s)}{\sqrt{\lambda_\textnormal{min}\big(\cov^\frac{1}{2}_{s}\cov_{t}\cov_{s}^\frac{1}{2}\big)}} \mathbf{\Delta}_{n_t}(\eta,t)+\frac{\kappa(\cov_t) \norm{\cov_t}_\textnormal{op}}{\sqrt{\lambda_\textnormal{min}\big(\cov^{-1/2}_{n_s}\cov_{n_t}\cov_{n_s}^{-1/2}\big)}}\mathbf{\Delta}_{n_s}(\mu^\texttt{NA},t)\Bigg]K^x_4\\
         + & \mathcal{O}(n_s^{-1}) + \mathcal{O}(n_t^{-1})
    \end{align*}
    with probability at least $1-4e^{-t}$ where 
    \begin{align*}
        K^x_4 := \sqrt{\norm{\mathbf{P}}_{\textnormal{op}} \big(\norm{\cov_x}_\textnormal{op} \mathbf{r}\big(\cov_x\big) + \norm{\mathbf{m}_x}_2^2\big)}
    \end{align*}
    The result follows from taking $n_s$ and $n_t$ large enough to be on $\mathcal{E}_1$.
\end{proof}

\section{Computational Details and Experimental Results}

\label{appendix:experiments}

\subsection{Computational Budget}

We focus in this work on tasks with low computational cost. All experiments are run on a 2020 Macbook Pro laptop, with a 2 GHz Intel Core i5 4 cores CPU and 16 Go of RAM.

\subsection{ISVT Algorithm}

 Assume that $\mathbf{A}\in \mathbb{R}^{n \times d}$ is of rank $r$. We define the soft-thresholding operator 
\begin{align*}
    \mathbf{S}_\lambda: \mathbf{A} \mapsto \mathbf{U}\mathbf{D}_\lambda\mathbf{V}^\top
\end{align*}
where $\mathbf{U}\mathbf{D}\mathbf{V}^\top$ is the singular value decomposition (SVD) of $\mathbf{A}$, $\mathbf{D} = \textnormal{diag}[d_1,\dots,d_r]$ and $$\mathbf{D}_\lambda := \textnormal{diag}\big[(d_1-\lambda)_+,\dots,(d_r-\lambda)_+\big].$$

\begin{algorithm}
\caption{Iterative Singular Value Thresholding \citep{klopp2015matrix}}
\begin{algorithmic}[1]
\Require Input matrix $\mathbf{X}^\texttt{NA}$, mask matrix $\Omega$, regularization parameter $\lambda$, and $a$, an upper bound on $\norm{\mathbf{X}}_{\infty}$.
\State Initialize $\mathbf{X}^{\text{old}} = 0$
\State Initialize $\mathbf{X}^{\text{new}} = \mathbf{S}_\lambda\big(\mathbf{X}^\texttt{NA}\big)$
\While{$\norm{(1-\Omega)\odot(\mathbf{X}^\textnormal{new}-\mathbf{X}^\textnormal{old})}_2 \geq \frac{\lambda}{3}$ or $\norm{\mathbf{X}^\textnormal{new}-\mathbf{X}^\textnormal{old}}_\infty \geq a$} 
    \State $\mathbf{X}^\textnormal{new} \leftarrow \mathbf{S}_\lambda\big( \mathbf{X}^\texttt{NA} + (1-\Omega)\odot \mathbf{X}^{\textnormal{old}}\big)$
    \State For all $i,j$, \, $\mathbf{x}^{\textnormal{old}}_{ij} \leftarrow  a\,\textnormal{sign}(\mathbf{x}^\textnormal{new}_{ij})\mathds{1}_{\{|\mathbf{x}^\textnormal{new}_{ij}| > a\}} + \mathbf{x}^\textnormal{new}_{ij}\mathds{1}_{\{|\mathbf{x}^\textnormal{new}_{ij}| \leq a\}}$ 
\EndWhile
\State Return $\mathbf{X}^\texttt{new} =: \texttt{isvt}(\mathbf{X}^\texttt{NA})$.
\end{algorithmic}
\caption{Iterative Singular Value Thresholding \citep{klopp2015matrix}}
\label{alg:svt}
\end{algorithm}

\subsection{Details on Cross-Validation in Matrix Completion}

The typical procedure is to construct a new mask $\Omega_{\textnormal{train}}$ that starts from the initial mask $\Omega$, and removes a fixed share $\delta_\textnormal{val} \in [0,1]$ of the observed values uniformly at random. Hence the number of missing values is greater in $\Omega_{\textnormal{train}}$ than in $\Omega$. The indices of the removed entries are stored in a new mask $\Omega_\textnormal{val}$, which can be seen as a held-out validation set. This new mask has hence a share $\delta_\textnormal{val}$ of entries equal to $1$. Finally, the imputation model is trained only on values kept in $\Omega_{\textnormal{train}}$ and evaluated by computing the relative error in Frobenius norm
\begin{align*}
    \textnormal{RF}(\lambda) := \frac{\Big\lVert\Omega_\textnormal{val}\odot\texttt{isvt}_\lambda\Big[\mathbf{X}^\texttt{NA} \odot \Omega_\textnormal{train}\Big]- \Omega_\textnormal{val}\odot\mathbf{X}^\texttt{NA}\Big\rVert_2}{\norm{\Omega_\textnormal{val}\odot\mathbf{X}^\texttt{NA}}_2},
\end{align*}
on the entries in $\Omega_\textnormal{val}$. This metric is commonly used in matrix completion \citep{mazumder2010spectral,klopp2015matrix,alaya2019collective}.

\end{document}